\documentclass[11]{article}

\usepackage{amsmath, amssymb, amsthm, fullpage}
\usepackage{tikz}
\usepackage{mkolar_definitions}
\usepackage{comment}
\usepackage{color, graphicx, subcaption}
\usepackage{algorithm, algorithmicx,algpseudocode}
\usepackage{hyperref}
\usepackage{natbib}
\usepackage{xspace}

\newcount\Comments  % 0 suppresses notes to selves in text
\definecolor{darkgreen}{rgb}{0,0.5,0}
\definecolor{darkred}{rgb}{0.7,0,0}
\definecolor{teal}{rgb}{0.3,0.8,0.8}
\newcommand{\kibitz}[2]{\ifnum\Comments=1\textcolor{#1}{#2}\fi}

\newcommand{\nan}[1]{\kibitz{olive}{[NJ: #1]}}

\newcommand{\ntrain}{n}
\newcommand{\neval}{n_{\textrm{eval}}}

\newcommand{\nest}{n_{\textrm{est}}}

\newcommand{\treeroot}{\varnothing}
\newcommand{\berr}[1]{\Ecal(#1)}
\newcommand{\dtctberr}[1]{\tilde \Ecal(#1)} % the detection step
\newcommand{\empberr}[1]{\hat \Ecal(#1)} % the learning step
\newcommand{\Fht}[1]{\Fcal_{i_{#1}}}

\newcommand{\ft}[1]{f_{#1}}
\newcommand{\pit}[1]{\pi_{#1}}
\newcommand{\Th}{T_h} % # rounds h_t = h
\newcommand{\elip}{G}	% matrix for ellipsoid
\newcommand{\Vf}[1]{V_{#1}}
\newcommand{\Vpi}[1]{V^{#1}}
\newcommand{\empVf}[1]{\hat{V}_{#1}}

\newcommand{\brank}{\textrm{Bellman rank}\xspace}
\newcommand{\brankbig}{\textrm{Bellman Rank}\xspace}
\newcommand{\brankem}{\textrm{\it Bellman rank}\xspace}
\newcommand{\bfac}{\textrm{Bellman factorization}\xspace}
\newcommand{\bfacbig}{\textrm{Bellman Factorization}\xspace}

\newcommand{\bfacem}{\textrm{\it Bellman factorization}\xspace}
\newcommand{\vstar}{V_{\Fcal}^\star} % optimal value 

\newcommand{\primalbound}{\ensuremath{\Psi}}
\newcommand{\dualbound}{\ensuremath{\Phi}}

\newcommand{\normbound}{\ensuremath{\zeta}} % a unified constant that absorbs all norm-related constants

\newcommand{\modellong}{Contextual Decision Process\xspace}
\newcommand{\modellongs}{Contextual Decision Processes\xspace}
\newcommand{\modelshort}{CDP\xspace}
\newcommand{\modelshorts}{CDPs\xspace}
\newcommand{\model}{\modelshort}

\newcommand{\cdplearn}{\textsc{Olive}\xspace}
\newcommand{\cdplearnslack}{\textsc{Oliver}\xspace}
\newcommand{\cdplearnbf}{OLIVE\xspace}
\newcommand{\cdplearnslackbf}{OLIVER\xspace}

\newcommand{\pvec}{\nu} % the distribution vector in Bellman rank definition
\newcommand{\bvec}{\xi} % the bellman error vector in Bellman rank definition

\newcommand{\slack}{\theta}		% slackness in validity condition
\newcommand{\slackM}{\eta}	% slackness in factorization
\newcommand{\fstarslack}{f_{\slack}^\star} % f* when there is slackness
\newcommand{\vstarslack}{V_{\Fcal, \slack}^\star} % v* when there is slackness

\newcommand{\vval}{V-value\xspace} % state-value function

\newcommand{\qFht}[1]{\Ucal(\Fht{#1})}
\newcommand{\tick}{\tau} % time step in the ellipsoid lemma
\newcommand{\cutto}{\gamma} % previous \tau in todd; conflict with \tick
\newcommand{\trans}{\top}

\begin{document}

\begin{center}
{\LARGE{{\bf{Contextual Decision Processes with Low \brankbig \\ are PAC-Learnable \par}}}}

\vspace*{.2in}

\iffalse
\begin{tabular}{l r}
\textbf{Nan Jiang} & nanjiang@umich.edu\\
Computer Science and Engineering & \\
University of Michigan& \\
Ann Arbor, MI & \\
\textbf{Akshay Krishnamurthy} & akshay@cs.umass.edu\\
College of Information and Computer Sciences & \\
University of Massachusetts, Amherst & \\
Amherst, MA & \\
\textbf{Alekh Agarwal} & alekha@microsoft.com\\
\textbf{John Langford} & jcl@microsoft.com\\
\textbf{Robert E. Schapire} & schapire@microsoft.com\\
Microsoft Research & \\
New York, NY & 
\end{tabular}
\fi

\begin{tabular}{ccc}
  Nan Jiang$^\ddagger$
  &
  Akshay Krishnamurthy$^\star$
  &
  Alekh Agarwal$^\dagger$\\
  \texttt{nanjiang@umich.edu} &
  \texttt{akshay@cs.umass.edu}&
  \texttt{alekha@microsoft.com} 
  \vspace*{.1in}\\
  \multicolumn{3}{c}{
  John Langford$^\dagger$
  \qquad \qquad \quad
  Robert E. Schapire$^\dagger$}\\
    \multicolumn{3}{c}{
  \texttt{jcl@microsoft.com} \qquad
  \texttt{schapire@microsoft.com}}
\end{tabular}

\vspace*{.2in}

\begin{tabular}{ccc}
    University of Michigan$^\ddagger$ & University of
  Massachusetts, Amherst$^\star$ & Microsoft Research$^\dagger$ \\
  Ann Arbor, MI & Amherst, MA & New York, NY
\end{tabular}
\end{center}

\begin{abstract}
This paper studies systematic exploration for reinforcement learning
with rich observations and function approximation. We introduce a new
model called \emph{contextual decision processes}, that unifies
and generalizes most prior settings. Our first contribution is a
complexity measure, the \brankem, that we show enables tractable
learning of near-optimal behavior in these processes and is naturally
small for many well-studied reinforcement learning settings. Our
second contribution is a new reinforcement learning algorithm that
engages in systematic exploration to learn contextual decision
processes with low \brank. Our algorithm provably learns near-optimal
behavior with a number of samples that is polynomial in all relevant
parameters but independent of the number of unique observations. The
approach uses Bellman error minimization with optimistic exploration
and provides new insights into efficient exploration for reinforcement
learning with function approximation.
\end{abstract}

\section{Introduction}
\label{sec:intro}

How can we tractably solve sequential decision making problems where
the agent receives rich observations? 

In this paper, we study this question by considering reinforcement
learning (RL) problems where the agent receives rich sensory
observations from the environment, forms complex contexts from the
sensorimotor streams, uses function approximation to generalize to
unseen contexts, and must engage in \emph{systematic exploration} to
efficiently learn to complete tasks.  Such problems are at the core of
empirical reinforcement learning research (e.g., \citep{mnih2015human,
  bellemare2016unifying}), yet no existing theory provides rigorous
and satisfactory guarantees in a general setting.  \nan{``This
  situation motivates an important research question: how can we solve
  RL problems where exploration is critical and the agent receives
  rich observations in a sample-efficient manner?'' To avoid
  repetition I am removing this sentence, although I prefer removing
  the beginning one.}

To answer the question, we propose a new formulation, which we call
\emph{Contextual Decision Processes} (CDPs), to capture a large class
of sequential decision-making problems: CDPs generalize MDPs where the
state forms the context (Example~\ref{exm:mdp}), and POMDPs where the
history forms the context (Example~\ref{exm:pomdp}). We describe
\modelshorts formally in Section~\ref{sec:cdp}, and the learning goal
is to find a near-optimal policy for a \modelshort in a
sample-efficient manner.\footnote{Throughout the paper, by
  sample-efficient we mean a number of trajectories that is polynomial
  in the problem horizon, number of actions, 
  %notions analogous to the number of states which will be introduced
  Bellman rank (to be introduced), and polylogarithmic in the
  number of candidate value-functions.}

\paragraph{A structural assumption.} 
When the context space is very large or infinite, as is common in
practice, lower bounds that are \emph{exponential in the problem
  horizon} preclude efficient learning of CDPs, even when simple
function approximators are used.
%% In practice, the context space is
%% typically very large or infinite, which leads to lower bounds on the
%% sample-complexity of learning CDPs which are \emph{exponential in the
%%   problem horizon}, even when compact function approximators are
%% used. 
However, RL problems arising in applications are often far more benign
than the pathological lower bound instances, and we identify a
structural assumption capturing this intuition. As our first major
contribution, we define a notion of \bfacem
(Definition~\ref{def:brank}) in Section~\ref{sec:brank}, and focus on
problems with low \brankem.

\begin{table}[ht]
\centering
\begin{tabular}{|p{2.5cm}|c|c|c|c|c|c|} \hline
\emph{Model} & tabular MDP &  low-rank MDP & reactive POMDP & reactive PSR & LQR  \\ \hline
\brankem & \# states & rank & \# hidden states & PSR rank & \# state variables\\ \hline
\emph{PAC Learning} & known & new & extended & new & known\footnotemark[3] \\ \hline
\end{tabular}
\caption{Summary of settings that admit low \brank, with formal
  statements in Section~\ref{sec:brank}
  (Proposition~\ref{prop:mdp_finite} to \ref{prop:lqr}, from left to
  right in the table).  The 2nd row gives the parameters that bound
  the \brank. In the 3rd row, ``known" means that sample-efficient
  algorithms already exist for this setting (e.g., tabular MDPs),
  ``extended" means results here substantially extend previous work
  (e.g., POMDPs with large observation spaces and reactive value
  functions~\citep{krishnamurthy2016contextual}) and ``new" means our
  result gives the first sample-efficient algorithm (e.g., MDPs with
  low-rank transition dynamics).}
\label{tab:models_low_brank}
\end{table}

At a high level, \brank is a form of algebraic dimension on the
interplay between the CDP and the value-function approximator that we
show is small for many natural settings.  For example, every MDP with
a tabular value-function has \brank bounded by the rank of its
transition matrix, which is at most the number of states but can be
considerably smaller. For a POMDP with reactive value-functions, the
\brank is at most the number of hidden states and has no dependence on
the observation space. We provide other instances of low \brank
including Linear Quadratic Regulators and Predictive State
Representations. Overall, \modelshorts with a small \brank yield a
unified framework for a large class of sequential decision making
problems.

\paragraph{A new algorithm.} Our second contribution is a new algorithm for episodic
reinforcement learning called \cdplearn (Optimism Led Iterative
Value-function Elimination), detailed in Section~\ref{sec:alg_alg}.
\cdplearn combines optimism-driven exploration and Bellman error-based
search in a new way crucial for theoretical guarantees.

The algorithm is an iterative procedure that successively refines a
space of candidate Q-value functions $\Fcal$. At iteration $t$, it
first finds the surviving value function $f \in \Fcal_t$ that predicts
the highest value on the initial context distribution.  By collecting
a few trajectories according to $f$'s greedy policy, $\pi_f$, we can
verify this prediction. If the attained value is close to the
prediction, our algorithm terminates and outputs $f$. If not, we
eliminate all surviving $f' \in \Fcal_t$ which violate certain Bellman
equations on trajectories sampled using $\pi_f$. $\Fcal_{t+1}$ is set
to all the surviving functions and this process repeats.

\paragraph{A PAC guarantee.} 
We prove that \cdplearn performs sample-efficient learning in
\modelshorts with a small \brank (See
Section~\ref{sec:alg_pac}). Concretely, when the optimal
value-function in a CDP can be represented by the function
approximator $\Fcal$, the algorithm uses $\otil(M^2H^3K\log
(N/\delta)/\epsilon^2)$ trajectories to find an $\epsilon$-suboptimal
policy,\footnote{A logarithmic dependence on a norm parameter
  $\normbound$ is omitted here, as $\normbound$ is polynomial in most
  cases.}  where $M$ is the \brank, $H$ is the horizon (the
length of an episode), $K$ is the number of actions, $N$ is the
cardinality of $\Fcal$, and $\delta$ is the failure probability.

Importantly, the sample complexity bound has a logarithmic dependence
on $\Fcal$, thus enabling powerful function approximation, and no
direct dependence on the size of the context space, which can be very
large or even infinite. As many existing models, including the ones
mentioned above, have low \brank, the result immediately implies
sample-efficient learning in all of these settings,\footnote{Our algorithm requires discrete action spaces and does not immediately apply to LQRs; see more discussion in Section~\ref{sec:brank}.\label{foot:lqr}} as highlighted in
Table~\ref{tab:models_low_brank}. 

We also present several extensions of the main result, showing
robustness to the failure of the assumption
that the optimal value-function is captured by the function
approximator, adaptivity to unknown \brank, and extension to infinite
function classes of bounded statistical complexity. Altogether, these
results show that the notion of \brank robustly captures the
difficulty of exploration in sequential decision-making problems. \nan{The first sentence reads a bit awkward to me.}

To summarize, this work advances our understanding in reinforcement
learning with complex observations where long-term planning and
exploration are critical.  There are, of course, several additional
questions that must be resolved before we have satisfactory tools for
these problems.  The biggest drawback of \cdplearn is its
computational complexity, which is polynomial in the number of value
functions and hence intractable for the powerful classes of interest.
This issue must be addressed before we can empirically evaluate the
effectiveness of the proposed algorithm.  We leave this and other 
open questions for future work.

\paragraph{Related work.}
There is a rich body of theoretical literature on learning Markov
Decision Processes (MDPs) with small state
spaces~\citep{kearns2002near,brafman2003r,strehl2006pac}, with an
emphasis on sophisticated exploration techniques that find
near-optimal policies in a sample-efficient manner. While there have
been attempts to extend these techniques to large state
spaces~\citep{kakade2003exploration,jong2007model,pazis2016efficient},
these approaches fail to be a good fit for practical scenarios where
the environment is typically perceived through complex sensory
observations such as image, text, or audio signals. Alternatively,
Monte Carlo Tree Search (MCTS) methods can handle arbitrarily large
state spaces, but only at the cost of exponential dependence on the
planning horizon~\citep{kearns2002sparse,kocsis2006bandit}.  Our work
departs from these existing efforts by aiming for a sample complexity
that is independent of the size of the context space and at most
polynomial in the horizon.  Similar goals have been attempted by
\citet{wen2013efficient} and \citet{krishnamurthy2016contextual} where
attention is restricted to decision processes with deterministic
dynamics and special structures. In contrast, we study a much broader
class of problems with relatively mild conditions.

On the empirical side, the prominent recent success on both the Atari
platform~\citep{mnih2015human,wang2015dueling} and
Go~\citep{silver2016mastering} have sparked a flurry of research
interest.  These approaches leverage advances in deep learning for
powerful function approximation, while, in most cases, using simple
heuristic strategies, such as $\epsilon$-greedy, for exploration.
More advanced exploration strategies include extending 
the methods for small state spaces (e.g., the use of pseudo-counts
in~\citet{bellemare2016unifying}), % and Monte-Carlo Tree Search in AlphaGo~\citep{silver2016mastering}). 
and combining %Some of these approaches, such as 
MCTS with function approximation (e.g., \citet{silver2016mastering}). %, and, more generally, 
Unfortunately, both types of approaches %this program of adapting small-state methods to complex problems 
often require strong
domain knowledge and large amounts of data to be
successful.

\citet{hallak2015contextual} have proposed a setting called Contextual
MDPs, where a context refers to some static information that can be
used to generalize across many similar MDPs. In this paper, a context
is most similar to state features in the RL literature and is a
natural generalization of the notion of context as in the contextual
bandit literature \citep{langford2008epoch}.

%\section{Setup and Assumptions}

\section{Contextual Decision Processes (CDPs)}
\label{sec:cdp}

We introduce a new model, called a \modellong, as a unified framework
for reinforcement learning with rich observations. We first present
the model, before the relevant notation and definitions.

\subsection{Model and Examples}
\label{sec:cdp_examples}

%% Traditionally, Markov Decision Processes (MDPs) have been used as a
%% standard model of the environment in RL, but the Markovian assumption
%% is often unrealistic in real-world problems. While models such as
%% Partially Observable MDPs (POMDPs) allow for non-Markovian
%% environments, the emphasis is typically on the finite and small
%% observation spaces or on the latent structures which allow the
%% partial observability to be resolved, making the techniques unsuitable
%% in several practical applications. \nan{My original point here is that \# hidden states is an additional complexity constraint, so POMDPs are not general enough. I don't think it makes sense to discuss solution techniques here?}
%% %%on a particular latent structure makes it
%% %%less suitable for describing a general framework, and the observation
%% %%space is often assumed to be finite and small or possess some special
%% %%structures. 
%% Contextual Decision Processes use minimal assumptions to capture a
%% very general class of RL problems.

\modellongs make minimal assumptions to capture a very general class
of RL problems and are defined as follows.
%% , and then discuss their relationship with
%% existing models such as MDPs and POMDPs.

\begin{definition}[\modellong(\model)]
A (finite-horizon) \modellong (\model for short) is defined as a tuple
$(\Xcal, \Acal, H, P)$, where $\Xcal$ is the
context space, $\Acal$ is the action space, and $H$ is the horizon of
the problem. $P = (P_{\treeroot}, P_+)$ is the \emph{system
  descriptor}, where $P_{\treeroot} \in \Delta(\Xcal)$ is a
distribution over initial contexts, that is $x_1 \sim P_{\treeroot}$,
and $P_+: (\Xcal \times \Acal \times \RR)^* \times \Xcal \times \Acal
\to \Delta(\RR \times \Xcal)$ elicits the next reward and context from
the interactions so far $x_1, a_1, r_1, \ldots, x_h, a_h$:
$$ (r_h, x_{h+1}) \sim P_+(x_1, a_1, r_1, \ldots, x_h, a_h).
$$
\end{definition}

In a \model, the agent's interaction with the environment proceeds in
episodes. In each episode, the agent observes a context $x_1$, takes
action $a_1$, receives reward $r_1$ and observes $x_2$, repeating %this sequence
$H$ times. A policy $\pi: \Xcal \to \Acal$ specifies the
decision-making strategy of an agent, that is $a_h = \pi(x_h), ~
\forall h\in[H]$, and induces a distribution over the trajectory
$(x_1, a_1, r_1, \ldots, x_H, a_H, r_H, x_{H+1})$ according to the
system descriptor $P$.
\footnote{More generally, a sequence of stochastic policies $\pi_1,
  \ldots, \pi_H: \Xcal \to \Delta(\Acal)$ induces a distribution over
  trajectories in a similar way, where $a_h \sim \pi_h(x_h) ~ \forall
  h\in[H]$.}  The value of a policy, $V^{\pi}$, is defined as
\begin{align}
\label{eq:vpi_def}
V^{\pi} = \EE_P \left[ \textstyle \sum_{h=1}^H r_H ~\Big|~ a_{1:H}
  \sim \pi \right],
\end{align}
where $ a_{1:H} \sim \pi$ abbreviates for $a_1 = \pi(x_1), \ldots, a_H
= \pi(x_H)$. Here, and in the sequel, the expectation is always taken
over contexts and rewards drawn according to the system descriptor
$P$, so we suppress the subscript $P$ for brevity. The goal
of the agent is to find a policy $\pi$ that attains the largest value.

Below we show that \modelshorts capture classical RL models, including MDPs and POMDPs, and the optimal policies can be expressed as a function of appropriately chosen contexts.

\begin{example}[\emph{MDPs with states as contexts}] \label{exm:mdp}
Consider a finite-horizon MDP $(\Scal, \Acal, H, \Gamma_1,
\Gamma, R)$, where $\Scal$ is the state space, $\Acal$ is the
action space, $H$ is the horizon, $\Gamma_1 \in \Delta(\Scal)$ is the
initial state distribution, $\Gamma: \Scal \times \Acal \to
\Delta(\Scal)$ is the state transition function, $R: \Scal \times
\Acal \to \Delta([0, 1])$ is the reward function, and an episode takes
the form of $(s_1, a_1, r_1, \ldots, s_H, a_H, r_H)$. We can convert
the MDP to a CDP $(\Xcal, \Acal, H, P)$ by letting
$\Xcal = \Scal \times [H]$ and $x_h = (s_h, h)$, which allows the set of policies $\{\Xcal\to\Acal\}$ to contain the
optimal policy \citep{puterman1994markov}. The system descriptor
is $P=(P_{\treeroot}, P_+)$, where $P_{\treeroot}(x_1) =
\Gamma_1(s_1)$, and $ P_+(r_h, x_{h+1} \,|\, x_1, a_1, r_1, \ldots,
x_h, a_h) = R(r_h | s_h, a_h) \, \Gamma(s_{h+1} | s_h, a_h) $.
\end{example}

The system descriptor for a particular model is usually obvious from
the definitions, and here we give its explicit form as an illustration. We
omit the specification of system descriptor in the remaining examples.
 
Next we turn to POMDPs. It might seem that a CDP describes a similar
process as a POMDP but limits the agent's decision-making strategies
to memoryless (or reactive) policies, as we only consider policies in
$\{\Xcal\to\Acal\}$. This is not true. We clarify this issue by
showing that we can use the history as context, and the induced CDP
suffers no loss in the ability to represent optimal policies.

\begin{example}[\emph{POMDPs with histories as contexts}]\label{exm:pomdp}
Consider a finite-horizon POMDP with a \emph{hidden} state space
$\Scal$, an observation space $\Ocal$,
and an emission process $D_s$ that associates each $s\in\Scal$ with a
distribution over $\Ocal$. We can convert the POMDP to a CDP $(\Xcal, \Acal, H, P)$ by letting $\Xcal = (\Ocal\times \Acal
\times \RR)^* \times \Ocal$ and $x_h = (o_1, a_1, r_1, \ldots, o_h)$
is the observed history at time $h$.
\end{example}

It is also clear from this example that we can assume 
contexts are Markovian in CDPs without loss of generality, as we
can always use history as context. While we do not commit to this
assumption to allow for a flexible framework with simple notation
(see Example~\ref{exm:pomdp_window}), we later connect to well-known
results in MDP literature based on this observation so that readers
can transfer insights from MDPs to CDPs.

\begin{example}[\emph{POMDPs with sliding windows of observations as contexts}]
\label{exm:pomdp_window}
In some application scenarios, partial observability can be resolved
by using a small sliding window: for example, in Atari games, it is
common to keep track of the last $4$ frames of images
\citep{mnih2015human}. In this case, we can represent the problem as a
CDP by letting $x_h = (o_{h-3}, o_{h-2}, o_{h-1}, o_h)$.
\end{example}

We hope these examples convince the reader of the flexibility of
keeping contexts separate from intrinsic quantities such as states or
observations. Finally, we introduce a regularity assumption on the
rewards.

\begin{assum}[Boundedness of rewards]\label{ass:bounded}
We assume that regardless of how actions are chosen, for any $h=1,
\ldots, H$, $r_h \ge 0$ and $\sum_{h=1}^H r_h \le 1$ almost surely.
\end{assum}

\subsection{Value-based RL and Function Approximation}
\label{sec:values}

Now that we have a model in place, we turn to some important
solution concepts.

A \model makes no assumption on the cardinality of the context space,
which makes it critical to generalize across contexts, since the agent
might not encounter the same context twice. Therefore, we consider
value-based RL with function approximation. That is, the agent is
given a set of functions $\Fcal \subseteq \Xcal\times \Acal \to [0,
  1]$ and uses it to approximate an \emph{action-value function} (or
Q-value function). Without loss of generality we assume that
$f(x_{H+1},a) \equiv 0$.\footnote{This frees us from having to treat
  the last level ($h=H$) differently in the Bellman equations.} For
the purpose of presentation, we assume that $\Fcal$ is a finite space
with $|\Fcal| = N < \infty$ for most of the paper. In
Section~\ref{sec:inf_hyp_class} we relax this assumption and allow
infinite function classes with bounded complexity.

As in typical value-based RL, the goal is to identify $f\in\Fcal$
which respects a particular set of \emph{Bellman equations} and
achieves a high value with its greedy policy $\pi_f(x) =
\argmax_{a\in\Acal} f(x,a)$. We next set up the appropriate extensions
of Bellman equations to \modelshorts and the optimal value $\vstar$
through a series of definitions. Unlike typical definitions in MDPs,
these involve both the \modelshort and function approximator $\Fcal$.

\begin{definition}[Average Bellman error]\label{def:berr}
Given any policy $\pi: \Xcal \to \Acal$ and a function $f:
\Xcal\times \Acal \to [0,1]$, the \emph{average Bellman error} of $f$
under roll-in policy $\pi$ at level $h$ is defined as
\begin{align}\label{eqn:berr}
\berr{f,\pi,h} = \EE \, \big[f(x_h, a_h) - r_h - f(x_{h+1},a_{h+1})
  ~\big|~ a_{1:h-1} \sim \pi,~ a_{h:h+1} \sim \pi_f \big].
\end{align}
\end{definition}

In words, the average Bellman error measures the self-consistency of a
function $f$ between its predictions at levels $h$ and $h+1$ when all
the previous actions are taken according to some policy
$\pi$. \footnote{In many existing approaches (e.g.,
  LSPI~\citep{lagoudakis2003least} and FQI~\citep{ernst2005tree}), the
  Bellman errors are defined as taking the expectation of a squared
  error unlike this definition.}  Given this definition, we now define
a set of Bellman equations.

\begin{definition}[Bellman equations and validity of $f$]\label{def:valid}
Given an $(f, \pi, h)$ triple, a \emph{Bellman equation} posits $\berr{f,
\pi, h} = 0$.  We say $f\in\Fcal$ is \emph{valid} if the Bellman equation
on $(f, \pi_{f'}, h)$ holds for every $f' \in \Fcal, h\in[H]$.
\end{definition}

Note that the validity assumption only considers roll-ins according to
the greedy policies $\pi_f$, which is the natural policy class in a
function approximation setting. In Section~\ref{sec:policy_vval}, we
show how to incorporate a separate policy class in these
definitions. In the MDP setting, each Bellman equation can be viewed
as the linear combination of the standard Bellman optimality equations
for $Q^\star$, \footnote{Readers who are not familiar with the
  definition of $Q^\star$ are advised to consult a textbook, such as
  \citep{sutton1998reinforcement}.}  where the coefficients
are the probabilities with which the roll-in policy $\pi$ visits each
state. This leads to the following consequence.

\begin{fact}[$Q^\star$ is always valid]
\label{fact:valid_mdp}
Given an MDP and a space of functions $\Fcal: \Scal \times [H] \times
\Acal \to [0, 1]$, if the optimal Q-value function of the MDP
$Q^\star$ lies in $\Fcal$,
then in the
corresponding CDP with $\Xcal = \Scal \times [H]$, $Q^\star$ is valid.
\end{fact}

While $Q^\star$ satisfies the Bellman equations and yields the optimal
policy $\pi^\star = \pi_{Q^\star}$, there can be other
functions which also satisfy the equations while yielding suboptimal
policies. This happens because Eq.~\eqref{eqn:berr} only considers
$a_h$ drawn according to $\pi_f$ and does not use the values
on other actions. For instance, consider a \model where at every
context, action $a$ always gets a reward of 0 and action $a'$ always
gets a reward of 1.  A function that predicts $f(x,a)=f(x,a')=0~
\forall x,a$ is trivially valid as long as tie-breaks always favor
$a$.

Since validity alone does not imply that we get a good policy, it is
natural to search for a valid value function which also induces a
high-value policy.  We formalize this goal in the next definition.

\begin{definition}[Optimal value]\label{def:vstar}
Define 
$
f^\star = \argmax_{f\in\Fcal: \,f \textrm{\,is valid}} \Vpi{\pi_f},
$ 
and 
$
\vstar = \Vpi{\pi_{f^\star}}.
$ 
\end{definition}

\begin{fact} \label{fact:qstar_vstar}
For the same setting as in Fact~\ref{fact:valid_mdp}, when $Q^\star
\in \Fcal$, we have $f^\star = Q^\star$, and $\vstar = V^\star$, which
is the optimal long-term value.
\end{fact}

Definition~\ref{def:vstar} implicitly assumes that there is at least
one valid $f\in\Fcal$. This is weaker than the realizability
assumption made in the value-based RL literature, that $\Fcal$
contains the optimal Q-value function of an MDP
$Q^\star$~\citep{antos2008learning, krishnamurthy2016contextual}.
Indeed, the setup subsumes realizability, as evidenced by
Fact~\ref{fact:valid_mdp} and \ref{fact:qstar_vstar}. When $Q^\star
\in \Fcal$, the algorithm aims to identify a policy achieving value
close to $V^\star$, the optimal value achievable by any agent. When no
functions in $\Fcal$ approximate $Q^\star$ well, finding the best
valid value function is still a meaningful and non-trivial objective.
In this sense, our work makes substantially weaker realizability-type
assumptions than prior theoretical results for value-based
RL~\citep{antos2008learning, krishnamurthy2016contextual}, which
assume $Q^\star \in \Fcal$ often in addition to several stronger
requirements.

\paragraph{Approximation to Bellman Equations.} 
In general, $\Fcal$ may not contain $Q^*$, or any valid functions at
all, which makes our learning goal trivial. It is desirable to have an
algorithm robust to such a scenario, and we show how our algorithm
requires only an approximate notion of validity in
Section~\ref{sec:robust}, implying a graceful degradation in the
results.

\section{\bfacbig and \brankbig}
\label{sec:brank}
\modelshorts are general models for sequential decision making, but
are there efficient RL algorithms for them? 

Unfortunately, without further assumptions, learning in CDPs is generally
hard, since they subsume MDPs and POMDPs with arbitrarily large
state/observation spaces. Moreover, a function class $\Fcal$ with low
statistical complexity,
which would generalize effectively in a standard supervised learning setting with a fixed data distribution,
does not overcome this difficulty in \modelshorts 
%generalization in \modelshorts is much more challenging 
where the data distribution crucially depends on the agent's policy. 
In particular, even when
$\log N$, the statistical complexity for finite classes, is small, there
exists an $\Omega(K^H)$ lower bound on the sample complexity of
learning \modelshorts.  The result is due
to~\citet{krishnamurthy2016contextual}, and is included in
Appendix~\ref{app:lower_exp} for completeness.

%% Unfortunately, this intuition does not hold: while simple function
%% classes do generalize effectively in supervised learning where data
%% come from a fixed distribution,\footnote{For a finite hypothesis
%%   class, its generalization error in supervised learning is controlled
%%   by the log of its cardinality.} in RL the data distribution
%% crucially depends on the agent's policy. 
%% Consequently, low
%% statistical capacity is not sufficient to guarantee generalization in
%% RL problems where exploration is critical. 

While exponential lower bounds for learning \modelshorts exist, they
are fairly pathological, and most real problems have substantially
more structure.  To capture these realistic instances and circumnavigate the lower
bounds, we propose a new complexity measure and restrict our attention
to settings where this measure is low. As we will see, this measure is
naturally small for many existing models, and, when it is small,
efficient reinforcement learning is possible.

The complexity measure we propose is a structural characterization of
the set of Bellman equations induced by the \model and the value-function
class (recall Definition~\ref{def:berr}) that we need to check to find
valid functions. While checking validity by enumeration is intractable
for large $\Fcal$, observe that the Bellman equations are structured
in tabular MDPs: the average Bellman error under any roll-in policy is
a stochastic combination of the single-state errors, and checking the
single-state errors (which is tractable) is sufficient to guarantee
validity. This observation hints toward a more general phenomenon:
whenever the collection of Bellman errors across all roll-in policies
can be concisely represented, we may be able to check the validity of
all functions in a tractable way.

%% This quantity is important as naive ways of checking the validity of all functions 
%% %A key difficulty in general \modelshorts is that the set of Bellman
%% %equations (Definition~\ref{def:valid}) is very large. We seek a
%% %function $f$ which is valid and induces a good policy. Yet, even
%% %checking the validity of an $f$ 
%% might necessitate the collection of
%% samples according to $\pi_{f'}$ for every $f' \in \Fcal$, 
%% which is intractable for large $\Fcal$. 
%% In tabular
%% MDPs, this is not a concern since knowing the Bellman error of $f$ on
%% every state is sufficient to judge the validity of $f$: the Bellman
%% error under any policy is just a linear combination of the errors on
%% individual states. This observation hints towards a more general
%% phenomenon: whenever the collection of Bellman errors across all
%% roll-in policies can be concisely represented, we might be able to
%% check the validity of all functions in a sample-efficient
%% manner.

This intuition motivates a new complexity measure that we call the
\brankem. Define the Bellman error matrices, one for each $h$, to be
$|\Fcal|\times|\Fcal|$ matrices where the $(f, f')^{\textrm{th}}$ entry
is the Bellman error $\berr{f,\pi_{f'},h}$. Informally, the \brank for
a \model and a given value-function class $\Fcal$ is a uniform upper
bound on the rank of these $H$ Bellman error matrices.

%% Based on earlier discussion, it is immediate that the \brank
%% defined this way is controlled by the number of states in a tabular
%% MDP setting (the formal statement is given in
%% Proposition~\ref{prop:mdp_finite}). %% Recall from Eq.~\eqref{eqn:berr}
%% %% that $\berr{f,\pi_{f'},h}$ depends on $\pi_{f'}$ via using $\pi_{f'}$
%% %% as the roll-in policy for the first $h-1$ time steps. \brank is low in
%% %% this case because, regardless of which roll-in policy $\pi_{f'}$ is
%% %% used, all that really matters for $\berr{f,\pi_{f'},h}$ is the induced
%% %% distribution over states at level $h$, which can be represented as a
%% %% $|\Scal|$-dimensional vector.
%% %%
%% More generally, whenever there are some compact statistics that can
%% summarize the influence of the roll-in policy $\pi$ for the purpose of
%% calculating the Bellman error of any $f$, the \brank of the resulting
%% problem is small. For tabular MDPs, such a statistic is the state at
%% level $h$. Moreover, the observability of the statistic is not
%% required; in POMDPs for example, the hidden state can also serve as
%% such a statistic (Proposition~\ref{prop:reactive_pomdp}).

Now we give the formal definition below.

\begin{definition}[\bfac and \brank]
\label{def:brank}
We say that a \model $(\Xcal, \Acal, H, P)$ and $\Fcal
\subset \Xcal \times \Acal \to [0, 1]$ admit \bfac with \brankem $M$
and norm parameter $\normbound$, if there exists $\pvec_h: \Fcal \to
\RR^M, \bvec_h: \Fcal \to \mathbb{R}^M$ for each $h\in[H]$, such that
for any $f,f'\in\Fcal, h\in[H]$,
\begin{align} \label{eq:bellman_decomposition}
\berr{f,\pi_{f'},h} = \langle \pvec_h(f'), \bvec_h(f) \rangle,
\end{align}
and  $\|\pvec_h(f')\|_2 \cdot \|\bvec_h(f)\|_2 \le \normbound < \infty$.
\end{definition}

The exact factorization in Eq.~\eqref{eq:bellman_decomposition} can be
relaxed to an approximate version as is discussed in
Section~\ref{sec:robust}. In the remaining sections of this paper we
introduce the main algorithm, and analyze its sample-efficiency in
problems with low \brank. In the remainder of this section we showcase
the generality of Definition~\ref{def:brank} by describing a number of
common RL settings that have a small \brank. Throughout, we see how
the \brank captures the process-specific structures that allow for
efficient exploration. Proofs of all claims in this section are
deferred to Appendix~\ref{app:low_brank_proof}.

We start with the tabular MDP setting, and show that the \brank is at
most the number of states.

\begin{proposition}[\brank bounded by number of states in MDPs] \label{prop:mdp_finite}
Consider the MDP setting of Example~\ref{exm:mdp} with the
corresponding \modelshort. With any $\Fcal \subset \Xcal\times \Acal
\to [0, 1]$, this model admits a \bfac with $M = |\Scal|$ and
$\normbound=2\sqrt{M}$.
\end{proposition}

A related model introduced by \citet{li2009unifying} for
extending tabular PAC-MDP methods to large MDPs using a form of state
abstractions also has low \brank (See
Appendix~\ref{app:li}).

%% We also note that there is a previous effort in extending tabular
%% PAC-MDP methods to large MDPs using a form of state abstractions
%% \citep[Section 8.2.3]{li2009unifying}, which we are also able to
%% subsume. 

The MDP example is particularly simple as 
each coordinate of the $M$-dimensional space corresponds to a state, which is observable. 
%the statistics which witness
%the \bfac, the Bellman errors in individual states, are
%observable. 
Our next few examples show that this is not necessary, and that \bfac
can be based on latent properties of the process. We next consider
large MDPs whose transition dynamics have a low-rank structure. A
closely related setting has been considered by
\citet{barreto2011reinforcement, barreto2014policy} where the low-rank
structure is exploited to speed up MDP planning, but prior to this
work, no sample-efficient RL algorithms are known for this setting.

\begin{proposition}[\brank in low-rank MDPs, informally]
\label{prop:mdp_low_rank_dynamics}
Consider the MDP setting of Example~\ref{exm:mdp} with a transition
matrix $\Gamma$ having rank at most $M$. The induced \modelshort along
with any $\Fcal \subset \Xcal\times \Acal \to [0, 1]$ admits a \bfac
with \brank $M$.
\end{proposition}

The next example considers POMDPs with large observations spaces and
reactive value functions, where the \brank is at most the number of
hidden states.

\begin{proposition}[\brank bounded by hidden states in reactive POMDPs]
\label{prop:reactive_pomdp}
Consider the POMDP setting of Example~\ref{exm:pomdp_window} with
$|\Scal| < \infty$ and a sliding window of size 1 along with the
induced \modelshort. Given any $\Fcal \subset \Xcal \times \Acal \to
[0, 1]$, this model admits a \bfac with $M = |\Scal|$ and
$\normbound=2\sqrt{M}$.
\end{proposition}

Proposition \ref{prop:mdp_low_rank_dynamics} and
\ref{prop:reactive_pomdp} can be proved under a unified model that
generalizes POMDPs by allowing the transition function and the reward
function to depend on the observation (See
Figure~\ref{fig:reactive_pomdp} (a) -- (c) for graphical
representations of these models). This unified model captures the
experimental settings considered in state-of-the-art empirical RL work
(Figure~\ref{fig:reactive_pomdp} (d)), where agents act in a
grid-world ($|\Scal|$ is small) and receives complex and rich
observations such as raw pixel images ($|\Ocal|$ is large). The model also subsumes and generalizes the setting of \citet{krishnamurthy2016contextual} which requires deterministic transitions in the underlying MDP. Our new algorithm eliminates the need for determinism, and still guarantee sample-efficient learning.

\begin{figure}
\centering
\begin{subfigure}[b]{0.45\textwidth}
\centering
\begin{tikzpicture}
\node[draw=black,fill=lightgray,circle,name=s1,minimum width=20pt] at (0,0) {};
\node[draw=black,fill=lightgray,circle,name=s2,minimum width=20pt] at (4,0) {};
\node[draw=black,circle,name=o1,minimum width=20pt] at (0,-1.5) {$s$};
\node[draw=black,circle,name=r1,minimum width=20pt] at (1.25,-1.75) {$r$};
\node[draw=black,circle,name=o2,minimum width=20pt] at (4,-1.5) {$s'$};
\node[draw=black,rectangle,name=a1,minimum height=20pt,minimum width=20pt,rotate=45] at (0.5, -3) {};
\node at (0.5,-3) {$a$};
%% ARROWS
\draw[black, ->] (s1) to (o1);
\draw[black, ->] (s2) to (o2);
%\draw[red, ->] (s1) to (r1);
\draw[black, ->] (o1) to (r1);
\draw[dashed, black, ->] (o1) to (a1);
\draw[black, ->] (a1) to (r1);
%\draw[red, ->] (s1) to (s2);
%\node at (2, 0.2) {$\Gamma(s,z,a)$};
\draw[black, ->] (o1) to (s2);
\draw[black, ->] (a1) to (s2);
\draw[black] (-.5 ,-0.75) node{$\ldots$};
\draw[black] (4.5 ,-0.75) node{$\ldots$};
\end{tikzpicture}
\caption{MDP with low-rank transition dynamics. Gray nodes represent
  the hidden factors in the low-rank factorization.} \label{fig:mdp_low_rank}
\end{subfigure}
~~
\begin{subfigure}[b]{0.35\textwidth}
\centering
\begin{tikzpicture}
\node[draw=black,fill=lightgray,circle,name=s1,minimum width=20pt] at (0,0) {$s$};
\node[draw=black,fill=lightgray,circle,name=s2,minimum width=20pt] at (4,0) {$s'$};
\node[draw=black,circle,name=o1,minimum width=20pt] at (0,-1.5) {$o$};
\node[draw=black,circle,name=r1,minimum width=20pt] at (1.25,-1.75) {$r$};
\node[draw=black,circle,name=o2,minimum width=20pt] at (4,-1.5) {$o'$};
\node[draw=black,rectangle,name=a1,minimum height=20pt,minimum width=20pt,rotate=45] at (0.5, -3) {};
\node at (0.5,-3) {$a$};
%% ARROWS
\draw[black, ->] (s1) to (o1);
\draw[black, ->] (s2) to (o2);
\draw[black, ->] (s1) to (r1);
%\draw[blue, ->] (o1) to (r1);
\draw[dashed, black, ->] (o1) to (a1);
\draw[black, ->] (a1) to (r1);
\draw[black, ->] (s1) to (s2);
%\node at (2, 0.2) {$\Gamma(s,z,a)$};
%\draw[blue, ->] (o1) to (s2);
\draw[black, ->] (a1) to (s2);
\draw[black] (-.5 ,-0.75) node{$\ldots$};
\draw[black] (4.5 ,-0.75) node{$\ldots$};
\end{tikzpicture}
\caption{POMDP with reactive policy. Gray nodes represent hidden states.\\~%#for alignment
} \label{fig:pomdp}
\end{subfigure} \\
\vspace*{1em}
\begin{subfigure}[b]{0.4\textwidth}
\centering
\begin{tikzpicture}
\node[draw=black,fill=lightgray,circle,name=s1,minimum width=20pt] at (0,0) {$s$};
\node[draw=black,fill=lightgray,circle,name=s2,minimum width=20pt] at (4,0) {$s'$};
\node[draw=black,circle,name=o1,minimum width=20pt] at (0,-1.5) {$o$};
\node[draw=black,circle,name=r1,minimum width=20pt] at (1.25,-1.75) {$r$};
\node[draw=black,circle,name=o2,minimum width=20pt] at (4,-1.5) {$o'$};
\node[draw=black,rectangle,name=a1,minimum height=20pt,minimum width=20pt,rotate=45] at (0.5, -3) {};
\node at (0.5,-3) {$a$};
%% ARROWS
\draw[black, ->] (s1) to (o1);
\draw[black, ->] (s2) to (o2);
\draw[black, ->] (s1) to (r1);
\draw[black, ->] (o1) to (r1);
\draw[dashed, black, ->] (o1) to (a1);
\draw[black, ->] (a1) to (r1);
\draw[black, ->] (s1) to (s2);
%\node at (2, 0.2) {$\Gamma(s,z,a)$};
\draw[black, ->] (o1) to (s2);
\draw[black, ->] (a1) to (s2);
\draw[black] (-.5 ,-0.75) node{$\ldots$};
\draw[black] (4.5 ,-0.75) node{$\ldots$};
\end{tikzpicture}
\caption{A unified model that subsumes (a) and (b) and yields low \brank.} \label{fig:ctx_pomdp}
\end{subfigure}
\qquad
\begin{subfigure}[b]{0.4\textwidth}
\includegraphics[width=\textwidth]{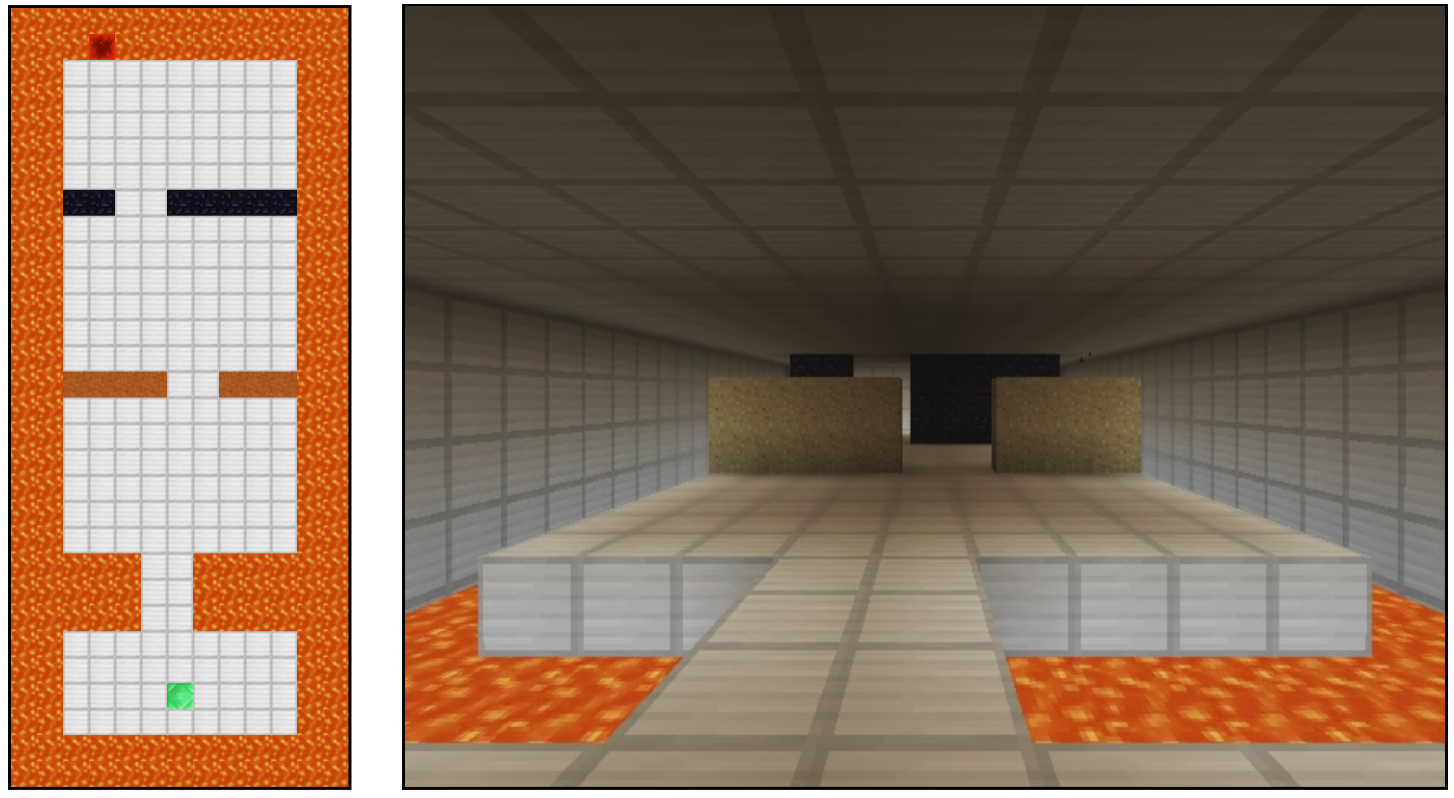}
\caption{A popular RL experiment setting \citep{johnson2016malmo}.} \label{fig:malmo}
\end{subfigure}
\caption{Example RL scenarios that yield low \brank. No
  sample-efficient exploration strategies are known for these problems
  previously, so the algorithm and analysis provide the first
  PAC-Learning guarantees. \textbf{(a)} An MDP that has low-rank
  transition structure (see
  Proposition~\ref{prop:mdp_low_rank_dynamics}). The unlabelled nodes
  correspond to the hidden factors in the low-rank factorization, and
  can take $M$ different values. \textbf{(b)} A POMDP with reactive
  policies (see Proposition~\ref{prop:reactive_pomdp}). The dashed
  arrow from $o$ to $a$ implies that we only consider policies that
  map $\Ocal$ to $\Acal$ (the dependence on time step is made
  implicit). \textbf{(c)} A model that unifies (a) and (b), which
  still yields low \brank (proof is in
  Appendix~\ref{app:low_brank_proof}). To embed (a) in (c), the
  observed states $s$ and $s'$ in (a) become the observations $o$ and
  $o'$ in (c). \textbf{(d)} A popular experiment setting in
  state-of-the-art empirical work. The environment is a grid-world
  (hence $|\Scal|$ is small and so is the \brank), and the agent's
  sensory inputs are raw pixel images (hence $|\Ocal|$ is very
  large).}
\label{fig:reactive_pomdp}
\end{figure}
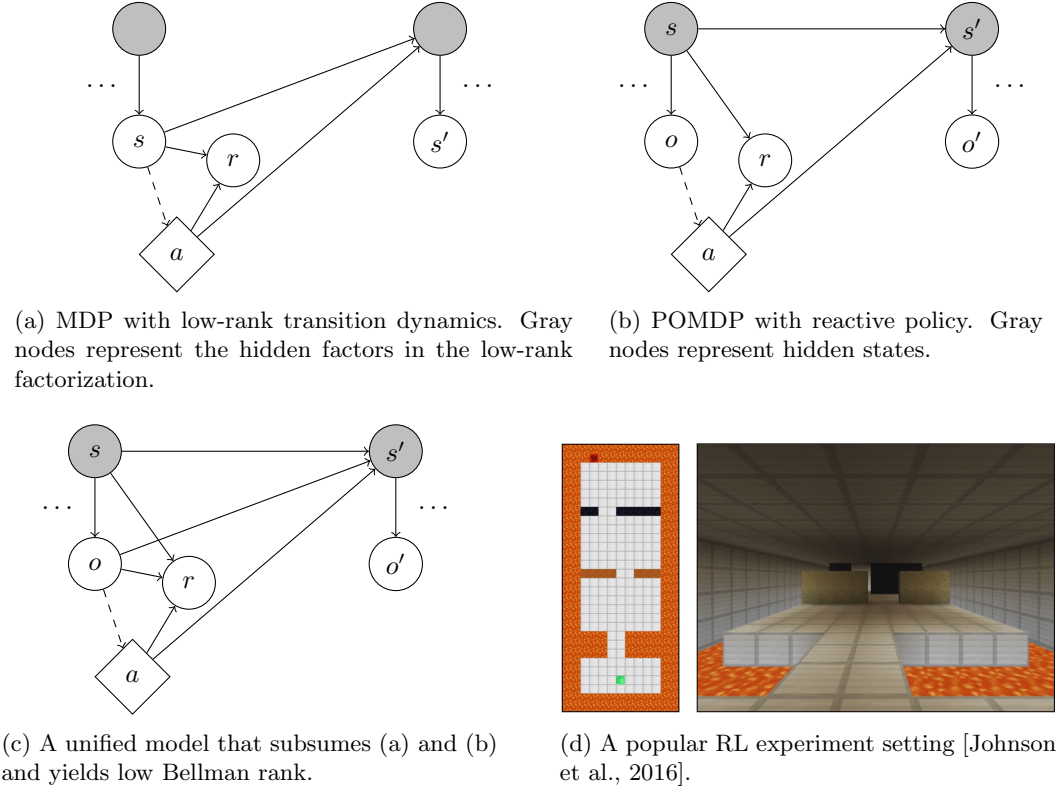

Next, we consider Predictive State Representations (PSRs), which are
models of partially observable systems with parameters
grounded in observable quantities
\citep{littman2001predictive}. Similar to the case of
POMDPs, we can bound the \brank in terms of the rank of the PSR\footnote{Every POMDP has an equivalent PSR whose rank is bounded by the number of hidden states \citep{singh2004predictive}.} when the
candidate value functions are reactive. %To avoid introducing the definitions and the properties of PSRs, 
%% For presentation purposes, we only give an informal result here, and defer the definitions and the properties of PSRs, the formal statement, and the proof to Appendix~\ref{app:psr}
 %% where the basic definitions and properties of PSRs are also recalled. % together with basic definitions and properties of PSRs.%the detailed definitions and properties of PSRs are recalled in the appendix.

\begin{proposition}[\brank in PSRs, informally] \label{prop:psr}
Consider a partially observable system with observation space $\Ocal$, and the induced CDP $(\Xcal, \Acal, H, P)$  with $x_h = (o_h,h)$. If the linear dimension of the system (i.e., rank of its PSR model) is at most $L$, then given any $\Fcal: \Xcal\times\Acal \to [0,1]$, the \brank is bounded by $LK$.
\end{proposition}

%The next example considers the reactive POMDP setting with a small number of observations and unbounded number of hidden-states. Interestingly, such a setting subsumes a previous extension of tabular
% PAC-MDP methods to large MDPs using so-called \emph{$Q^\star$-irrelevance} state abstractions \citep[Section 8.2.3]{li2009unifying}. The detailed connection is described in the appendix.

The last example considers a class of linear control problems well studied in control theory, called Linear Quadratic Regulators (LQRs). We show that the \brank in LQRs is bounded by the dimension of the state space. 
Exploration in this class of problems has been previously considered by \citet{osband2014model}. 
%% Again for presentation purposes, %to avoid introducing LQR definitions, 
%% we only state an informal result here and defer the formal statement to Appendix~\ref{app:lqr}. 
Note that the algorithm to be introduced in the next section does not directly apply to LQRs due to the continuous action space, and adaptations that exploit the structure of the action space may be needed, which we leave for future work. 

%\begin{proposition}[Linear Quadratic Regulator]
%A linear quadratic regulator (LQR) is a classical dynamical system studied in control theory~\citep{anderson2007optimal}.
%In these systems, both the observation space and action space are continuous, typically $\Xcal = \RR^d$ and $\Acal = \RR^K$. 
%We consider episodic LQRs where the dynamics and the cost  (i.e., negative reward) are described by the equations:
%\begin{align}
%x_{h+1} = Ax_h + Ba_h + \epsilon_h \qquad \textrm{and} \qquad c_h = x_h^\trans Qx_h + a_h^\trans a_h + \tau_h.
%\end{align}
%Here $x$'s denote the observations, $a$s denote the actions, and $A,Q
%\in \RR^{d\times d}, B \in \RR^{d \times K}$.  The noise variables
%$\epsilon_h, \tau_h$ are centered and independent.
%
%In an LQR, it is known that the optimal value function is a quadratic
%function of the state, while the optimal policy is linear.  With these
%value function and policy classes, in Appendix~\ref{app:lqr} we prove
%that the LQR has Bellman Rank at most $d^2+1$, and if the system
%parameters are bounded in spectral norm then the \bfac
%has its norm parameter $\normbound$ exponential in $H$ and polynomial in all other
%parameters.
%\label{prop:lqr}
%\end{proposition}

\begin{proposition}[\brank in LQRs, informally]\label{prop:lqr}
An LQR can be viewed as an MDP with continuous state space $\RR^d$  and action space $\RR^K$, where the dynamics are described by some linear equations. Given any function class $\Fcal$ consisting of non-stationary quadratic functions of the state, the \brank is bounded by $d^2+1$.
\end{proposition}

%\nan{Besides unifying settings where efficient exploration strategies are known, we also solve new problems that previous literature cannot handle. I tried to use a table to highlight this point.}

\section{Algorithm and Main Results}
\label{sec:algorithm}
In this section we present the algorithm for learning \modelshorts
that have a \bfac with a small \brank and the main
sample complexity guarantee.  To aid presentation and help convey the
main ideas, we make three simplifying assumptions:
\begin{enumerate}
\item We assume the \brank parameter $M$ is known to the agent.\footnote{We also assume
    knowledge of the corresponding norm parameter, but this is
    relatively minor.}
\item We assume the function class $\Fcal$ is finite with $|\Fcal| =
  N$.
\item We assume exact validity (Definition~\ref{def:valid}) and exact
  \bfac (Definition~\ref{def:brank}).
\end{enumerate}
All three assumptions can be relaxed, and we sketch these relaxations in Section~\ref{sec:extensions}. 
%Recall that the third assumption slightly generalizes the classical realizability assumption that $Q^\star \in \Fcal$. 

We are interested in designing an algorithm for \emph{PAC Learning \modelshorts}.
We say that an algorithm PAC learns if given $\Fcal$, two parameters $\epsilon, \delta \in (0,1)$, and access to a \modelshort, the algorithm outputs a policy $\hat{\pi}$ with $V^{\hat{\pi}} \ge \vstar - \epsilon$ with probability at least $1-\delta$.
The sample complexity is the number of episodes needed to achieve such a guarantee, and is typically expressed in terms of $\epsilon$, $\delta$ and other relevant parameters.
%The sample complexity is a function $n: (0,1)^2 \rightarrow \NN$ such that the number of episodes when run with parameters $\epsilon, \delta$ is at most $n(\epsilon, \delta)$.
The goal is to design an algorithm with sample complexity that is $\textrm{Poly}(M,K,H,1/\epsilon,\log(N),\log(1/\delta))$ where $M$ is the \brank, $K$ is the number of actions, and $H$ is the time horizon.
%% We say that such a sample complexity bound is polynomial in all relevant parameters.
Importantly, the bound has no dependence on the number of unique contexts $|\Xcal|$. 

\subsection{Algorithm}
\label{sec:alg_alg}
Pseudocode for the algorithm, which we call \cdplearn (Optimism Led
Iterative Value-function Elimination), is displayed in
Algorithm~\ref{alg:simple}.  Theorem~\ref{thm:cdp_complexity}
describes how to set the parameters $\nest, \neval, \ntrain$, and
$\phi$.

At a high level, the algorithm aims to eliminate functions $f \in \Fcal$ that fail to satisfy the validity condition in Definition~\ref{def:valid}. 
This is done by Lines~\ref{lin:estm_berr} and~\ref{lin:learning_step} inside the loop of the algorithm. 
Observe that, since the actions $a_{h_t}$ are chosen uniformly at random, Eq.~\eqref{eq:estm_berr} produces an unbiased estimate of $\berr{f,\pit{t},h_t}$, the average Bellman error for function $f$ on roll-in policy $\pit{t}$ at time $h_t$.
Thus, Eq.~\eqref{eq:learning_step} eliminates functions that have high average Bellman error on this distribution, which means they fail to satisfy the validity criteria. 
\nan{Mention that the estimate uses importance sampling?}

The other major component of the algorithm involves choosing the roll-in policy and level on which to do the learning step. 
At iteration $t$, we choose the roll-in policy $\pit{t}$ \emph{optimistically}, by choosing $\ft{t}$ that predicts the highest value at the starting context distribution, and letting $\pit{t} = \pi_{\ft{t}}$.
To pick the level, we compute $\ft{t}$'s average Bellman error on its own roll-in distribution (Eq.~\eqref{eq:ft_bellman}), and set $h_t$ to be any level for which this average Bellman error is high (See Line~\ref{lin:detected}). 
As we will show, these choices ensure that substantial learning happens on each iteration, guaranteeing that the algorithm uses polynomially many episodes.

The last component is the termination criterion. 
The algorithm terminates if $\ft{t}$ has small average Bellman error on its own roll-in distribution at all levels. 
This criteria guarantees that $\pit{t}$ is near optimal. 

Computationally, the algorithm requires enumeration of the
value-function class, which we expect to be extremely large or
infinite in practice. %While computational efficiency is essential for
A computationally efficient implementation is essential for a practical algorithm, which is left to future work. We focus on the sample efficiency of the algorithm in this paper. 
% so we do not dwell further on this deficiency.

\begin{algorithm}[t]
\begin{algorithmic}[1]
\State \textbf{Collect} $\nest$ trajectories %$t^{(1)}, \ldots, t^{(\nest)}$ 
with actions taken in an arbitrary manner; %taking actions arbitrarily
save initial contexts %first observations 
$\{x_1^{(i)}\}_{i=1}^{\nest}$.
\label{lin:init_eval}
\State \textbf{Estimate} the predicted value for each $f \in \Fcal$: $\empVf{f} = \frac{1}{\nest}\sum_{i=1}^{\nest} f(x_1^{(i)}, \pi_f(x_1^{(i)}))$. \label{lin:vf}
%\nan{Conflict of symbol?}
\State $\Fcal_0 \gets \Fcal$.
\For{$t=1,2,\ldots$} \label{lin:forloop}
\State \textbf{Choose policy} $\ft{t} = \argmax_{f \in \Fcal_{t-1}}
\empVf{f}$, $\pit{t} = \pi_{\ft{t}}$. \label{lin:opt}
\State \textbf{Collect} $\neval$ trajectories $\{(x^{(i)}_1,a^{(i)}_1,r^i_1,\ldots, x^{(i)}_H,a^{(i)}_H,r^{(i)}_H)\}_{i=1}^{\neval}$ by following $\pit{t}$ (i.e. $a_h^{(i)} = \pit{t}(x_h^{(i)})$ for all $h,i$). \label{lin:eval} 
\State \textbf{Estimate} $\forall h\in[H]$,  \label{lin:detect}
\begin{align}
\dtctberr{\ft{t}, \pi_t, h} = \frac{1}{\neval}\sum_{i=1}^{\neval} \left[  \ft{t}(x^{(i)}_h, a^{(i)}_h) - r^{(i)}_h - \ft{t}(x^{(i)}_{h+1}, a^{(i)}_{h+1}) \right].
\label{eq:ft_bellman}
\end{align}
\If{$\sum_{h=1}^H \dtctberr{\ft{t}, \pi_t,h} \le 5\epsilon/8$}
\State Terminate and ouptut $\pit{t}$. 
\EndIf
\State Pick any $h_t \in [H]$ for which $\dtctberr{\ft{t}, \pi_t,h_t} \ge 5\epsilon/8H$ (One is guaranteed to exist). \label{lin:detected}
\State Collect trajectories $\{(x_1^{(i)}, a_1^{(i)}, r_1^{(i)}, \ldots, x_H^{(i)}, a_H^{(i)}, r_H^{(i)})\}_{i=1}^{\ntrain}$ where $a_h^{(i)} = \pit{t}(x_h^{(i)})$ for all $h \ne h_t$ and $a_{h_t}^{(i)}$ is drawn uniformly at random. \nan{from Rob: why not stopping at $h_t$} \label{lin:explore}
\State \textbf{Estimate} \label{lin:estm_berr}
\begin{align}
\empberr{f,\pit{t}, h_t} &= \frac{1}{\ntrain}\sum_{i=1}^{\ntrain} \frac{\mathbf{1}[a_{h_t}^{(i)} = \pi_f(x_{h_t}^{(i)})]}{1/K} \Big(f(x_{h_t}^{(i)},a_{h_t}^{(i)}) - r_{h_t}^{(i)} - f(x_{h_{t}+1}^{(i)},\pi_f(x_{h_t+1}^{(i)}))\Big).
\label{eq:estm_berr}
\end{align}
\State \textbf{Learn} \label{lin:learning_step}
\begin{align}
\Fcal_{t} = \left\{f \in \Fcal_{t-1} : \left|\empberr{f,\pit{t},h_t} \right|\le \phi \right\}.
\label{eq:learning_step}
\end{align}
\EndFor
\end{algorithmic}
\caption{\cdplearn$(\Fcal,M,\normbound,\epsilon,\delta)$ -- \textbf{O}ptimism \textbf{L}ed \textbf{I}terative \textbf{V}alue-function \textbf{E}limination}
\label{alg:simple}
\end{algorithm}

\textbf{Intuition for \cdplearnbf.} To convey intuition, it is helpful
to ignore any sampling effects by replacing all empirical estimates
with population values, and set $\epsilon$ to $0$. The first important fact is that the algorithm
never eliminates a valid function, since the learning step in
Eq.~\eqref{eq:learning_step} only eliminates a function $f$  if we can
find a distribution on which it has a large average Bellman error. If
$f$ is valid, then $\berr{f,\pi,h} = 0$ for all $\pi,h$, so $f$ is
never eliminated.
%% This is clear by the properties of the learning steps in Eqs.~\eqref{eq:estm_berr} and~\eqref{eq:learning_step}, since if $f$ is valid, then $\berr{f,D(\pit{t},h_t)} = 0$. 

The second fact is that if a function $f$ is valid, then its predicted
value is exactly the value achieved by the greedy policy $\pi_f$, that
is $V_f = \EE[f(x_1,\pi_f(x_1))] = V^{\pi_f}$.  This follows by
telescoping the recursion in the definition of average Bellman error.
Therefore, since $\ft{t}$ is chosen \emph{optimistically} as the
maximizer of the value prediction among the surviving functions, and
since we never eliminate valid functions, if \cdplearn terminates, it
must output a %near 
policy with value $\vstar$.  In the analysis, we
incorporate sampling effects to derive robust versions of these facts
so the algorithm always outputs a policy that is at most
$\epsilon$-suboptimal.

The more challenging component is ensuring that the algorithm
terminates in polynomially many iterations, which is critical for
obtaining a polynomial sample complexity bound.  This argument
crucially relies on the \bfac (recall
Definition~\ref{def:brank}), which enables us to embed the
distributions into $M$ dimensions and measure progress in this
low-dimensional space. 

For now, fix some $h$ and focus on the iterations when $h_t = h$. If we ignore sampling effects we can set $\phi=0$, and, by using the
\bfac to write $\berr{f,\pi_{\ft{t}},h}$ as an inner product,
we can think of 
the learning step in Line~\ref{lin:learning_step} as
introducing a homogeneous linear constraint on the set of $\bvec_h(f)$
vectors, that is, $\langle \pvec_{h}(\ft{t}), \bvec_{h}(f)\rangle = 0$. Now, if we execute the learning step at level $h$
again in a later iteration $t'$, %that means we found a new function $\ft{t'}$ for which 
we have $\langle
\pvec_h(\ft{t'}),\bvec_h(\ft{t'}) \rangle \ne 0$ from Line~\ref{lin:detected}. Importantly, this
means that $\pvec_h(\ft{t'})$ must be linearly independent from previous $\pvec_h(\ft{t})$ since $\langle \pvec_h(\ft{t}), \bvec_h(\ft{t'}) \rangle = 0$. In general, every time $h_t = h$, the number of linearly independent constraints increases by 1, and therefore the number of iterations where $h_t=h$ is at most the dimension of the space, which is $M$.
Thus the \brank leads to a bound on the number of iterations.

 %, and the linear dimension of the remaining $\bvec_h(f)$ vectors decreases by 1.
%% such a number is upper bounded by $M$, the dimension of the space, which translates into a bound on the number of iterations.
%, introducing the constraint
%involving $\pvec_h(\ft{t'})$ guarantees we eliminate one more dimension
%from the null space of the constraint matrix. 
%Since there are only
%$M$ dimensions, this leads to a bound on the number of iterations.

The above heuristic reasoning, despite relying on the brittle notion
of linear independence, can be made robust.  With sampling effects,
rather than homogeneous linear equalities, the learning step for level
$h$ introduces linear inequality constraints to the $\bvec_h(f)$
vectors. But if $f'$ is a surviving function that forces us to train
at level $h$, it means that $\langle \pvec_h(f'),\bvec_h(f')\rangle$
is very large, while $\langle \pvec_h(\cdot),\bvec_h(f')\rangle$ is
very small for all previous $\pvec_h(\cdot)$ vectors used in the
learning step.  Intuitively this means that the new $\pvec(f')$ vector
is quite different from all of the previous ones. In our proof, we use
a volumetric argument to show that this suffices to guarantee
substantial learning takes place.

The optimistic choice for $\ft{t}$ is critical for driving the agent's
exploration.  With this choice, if $\ft{t}$ is valid, then the
algorithm terminates correctly, and if $\ft{t}$ is not valid, then
substantial progress is made. Thus the agent does not get stuck
exploring with many valid but suboptimal functions, which could result
in exponential sample complexity.

\subsection{Sample Complexity}
\label{sec:alg_pac}
We now turn to the main result, which guarantees that \cdplearn PAC-learns \modellongs with polynomial sample complexity. 

\begin{theorem}
\label{thm:cdp_complexity}
For any $\epsilon,\delta \in (0,1)$, any \modellong and function class $\Fcal$ that admits a \bfac with parameters $M,\normbound$, run \cdplearn with the following parameters:
\begin{align*}
\phi &= \frac{\epsilon}{12H\sqrt{M}}, \qquad \nest = \frac{32}{\epsilon^2}\log(6N/\delta), \\
\neval &= \frac{288H^2}{\epsilon^2}\log\left(\frac{12H^2 M \log(6H\sqrt{M}\normbound/\epsilon)}{\delta}\right),\\
\ntrain &= \frac{4608H^2MK}{\epsilon^2}\log\left(\frac{12NHM\log(6H\sqrt{M}\normbound/\epsilon)}{\delta}\right).
\end{align*}
Then, with probability at least $1-\delta$, \cdplearn halts and returns a policy $\hat{\pi}$ that satisfies $V^{\hat{\pi}} \ge \vstar - \epsilon$ ($\vstar$ in Definition~\ref{def:valid}), and the number of episodes required is at most\footnote{We use $\otil(\cdot)$ notation to suppress poly-logarithmic dependence on everything except $N$ and $\delta$.}
\begin{align}
\otil\left(\frac{M^2H^3K}{\epsilon^2}\log(N\normbound/\delta)\right).
\end{align}
\end{theorem}

According to this theorem, if a \modellong and function class $\Fcal$
admit a \bfac with small \brank and $\Fcal$ contains valid functions,
\cdplearn is guaranteed to find a near optimal valid function using
only polynomially many episodes.  
%Recall that many popular models 
%including MDPs with small state spaces, POMDPs with reactive value
%functions and small hidden-state spaces, and linear quadratic
%regulators 
%admit factorizations with 
%small \brank.
%Moreover, recall
%that the validity requirement is considerably weaker than
%realizability assumptions in prior works characterizing reinforcement
%learning with function approximation. 
%Our result is simultaneously 
%the most
%general polynomial sample complexity bound for reinforcement learning
%with rich observation spaces and for reinforcement learning with 
%function approximation. 
To our knowledge, our result is the most
general polynomial sample complexity bound for reinforcement learning
with rich observation spaces and function approximation, as  many popular models 
are shown to admit 
small \brank (see Section~\ref{sec:brank}). 
The result also certifies that the notion of
\bfac, which is quite general, is sufficient for efficient exploration
and learning in sequential decision making problems.

It is worth briefly comparing this result with prior work.
\begin{enumerate}
\item The most closely related result is the recent work
  of~\citet{krishnamurthy2016contextual}, who also consider episodic
  reinforcement learning with infinite observation spaces and function
  approximation.  The model studied there is a form of \modellong with
  \brank $M$, so the result applies as is to that setting.
  Importantly, we eliminate the need for deterministic transitions,
  resolving one of their open problems.  Moreover, the sample
  complexity bound improves the dependence on $H$ and $\epsilon$, at
  the cost of a worse dependence on $M$.  We emphasize that this
  result applies to a much more general class of models.
\item Another related body of work provides sample complexity bounds
  for fitted value/policy iteration methods (e.g.,
  \citep{munos2003error, antos2008learning, munos2008finite}). These
  works consider the infinite-horizon discounted MDP setting, and
  impose much stronger assumptions than we do including not only that
  the function class captures $Q^\star$, but also that it is approximately closed
  under Bellman update operators. %\akshay{Violations to these
%    assumptions are admitted, but the degree of violation appears in
%   the guarantees (see e.g., the inherent Bellman error in
%    \citep{antos2008learning}).} 
More importantly, the analyses rely
  on the so-called concentrability coefficients to correct the
  mismatch between training and test distributions
  \citep{farahmand2010error,lazaric2012finite}, implicitly assuming
  that an exploration distribution is given, hence these results do
  not address the exploration issue which is the main focus here.
\item Since \modelshorts include small-state MDPs~\citep{kearns2002near,brafman2003r,strehl2006pac}, the algorithm can
  be applied as is to these problems. Unfortunately, the sample
  complexity is polynomially worse than the state of the art
  $\otil(\frac{M\textrm{poly}(H)K}{\epsilon^2} \log(1/\delta))$ bounds 
  for PAC-learning MDPs~\citep{dann2015sample}.  On the other hand, the algorithm also applies to MDPs with infinite state spaces with Bellman factorizations, which cannot be handled by tabular methods.
  %\nan{Do we need to discuss metric-based exploration like \citep{kakade2003exploration} here?}
\item This approach applies to learning reactive policies in POMDPs (see Proposition~\ref{prop:reactive_pomdp}). \citet{azizzadenesheli2016reinforcement} provides a sample-efficient algorithm in a closely related setting, where both the observation space and the hidden-state space are small in cardinality. While their approach does not require realizable value-functions, the sample complexity depends polynomially on the number of unique observations, and the method relies on additional mixing assumptions, which we do not require. %properties of  the transition dynamics in the latent MDP.
\item Finally, \modellongs also encompass contextual bandits, where
  the optimal sample complexity is
  $O(K\log(N)/\epsilon^2)$~\citep{agarwal2012contextual}.  As
  contextual bandits have $M = 1$ and $H=1$, \cdplearn achieves
  optimal sample complexity in this special case.
\end{enumerate}

Turning briefly to lower bounds, since the \model setting with \bfac
is new, general lower bounds for the broad class do not
exist. However, we can use MDP lower bounds for guidance on the
question of optimality, since the small-state MDPs in
Example~\ref{exm:mdp} are a special case. While no existing MDP lower
bounds apply as is (because formulations vary), in
Appendix~\ref{app:lower_new} we adapt ideas from  \citet{auer2002nonstochastic} to obtain a
$\Omega(MKH/\epsilon^2)$ sample complexity lower bound for learning
the MDPs in Example~\ref{exm:mdp}.

Comparing with this lower bound, the sample complexity in
Theorem~\ref{thm:cdp_complexity} is worse in $M,H$, and $\log(N)$
factors, but of course the small-state MDP is a significantly simpler
special case.  We leave as future work the question of
optimal sample complexity for learning \modelshorts with low \brank.

\section{Extensions}
\label{sec:extensions}
We introduce four important extensions to the algorithm and analysis. 

\subsection{Unknown \brankbig} \label{sec:guessM}
The first extension eliminates the need to know $M$ in advance (note that Algorithm~\ref{alg:simple} requires $M$
as an input parameter).  A simple procedure described in
Algorithm~\ref{alg:guessM}, can guess the value of $M$ on a doubling
schedule and handle this situation with no consequences to asymptotic
sample complexity.\footnote{In Algorithm~\ref{alg:guessM} we assume
  that $\normbound$ is known. In the examples provided in
  Proposition~\ref{prop:mdp_finite}, \ref{prop:mdp_low_rank_dynamics},
  and \ref{prop:reactive_pomdp}, however, $\normbound$ grows with $M$
  in the form of $\normbound = 2\sqrt{M}$. In this case, we can
  compute $\normbound' = 2\sqrt{M'}$ and call \cdplearn with
  $\normbound'$ instead of $\normbound$. As long as $\normbound$ is a
  polynomial term and non-decreasing in $M$ the same analysis applies
  and Theorem~\ref{thm:guessM} holds.} %, which are indeed true in

\begin{algorithm}
\caption{\textsc{GuessM}$(\Fcal, \normbound, \epsilon, \delta)$}
\label{alg:guessM}
\begin{algorithmic}[1]
\For{$i=1,2,\ldots$} %// terminate with high probably at $i = \lceil \log_2 M \rceil$
\State $M' \gets 2^i$.
\State Call $\cdplearn(\Fcal, M', \epsilon, \frac{\delta}{i(i+1)})$ with parameters specified on Theorem~\ref{thm:cdp_complexity}.
\State Terminate the subroutine when $t > H M'\log\left(\frac{6H \sqrt{M'} \normbound}{\epsilon}\right)/\log(5/3)$ \label{lin:hard_stop}
%$t > H M'\log\left(\frac{6H\sqrt{M'}\primalbound\dualbound\circumscribe}{\epsilon\inscribe}\right)/\log(5/3)$ 
in Line~\ref{lin:forloop} (the for-loop).
\If{a policy $\pi$ is returned from \cdplearn}
\Return $\pi$.
\EndIf
\EndFor
\end{algorithmic}
\end{algorithm}

\begin{theorem}
\label{thm:guessM}
For any $\epsilon,\delta \in (0,1)$, any \modellong and function
class $\Fcal$ that admits a \bfac with parameters $M,\normbound$, if
we run \textsc{GuessM}$(\Fcal, \epsilon, \delta)$, then with
probability at least $1-\delta$, \cdplearn halts and returns a policy which
satisfies $V^{\hat{\pi}} \ge \vstar - \epsilon$, and the number of
episodes required is at most
\begin{align*}
\otil\left(\frac{M^2H^3K}{\epsilon^2}\log(N\normbound/\delta)\right).
\end{align*}
\end{theorem}
We give some intuition about the proof here, with details in
Appendix~\ref{app:extensions}. In Algorithm~\ref{alg:guessM}, $M'$ is
a guess for $M$ which grows exponentially. When $M' \ge M$, analysis
of the main algorithm shows that $\cdplearn(\Fcal, M', \epsilon,
\frac{\delta}{i(i+1)})$ terminates and returns a near-optimal policy
with high probability.  The doubling schedule implies that the largest
guess is at most $2M$, which has negligible effect on the sample
complexity.  On the other hand, \cdplearn may not explore effectively
when $M' < M$, because not enough samples (chosen according to $M'$)
are used to estimate the average Bellman errors in
Eq.~\eqref{eq:estm_berr}. This worse accuracy does not guarantee
sufficient progress in learning.

However, the high-probability guarantee that $f^\star$ is not
eliminated is unaffected, because the threshold $\phi$ on
Line~\ref{eq:learning_step} of \cdplearn is set in accordance with the
sample size $\ntrain$ specified in Theorem~\ref{thm:cdp_complexity},
regardless of $M$. Consequently, if the algorithm ever terminates when
$M' < M$, we still get a near-optimal policy.  When $M' < M$ the \cdplearn
subroutine may not terminate, which the explicit termination on
line~\ref{lin:hard_stop} in Algorithm~\ref{alg:guessM} addresses.
%% can tackle by putting a hard stop
%% when the sample complexity bound is reached (Line~\ref{lin:hard_stop}
%% in Algorithm~\ref{alg:guessM}). 
Finally, by splitting the failure probability $\delta$ appropriately
among all guesses of $M'$, we obtain the same order of sample
complexity as in Theorem~\ref{thm:cdp_complexity}.

\subsection{Separation of Policy Class and \vval Class}
\label{sec:policy_vval}
So far, we have assumed that the agent has access to a class of Q-value functions $\Fcal \subset \Xcal\times \Acal \to [0,1]$.
In this section, we show the algorithm allows separate representations of policies and \emph{\vval} functions. 

For every $f\in\Fcal$, and any $x\in\Xcal, a\ne \pi_f(x)$, we note that the value of $f(x,a)$ is not used by Algorithm~\ref{alg:simple}, and changing it to arbitrary values does not affect the execution of the algorithm as long as $f(x,a) \le f(x,\pi_f(x))$ (so that $\pi_f$ does not change). In other words, the algorithm only interacts with $f$ in two forms: 
\begin{enumerate}
\item $f$'s greedy policy $\pi_f$.
\item A mapping $g_f: x \mapsto f(x,\pi_f(x))$. We call such mappings \vval functions to contrast the previous use of Q-value functions.\footnote{In the MDP setting, such functions are also known as \emph{state-value functions}.}
\end{enumerate}
Hence, supplying $\Fcal$ is equivalent to supplying the following space of (policy, \vval function) pairs: 
\[
\{\big(\pi_f, g_f\big): f \in \Fcal \}.
\]
This observation provides further evidence that
Definition~\ref{def:valid} is significantly less restrictive than
standard realizability assumptions.  Validity of $f$ means that
$(\pi_f, g_f)$ obeys the \emph{Bellman Equations for Policy
  Evaluation} (i.e., $g_f$ predicts the long-term value of following
$\pi_f$), as opposed to the more common \emph{Bellman Optimality
  Equations}. In MDPs, there are many ways to satisfy the policy
evaluation equations at every state simultaneously, while $Q^\star$ is
the only function that satisfies all optimality equations.  

%% When both sets of equations are required to hold in a
%% state-wise fashion in MDPs, there can exist many solutions to the
%% former, while $Q^\star$ is the only solution to the latter.

More generally, instead of using a Q-value function class, we can run
\cdplearn with a policy space $\Pi \subset\Xcal \to \Acal$ and a \vval
function class $\Gcal \subset \Xcal \to [0,1]$ where we assemble
(policy,\vval function) pairs by taking the Cartesian product of $\Pi$
and $\Gcal$.  \cdplearn can be run here with the understanding that
each Q-value function $f$ in \cdplearn is associated with a $(\pi,g)$
pair, and the algorithm uses $\pi$ instead of $\pi_f$ and $g(x)$
instead of $f(x,\pi_f(x))$.  All the analysis applies directly with
this transformation, and the $\log |\Fcal|$ dependence in sample
complexity is replaced by $\log|\Pi| + \log|\Gcal|$.  Note also that
the definition of \bfac also extends naturally to this case, where the
first argument is the $(\pi,g)$ pair and the second argument is a
roll-in policy, $\pi'$.

\subsection{Infinite Hypothesis Classes}
\label{sec:inf_hyp_class}
The arguments in Section~\ref{sec:algorithm} assume that $|\Fcal|= N <
\infty$. However, almost all commonly used function approximators are
infinite classes, which restricts the applicability of the
algorithm. On the other hand, the size of the function class appears
in the analysis only through deviation bounds, so techniques from
empirical process theory can be used to generalize the results to
infinite classes. This section establishes parallel versions of those
deviation bounds for function classes with finite combinatorial
dimensions, and together with the rest of the original analysis we can
show the algorithm enjoys similar guarantees when working with
infinite hypothesis classes. %To achieve this, we provide a two-stage treatment.

Specifically, we consider the setting where $\Pi$ and $\Gcal$ are
given (see Section~\ref{sec:policy_vval}), and they are infinite
classes with finite combinatorial dimensions. We assume that $\Pi$ has
finite \emph{Natarajan dimension} (Definition~\ref{def:natarajan}),
and $\Gcal$ has finite \emph{pseudo dimension}
(Definition~\ref{def:pdim}). These two dimensions are standard
extensions of VC-dimension to multi-class classification and
regression respectively.

\begin{definition}[Natarajan dimension \citep{natarajan1989learning}]\label{def:natarajan}
Suppose $\Xcal$ is a feature space and $\Ycal$ is a finite label space. Given hypothesis class $\Hcal \subset \Xcal \to \Ycal$, its Natarajan dimension $\textrm{Ndim}(\Hcal)$ is defined as the maximum cardinality of a set $A \subseteq \Xcal$ that satisfies the following: there exists $h_1, h_2: A \to \Ycal$ such that (1) $\forall x\in A$, $h_1(x) \ne h_2(x)$, and (2) $\forall B \subseteq A$, $\exists h\in\Hcal$ such that $\forall x\in B$, $h(x)=h_1(x)$ and $\forall x\in A\setminus B$, $h(x)=h_2(x)$.
\end{definition}

\begin{definition}[Pseudo dimension \citep{haussler1992decision}]\label{def:pdim}
Suppose $\Xcal$ is a feature space. Given hypothesis class $\Hcal \subset \Xcal \to \RR$, its pseudo dimension $\textrm{Pdim}(\Hcal)$ is defined as $\textrm{Pdim}(\Hcal) = \textrm{VC-dim}(\Hcal^+)$, where $\Hcal^+ = \{(x, \xi) \mapsto \textbf{1}[h(x) > \xi] : h \in \Hcal\} \subset \Xcal \times \mathbb{R} \to \{0, 1\}$.
\end{definition}

The definition of pseudo dimension relies on that of VC-dimension, whose definition and basic properties are recalled
in the appendix. We state the final sample complexity result
here. Since the algorithm parameters are somewhat complex expressions,
we omit them in the theorem statement and provide specification
in the proof, which is deferred to Appendix~\ref{app:extensions}.

\begin{theorem}
\label{thm:inf_hyp_class}
Let $\Pi \subset \Xcal \to \Acal$ with $\textrm{Ndim}(\Pi) \le d_{\Pi} < \infty$ and $\Gcal \subset \Xcal \to [0,1]$ with $\textrm{Pdim}(\Gcal) \le d_{\Gcal} < \infty$. 
For any $\epsilon,\delta \in (0,1)$, any \modellong with policy space $\Pi$ and function space $\Gcal$ that admits a \bfac with parameters $M,\normbound$, if we run \cdplearn with appropriate parameters, then with probability at least $1-\delta$, \cdplearn halts and returns a policy $\hat{\pi}$ that satisfies $V^{\hat{\pi}} \ge \vstar - \epsilon$, and the number of episodes required is at most\begin{align}
\otil\left(\frac{M^2H^3K^2}{\epsilon^2}\Big(d_{\Pi} + d_{\Gcal} + \log(\normbound/\delta)\Big)\right).
\end{align}
\end{theorem}

Compared to Theorem~\ref{thm:cdp_complexity}, the sample complexity we
get for infinite hypothesis classes has two differences: (1) $\log N$
is replaced by $d_{\Pi} + d_{\Gcal}$, which is expected, based on the
discussion in Section~\ref{sec:policy_vval}, and (2) the dependence on
$K$ is quadratic as opposed to linear. In fact, in the proof of
Theorem~\ref{thm:cdp_complexity}, we exploited the low-variance
property of importance weights in Eq.~\eqref{eq:estm_berr}, and
applied Bernstein's inequality to avoid a factor of $K$. With infinite
hypothesis classes, the same approach does not apply
directly. However, this may only be a technical issue, and a more
refined analysis may recover a linear dependence (e.g., using
tools from \citet{panchenko2002some}).

\subsection{Approximate Validity and Approximate \brankbig}
\label{sec:robust}
% Recall that the sample-efficiency guarantee of \cdplearn relies on two major assumptions: that (1) the average Bellman errors have an exact low-rank factorization (Definition~\ref{def:brank}), and (2) $\Fcal$ contains valid functions (Definition~\ref{def:valid}). 

Recall that the sample-efficiency guarantee of \cdplearn relies on two major assumptions: 
\begin{itemize}
\item We assumed that $\Fcal$ contains valid functions (Definition~\ref{def:valid}). In practice, however, it is hard to specify a function class that contains strictly  valid functions, as the notion of validity depends on the environment dynamics, which are unknown. A much more realistic situation is that some functions in $\Fcal$ satisfy validity only \emph{approximately}.
\item We assumed that the average Bellman errors have an exact low-rank factorization (Definition~\ref{def:brank}). While this is true for a number of RL models (Section~\ref{sec:brank}), it is worth keeping in mind that these are only \emph{models} of the environments, which are different from and only approximations to the real environments themselves. Therefore, it is more realistic to assume that an \emph{approximate} factorization exists when defining \bfac.
\end{itemize}

In this section, we show that the algorithmic ideas of \cdplearn are indeed robust against both types of approximation errors, and degrades gracefully as the two assumptions are violated. Below we introduce the approximate versions of Definition~\ref{def:valid} and \ref{def:brank}, give a slightly extended version of the algorithm, \cdplearnslack (for Optimism-Led Iterative Value-function Elimination with Robustness, see Algorithm~\ref{alg:slack}), and state its sample complexity guarantee in Theorem~\ref{thm:slack}. 

\begin{definition}[Approximate validity of $f$]
\label{def:approx_valid}
Given any CDP and function class $\Fcal$, we say $f\in\Fcal$ is $\slack$-valid if for any $f'\in\Fcal$ and any $h\in[H]$, 
$
\left|\berr{f, \pi_{f'}, h}\right| \le \slack.
$
\end{definition}

The approximation error $\slack$ introduced in Definition~\ref{def:approx_valid} allows the algorithm to compete against a broader range of functions; hence the notions of optimal function and value need to be re-defined accordingly.
\begin{definition}
\label{def:approx_vstar}
For a fixed $\slack$, define 
$
\fstarslack = \argmax_{f\in\Fcal: \,f \textrm{\,is $\slack$-valid}} \Vpi{\pi_f},
$
and 
$
\vstarslack = \Vpi{\pi_{\fstarslack}}.
$ 
\end{definition}

%\begin{remark}
By definition, $\vstarslack$ is non-decreasing in $\slack$ with Definition~\ref{def:valid} being a special case where $\slack=0$.
%\end{remark}
%\begin{remark}
When $\slack > 0$, we compete against some functions that do not obey
Bellman equations, breaking an essential element of value-based RL. As
a consequence, %we can no longer guarantee to 
returning a policy with value close to
$\vstarslack$ in a sample-efficient manner is very challenging, so the value that \cdplearnslack can guarantee is suboptimal to $\vstarslack$ by a term that is proportional to $\slack$ and does not diminish with more data.

\begin{definition}[Approximate \brank] \label{def:approx_brank}
We say that a CDP $(\Xcal, \Acal, H, P)$ and $\Fcal \subset \Xcal \times \Acal \to [0, 1]$, admits a \bfac with \brankem $M$, norm parameter $\normbound$, and approximation error $\slackM$, % the $\slackM$-approximate \brank  is defined as the minimal integer $M$ that satisfies the following:
if there exists $\pvec_h: \Fcal \to \RR^M, \bvec_h: \Fcal \to \mathbb{R}^M$ for each $h\in[H]$, such that for any $f,f'\in\Fcal, h\in[H]$,
\begin{align}
\left|\berr{f,\pi_{f'},h} - \langle \pvec_h(f'), \bvec_h(f) \rangle\right| \le \slackM,
\end{align}
and $\|\pvec_h(f')\|_2 \cdot \|\bvec_h(f)\|_2 \le \normbound < \infty$. 
\end{definition}

\begin{algorithm}[t]
\begin{algorithmic}[1]
\State Let $\epsilon' = \epsilon + 2H(3\sqrt{M} (\slack + \slackM) + \slackM)$.  \label{lin:def_eps_prime_slack}
\State \textbf{Collect} $\nest$ trajectories %$t^{(1)}, \ldots, t^{(\nest)}$ 
with actions taken in an arbitrary manner; %taking actions arbitrarily
save initial contexts $\{x_1^{(i)}\}_{i=1}^{\nest}$.
%observations$\{x_1^{(i)}\}_{i=1}^{\nest}$. 
\label{lin:init_eval_slack}
\State \textbf{Estimate} the predicted value for each $f \in \Fcal$: $\empVf{f} = \frac{1}{\nest}\sum_{i=1}^{\nest} f(x_1^{(i)}, \pi_f(x_1^{(i)}))$. \label{lin:vf_slack}
\State $\Fcal_0 \gets \Fcal$. \label{lin:init_Fcal_slack}
\For{$t=1,2,\ldots$} \label{lin:forloop_slack}
\State \textbf{Choose policy} $\ft{t} = \argmax_{f \in \Fcal_{t-1}} \empVf{f}$, $\pit{t} = \pi_{\ft{t}}$.
\State \textbf{Collect} $\neval$ trajectories
$\{(x^{(i)}_1,a^{(i)}_1,r^i_1,\ldots,
x^{(i)}_H,a^{(i)}_H,r^{(i)}_H)\}_{i=1}^{\neval}$ by following
$\pit{t}$ (i.e. $a_h^{(i)} = \pit{t}(x_h^{(i)})$ for all $h,i$). \label{lin:eval_slack} 
\State \textbf{Estimate} $\forall h\in[H]$,  \label{lin:detect_slack}
\begin{align}
\dtctberr{\ft{t}, \pit{t}, h} = \frac{1}{\neval}\sum_{i=1}^{\neval}
\left[ \ft{t}(x^{(i)}_h, a^{(i)}_h) - r^{(i)}_h -
  \ft{t}(x^{(i)}_{h+1}, a^{(i)}_{h+1}) \right].
\label{eq:ft_bellman_slack}
\end{align}
\If{$\sum_{h=1}^H \dtctberr{\ft{t}, \pit{t},h} \le 5\epsilon'/8$} \label{lin:check_epsilon_prime_slack}
\State Terminate and ouptut $\pit{t}$. 
\EndIf
\State Pick any $h_t \in [H]$ for which $\dtctberr{\ft{t}, \pi_t,h_t} \ge 5\epsilon'/8H$ (One is guaranteed to exist). \label{lin:detected_slack}
\State Collect trajectories $\{(x_1^{(i)}, a_1^{(i)}, r_1^{(i)},
\ldots, x_H^{(i)}, a_H^{(i)}, r_H^{(i)})\}_{i=1}^{\ntrain}$ where
$a_h^{(i)} = \pit{t}(x_h^{(i)})$ for all $h \ne h_t$ and
$a_{h_t}^{(i)}$ is drawn uniformly at random.  \label{lin:explore_slack}
\State \textbf{Estimate}
\begin{align}
\empberr{f,\pit{t}, h_t} &= \frac{1}{\ntrain}\sum_{i=1}^{\ntrain} \frac{\mathbf{1}[a_{h_t}^{(i)} = \pi_f(x_{h_t}^{(i)})]}{1/K} \Big(f(x_{h_t}^{(i)},a_{h_t}^{(i)}) - r_{h_t}^{(i)} - f(x_{h_{t}+1}^{(i)},\pi_f(x_{h_t+1}^{(i)}))\Big).
\label{eq:estm_berr_slack}
\end{align}
\State \textbf{Learn}
\begin{align}
\Fcal_{t} = \left\{f \in \Fcal_{t-1} : \left|\empberr{f,\pit{t},h_t} \right|\le \phi + \slack \right\}.
\label{eq:learning_step_slack}
\end{align}
\EndFor
\end{algorithmic}
\caption{\cdplearnslack$(\Fcal, \slack, M,\normbound, \slackM, \epsilon, \delta)$}
\label{alg:slack}
\end{algorithm}

A modified version of \cdplearn that deals with these approximation
errors, \cdplearnslack, is specified
in Algorithm~\ref{alg:slack}. Here, we use $\epsilon$ to denote the
component of the suboptimality that diminish as more data is
collected, and the total suboptimality that we can guarantee is
$\epsilon$ plus a term proportional to $\slack$ and $\slackM$ (see
Eq.~\eqref{eq:subopt_slack} in Theorem~\ref{thm:slack}).  The
algorithm is almost identical to \cdplearn except in two places:
(1) it uses $\epsilon'$ (defined on Line~\ref{lin:def_eps_prime_slack}) in
the termination condition (Line~\ref{lin:check_epsilon_prime_slack}) as opposed to $\epsilon$, and (2) it uses a
higher threshold that depends on $\slack$ in
Eq.~\eqref{eq:learning_step_slack} to avoid eliminating $\slack$-valid
functions.

\begin{theorem}
\label{thm:slack}
For any $\epsilon,\delta \in (0,1)$, any \modellong and function class $\Fcal$ that admits a \bfac with parameters $M$, $\normbound$, and $\slackM$, suppose we run \cdplearnslack with any $\theta \in [0, 1]$, and
$\nest, \neval, \ntrain, \phi$ %\akshay{other parameters} 
as specified in Theorem~\ref{thm:cdp_complexity}. 
%\begin{align*}
%\phi & = \frac{\epsilon}{12H\sqrt{M}}, \qquad \nest = \frac{32}{\epsilon^2}\log(6N/\delta), \\
%\neval & = \frac{288H^2}{\epsilon^2}\log\left(\frac{12H^2\log\left(6HM\primalbound\dualbound/\epsilon\right)}{\delta}\right),\\
%\ntrain & = \frac{4608H^2MK}{\epsilon^2}\log\left(\frac{12NHM\log(6HM\primalbound\dualbound/\epsilon)}{\delta}\right),
%\end{align*}
%where
%\begin{align}
%\label{eq:epsilon_prime}
%\epsilon' = \epsilon + 2H(3\sqrt{M} (\slack + \slackM) + \slackM).
%\end{align}
Then with probability at least $1-\delta$, \cdplearnslack halts and returns a policy $\hat{\pi}$ which is at most
\begin{align} \label{eq:subopt_slack}
\epsilon + 8H\sqrt{M}(\slack+\slackM)
\end{align}
suboptimal compared to $\vstarslack$ defined in Definition~\ref{def:approx_vstar}, and the number of episodes required is at most
\begin{align}
\otil\left(\frac{M^2H^3K}{\epsilon^2}\log(N\normbound/\delta)\right).
\end{align}
\end{theorem}

\section{Proofs of Main Results}
\label{sec:sketch}

In this section, we provide the main ideas as well as
the key lemmas involved in proving
Theorem~\ref{thm:cdp_complexity}. We also show how the lemmas are
assembled to prove the theorem.
Detailed proofs of the lemmas are in Appendix~\ref{app:main-lemmas}.

The proof follows an \emph{explore-or-terminate} argument common to
existing sample-efficient RL algorithms. We argue that the
optimistic policy chosen in Line~\ref{lin:opt} of
Algorithm~\ref{alg:simple} is either approximately optimal, or visits
a context distribution under which it has a large Bellman error. This
implies that using this policy for exploration leads to learning on a
new context distribution. For sample efficiency, we then need to
establish that this event cannot happen too many times. This is done
by leveraging the \bfac of the process and arguing that the number of
times an $\epsilon$ sub-optimal policy is found can be no larger than
$\otil(MH)$. Combining with the number of samples collected for every
sub-optimal policy, this immediately yields the PAC learning
guarantee. 

\subsection{Key Lemmas for Theorem~\ref{thm:cdp_complexity}}

We begin by decomposing a policy-loss-like term into the sum of Bellman errors. 

\begin{lemma}[Policy loss decomposition]
\label{lem:bellman_decomposition}
Define $\Vf{f} = \EE [f(x_1, \pi_f(x_1))]$. Then $\forall f: \Xcal\times\Acal \to [0,1]$,
\begin{align}
\Vf{f} - \Vpi{\pi_f} = \sum_{h=1}^H \berr{f,\pi_f,h}.
\label{eq:bellman_decomposition_lem}
\end{align}
\end{lemma}

The structure of this lemma is similar to many existing results in RL that upper-bound the loss of following an approximate value function greedily using the function's Bellman errors (e.g., \citet{singh1994upper}). However, most existing results are inequalities that use max-norm relaxations to deal with mismatch in distributions; hence, they are likely to be loose. 
%Tighter results are available but they are usually hard to interpret and less popular in literature  \citep{munos2007performance}. 
This lemma, on the other hand, is an equality, thanks to the fact that we are comparing $\Vpi{\pi_f}$ to $\Vf{f}$, not $V^\star$. %, hence every term agrees on the distribution of trajectories, the one induced by following $\pi_f$. 
As the remaining analysis shows, this simple equation allows us to relate policy loss (from the LHS) with the average Bellman error (the RHS) that we use to drive exploration.
%The lemma is proved by a simple inductive argument which can be found
%in Appendix~\ref{app:main-lemmas}. 
%The lemma motivates the analysis of
%Bellman error is a crucial object in our analysis. 
In particular, this lemma implies an explore-or-terminate behavior for
the algorithm.

\begin{lemma}[Optimism drives exploration]
\label{lem:optimism_explore}
Suppose the estimates $\empVf{f}$ and $\dtctberr{\ft{t}, \pit{t}, h}$ in
Line~\ref{lin:vf} and \ref{lin:detect} always satisfy
\begin{align}
\label{eq:vf_vpi}
|\empVf{f} - \Vf{f}| \le \epsilon/8, \qquad and \qquad |\dtctberr{\ft{t}, \pit{t}, h} -
\berr{\ft{t}, \pit{t}, h }| \le \frac{\epsilon}{8H}
\end{align}
throughout the execution of the algorithm. Assume further that
$f^\star$ is never eliminated. Then in any iteration $t$, one of the
following two statements holds: 
\begin{itemize}
\item[(i)] the algorithm does not terminate and
\begin{align}
\label{eq:ht}
\berr{\ft{t}, \pit{t}, h_t} \ge \frac{\epsilon}{2H},
\end{align}
\item[(ii)] the algorithm terminates and the output policy
  $\pit{t}$ satisfies $\Vpi{\pit{t}} \ge \vstar -
  \epsilon$.
\end{itemize}
\end{lemma}

The lemma guarantees that the policy $\pi_t$ used at iteration $t$ in
\cdplearn has sufficiently large Bellman error on at least one of the
levels, provided that the two conditions in Equation~\eqref{eq:vf_vpi}
are met. These conditions require that (1) we have reasonably accurate
value function estimates from Line~\ref{lin:init_eval}, and (2) we
collect enough samples in Line~\ref{lin:eval} to form reliable Bellman
error estimates under $f_t$ at each level $h$. The result of
Theorem~\ref{thm:cdp_complexity} can then be obtained using two
further ingredients. First, we need to make sure that the first case
in Lemma~\ref{lem:optimism_explore} does not happen too many
times. Second, we need to collect enough samples in
Lines~\ref{lin:init_eval} and~\ref{lin:eval} to ensure the
preconditions in Equation~\eqref{eq:vf_vpi}. We first establish a
bound on the number of iterations using the \brankem of the problem,
before moving on to sample complexity questions.

\begin{lemma}[Iteration complexity]
\label{lem:vol}
If $\empberr{f, \pit{t}, h_t }$ in Eq.~\eqref{eq:estm_berr} always satisfies
\begin{align}
|\empberr{f, \pit{t}, h_t } - \berr{f, \pit{t}, h_t }| \le \phi
\label{eq:learning_works}
\end{align}
throughout the execution of the algorithm ($\phi$ is the threshold in the elimination criterion), then $f^\star$ is never eliminated. Furthermore, for any particular level $h$, if whenever $h_t=h$ we have
\begin{align}
|\berr{\ft{t},\pit{t}, h_t}| \ge 6\sqrt{M}\phi, \label{eq:large_violation}
\end{align}
then the number of iterations that $h_t=h$ is at most
$
%% \left(\frac{M}{2}\log M + M \log\frac{1}{\phi}\right) / \log\frac{5}{3}.
M\log \left(\frac{\normbound}{2\phi}\right) / \log\frac{5}{3}.
$
\end{lemma}  

Precondition~\eqref{eq:learning_works} simply posits that we collect
enough samples for reliable Bellman error estimation in
Line~\ref{lin:explore}. Intuitively, since $f^\star$ has no
Bellman error, this is sufficient to ensure that it is never
eliminated. Precondition~\eqref{eq:large_violation} is naturally
satisfied by the exploration policies $\pi_t$ given
Lemma~\ref{lem:optimism_explore}. Given this, the above lemma bounds
the number of iterations at which we can find a large Bellman error at
any particular level.

The intuition
behind this claim is most clear in the POMDP setting of
Proposition~\ref{prop:reactive_pomdp}.   %Example~\ref{exm:pomdp}. 
In this case, %the \brank corresponds to the number of hidden states, 
$\pvec_h(f')$ in Definition~\ref{def:brank} corresponds to the distribution over hidden states induced by $\pi_{f'}$ at level $h$.
%Furthermore, the Bellman error on any distribution over contexts is only influenced by the Bellman error of
%the value function on the underlying distribution over hidden states. 
At iteration $t$, the exploration policy $\pi_{\ft{t}}$ induces such a
hidden-state distribution $p = \pvec_{h}(\ft{t})$ at the chosen level
$h = h_t$, which results in the elimination of all functions that
have large Bellman error on $p$. Thanks to the \bfac, %. Letting
%$[\bvec_{h_t}(f)]_s$ be the Bellman error of $f$ on the hidden state $s$, 
this
corresponds to the elimination of all $f$ with a large $|p^\trans
\bvec_{h}(f)|$, where $\bvec_{h}(f)$ is also defined in
Definition~\ref{def:brank}.  In this case, it can be easily shown that
$\bvec_{h}(f) \in [-2,2]^M$, so the space of all such vectors
$\{\bvec_{h}(f): f\in\Fcal\}$ at each level $h$ is originally
contained in an $\ell_\infty$ ball in $M$ dimensions with radius $2$,
and, whenever $h_t=h$, we intersect this set with two parallel
halfspaces. Via a geometric argument adapted from \citet{todd1982minimum}, we show that each such
intersection reduces the volume of the space by a multiplicative factor of $3/5$. 
We also show that the volume is bounded from below, hence volume reduction cannot occur indefinitely. 
Together, these two facts lead to the iteration complexity upper bound
in Lemma~\ref{lem:vol}. The mathematical techniques used here are
analogous to the analysis of the Ellipsoid method in linear
programming. \nan{Rob: a good reference for this sentence?}

Finally, we need to ensure that the number of samples collected in
each of Lines~\ref{lin:init_eval}, \ref{lin:eval},
and~\ref{lin:explore} of \cdplearn can be upper
bounded, which yields the overall PAC learning result in
Theorem~\ref{thm:cdp_complexity}. The next three lemmas present precisely
the deviation bounds required for this argument. The first two
follow from simple applications of Hoeffding's inequality.

\begin{lemma}[Deviation bound for $\empVf{f}$]
\label{lem:devVf}
With probability at least $1-\delta$, 
$$
|\empVf{f} - \Vf{f}| \le \sqrt{\frac{1}{2\nest}\log\frac{2N}{\delta}}
$$
holds for all $f\in\Fcal$ simultaneously. 
Hence, we can set 
$
\nest \ge \frac{32}{\epsilon^2}\log\frac{2N}{\delta}
$ 
to guarantee that $|\empVf{f} - \Vf{f}|\le \epsilon/8$.
\end{lemma}
This controls the number of samples required in
Line~\ref{lin:init_eval}.

\begin{lemma}[Deviation bound for $\dtctberr{\ft{t}, \pit{t}, h}$]
\label{lem:devVht}
For any fixed $\ft{t}$, with probability at least $1-\delta$, 
$$
|\empberr{\ft{t}, \pit{t}, h} - \berr{\ft{t}, \pit{t}, h}| \le 3\sqrt{\frac{1}{2\neval}\log\frac{2H}{\delta}}
$$
holds for all $h\in[H]$ simultaneously. Hence, we can set 
$
\neval \ge \frac{288 H^2}{\epsilon^2}\log\frac{2H}{\delta}
$ 
to guarantee that $|\dtctberr{\ft{t}, \pit{t}, h} - \berr{\ft{t},  \pit{t},
  h}|\le \frac{\epsilon}{8H}$. 
\end{lemma}
This lemma can be seen as the sample complexity at each iteration in
Line~\ref{lin:eval}. Note that no union bound over $\Fcal$ is needed
here, since Line~\ref{lin:eval} only estimates the average Bellman
error for a single function, which is fixed before data is
collected. Finally, we bound the sample complexity of the learning
step.

\begin{lemma}[Deviation bound for $\empberr{f, \pit{t}, h_t }$]
\label{lem:phi}
For any fixed $\pit{t}$ and $h_t$, with probability at least
$1-\delta$,
$$ |\empberr{f,\pit{t}, h_t} - \berr{f,\pit{t}, h_t}| \le
\sqrt{\frac{8K\log\frac{2N}{\delta}}{\ntrain}} +
\frac{2K\log\frac{2N}{\delta}}{\ntrain}
$$ 
holds for all $f\in\Fcal$ simultaneously.  Hence, we can set $
\ntrain \ge \frac{32K}{\phi^2}\log\frac{2N}{\delta} $ to guarantee
that $|\empberr{f,\pit{t}, h_t} - \berr{f,\pit{t}, h_t}|\le
\phi$ as long as $\phi \le 4$.
\end{lemma}
This lemma uses Bernstein's inequality to exploit the small variance of the importance
weighted estimates. 

%% We are now ready to prove Theorem~\ref{thm:cdp_complexity} given these
%% lemmas.

\subsection{Proof of Theorem~\ref{thm:cdp_complexity}}

Suppose the preconditions of Lemma~\ref{lem:optimism_explore} (Eq.~\eqref{eq:vf_vpi}) and Lemma~\ref{lem:vol} (Eq.~\eqref{eq:learning_works}) hold; we show them via concentration inequalities later. Applying Lemma~\ref{lem:optimism_explore}, in every iteration $t$ before the algorithm terminates, 
$$
\berr{\ft{t},  \pi_{t},h_t } \ge \frac{\epsilon}{2H} =  6\sqrt{M}\phi,
$$
due to the choice of $\phi$. For level $h = h_t$, Eq.~\eqref{eq:large_violation} is satisfied. According to Lemma~\ref{lem:vol}, the event $h_t = h$ can happen at most $M\log \left(\frac{\normbound}{2\phi}\right)/ \log\frac{5}{3}$ times for every $h\in[H]$. Hence, the total number of iterations in the algorithm is at most 
$$
H   M\log \left(\frac{\normbound}{2\phi}\right)/ \log\frac{5}{3} = H M\log \left(\frac{6H\sqrt{M} \normbound}{\epsilon}\right) / \log\frac{5}{3}.
$$
Now we are ready to apply the concentration inequalities to show that Eq.~\eqref{eq:vf_vpi} and \eqref{eq:learning_works} hold with high probability. We split the total failure probability $\delta$ among the following estimation events:
\begin{enumerate}
\item Estimation of $\empVf{f}$ (Lemma~\ref{lem:devVf}; only once): $\delta / 3$.
\item Estimation of $\dtctberr{\ft{t},\pit{t},h}$ (Lemma~\ref{lem:devVht}; every iteration):  $\delta / \left( 3H M\log \left(\frac{6H\sqrt{M} \normbound}{\epsilon}\right) / \log\frac{5}{3}\right)$.
\item Estimation of $\empberr{f, \pit{t},h_t}$ (Lemma~\ref{lem:phi}; every iteration): same as above.
\end{enumerate}
Since these events happen in a particular sequence, the proof actually
bounds the probability of these failure events conditioned on all
previous events succeeding. This imposes no technical challenge as
fresh data is collected for every event, so it effectively reduces to
a standard union bound.

Applying Lemmas~\ref{lem:devVf}, \ref{lem:devVht}, and  \ref{lem:phi} with the above failure probabilities, we can verify that the choices of $\nest, \neval,$ and $\ntrain$ in the algorithm statement satisfy the preconditions of Lemmas~\ref{lem:optimism_explore} and \ref{lem:vol}. Finally, we upper bound the total number of episodes as
\begin{align*}
&~ \nest + \neval \cdot H M\log \left(\frac{6H\sqrt{M} \normbound}{\epsilon}\right) / \log\frac{5}{3}  + \ntrain \cdot H M\log \left(\frac{6H\sqrt{M} \normbound}{\epsilon}\right) / \log\frac{5}{3} \\
= &~ \tilde O\left(\frac{\log(N/\delta)}{\epsilon^2} + \frac{MH^3}{\epsilon^2}\log(\normbound/\delta) + \frac{M^2 H^3 K}{\epsilon^2}\log(N\normbound/\delta)\right)
= \tilde O\left(\frac{M^2 H^3 K}{\epsilon^2}\log(N\normbound/\delta)\right).
\tag*{\qedhere}
\end{align*}

\section{Conclusions and Discussions}

In this paper, we presented a new model for RL with rich observations,
called \modellongs, and a structural property, the \bfac, of these
models that enables sample-efficient learning. The unified approach
allows us to address several settings of practical interest that have
largely eluded RL theory to date. Via extensions of the main result,
we also demonstrated that the techniques are quite robust and degrade
gracefully with violation of assumptions. 
\nan{``While the results are an exciting development for RL theory'' Rob: be more modest and let the reader decide if the results are exciting.}
These results also elicit several
further questions:

\begin{enumerate}
  \item Can we obtain a computationally efficient algorithm for some
    form of this setting?  Prior related work (for instance in
    contextual bandits~\citep{dudik2011efficient,agarwal2014taming})
    used supervised learning oracles for computationally efficient
    approaches. Is there a suitable oracle for this setting?
  \item The sample complexity depends polynomially on the cardinality
    of the action space. Can we extend the results to handle large or
    continuous action spaces?
  \item Can we address sample-efficient RL given only a policy class
    rather than a value function class? Empirical approaches in policy
    search often rely on policy gradients, which are subject to local
    optima. Are there parallel results to this work, without
    access to value functions?
\end{enumerate}

Understanding these questions is a key bridge between RL theory and
practice, and is critical to success in challenging reinforcement
learning problems.

\newpage

\appendix
\section*{Appendix}

\section{Lower Bounds}

\subsection{An Exponential Lower Bound}
\label{app:lower_exp}
We include a result from \citet{krishnamurthy2016contextual} to formally show that, without making additional assumptions, the sample complexity of value-based RL for CDPs as introduced in Section~\ref{sec:cdp} has a lower bound of order $K^H$.

\begin{proposition}[Restatement of Proposition 2 in
    \citet{krishnamurthy2016contextual}] \label{prop:exp_lower_bound}
  For any $H, K \in \NN$ with $K\ge 2$, and any $\epsilon \in (0,
  \sqrt{1/8})$, there exists a family of finite-horizon MDPs with
  horizon $H$ and $|\Acal|=K$, such that we can construct a function
  space $\Fcal$ with $|\Fcal| = K^H$ to guarantee that $Q^\star \in
  \Fcal$ for all MDP instances in the family, yet there exists a
  universal constant $c$ such that for any algorithm and any $T \le
  cK^H / \epsilon^2$, the probability that the algorithm outputs a
  policy $\hat \pi$ with $V^{\hat \pi} \ge V^\star - \epsilon$ after
  collecting $T$ trajectories is at most $2/3$ when the problem
  instance is chosen from the family by an adversary.
\end{proposition}

\begin{proof}[Proof sketch for completeness]
The proof relies on the fact that \modelshorts include MDPs where the
state space is arbitrarily large. %, even exponential, in the horizon $H$. 
Each instance of the MDP family is a complete tree with
branching factor $K$ and depth $H$. Transition dynamics are
deterministic, and only leaf nodes have non-zero rewards. All leaves
give Ber$(1/2)$ rewards, except for one that gives
Ber$(1/2+\epsilon)$. Changing the position of the most rewarding leaf
node yields a family of $K^H$ MDP instances so collecting 
optimal Q-value functions forms the desired function class
$\Fcal$. Since $\Fcal$ provides no information other than the fact
that the true MDP lies in this family, the problem is equivalent to
identifying the best arm in a multi-arm bandit with $K^H$ arms, and
the remaining analysis follows exactly the same as in
\citet{krishnamurthy2016contextual}.
\end{proof}

\subsection{A Polynomial Lower Bound that Depends on \brankbig}
\label{app:lower_new}

In this section, we prove a new lower bound for layered episodic MDPs
that meet the assumptions we make in this paper.

We first recall some definitions.  A layered episodic MDP is defined
by a time horizon $H$, a state space $\Scal$, partitioned into sets
$\Scal_1, \ldots, \Scal_H$, each of size at most $M$, and an action
space $\Acal$ of size $K$.  The system descriptor is replaced with a
transition function $\Gamma$ that associates a distribution over
states with each state action pair.  More formally, for any $s_h \in
\Scal_h$, and $a \in \Acal$, $\Gamma(s_h,a) \in \Delta(\Scal_{h+1})$.
The starting state is drawn from $\Gamma_1 \in \Delta(\Scal_1)$, and
all transitions from $\Scal_H$ are terminal.

There is also a reward distribution $R$ that associates a random
reward with each state-action pair. We use $r \sim R(s,a)$ to denote
the random instantaneous reward for taking action $a$ at state $s$.
We assume that the cumulative reward $\sum_{h=1}^H r_h \in [0,1]$,
where $r_h$ is the reward received at level $h^{\textrm{th}}$ as in Assumption~\ref{ass:bounded}.

Observe that this process is a special case of the finite-horizon
\modellong and moreover, with the set of all value functions $\Fcal =
(\Scal \times \Acal \rightarrow [0,1])$, admits a \bfac with \brank at most $M$ (by Proposition~\ref{prop:mdp_finite}).
Thus the upper bounds for PAC learning apply directly to this setting.

We now state the lower bound.
\begin{theorem}
Fix $M \ge 4, H,K \ge 2$ and $\epsilon \in
(0,\frac{1}{48\sqrt{8}})$. For any algorithm and any $n \le
cMKH/\epsilon^2$, there exists a layered episodic MDP with $H$ layers,
$M$ states per layer, and $K$ actions, such that the probability that
the algorithm outputs a policy $\hat{\pi}$ with $V(\hat{\pi}) \ge
V^\star - \epsilon$ after collecting $n$ trajectories is at most
$11/12$. Here $c > 0$ is a universal constant.
\end{theorem}

The result precludes a $o(MKH/\epsilon^2)$ PAC-learning sample
complexity bound since in this case the algorithm must fail with
constant probability. The result is similar in spirit to other lower
bounds for PAC-learning
MDPs~\citep{dann2015sample,krishnamurthy2016contextual}, but we are
not aware of any lower bound that applies directly to the setting.
There are two main differences between this bound and the lower bound
due to~\citet{dann2015sample} for episodic MDPs. First, that bound
assumes that the total reward is in $[0,H]$, so the $H^2$ dependence
in the sample complexity is a consequence of scaling the
rewards. Second, that MDP is not layered, but instead has $M$ total
states shared across all layers. In contrast, the process is
layered with $M$ distinct states per layer and total reward bounded in
$[0,1]$. Intuitively, the additional $H$ dependence arises simply from
having $MH$ total states.

At a high level, the proof is based on embedding $\Theta(MH)$
independent multi-arm bandit instances into a MDP and requiring that
the algorithm identify the best action in $\Omega(MH)$ of them to
produce a near-optimal policy. By appealing to a sample complexity
lower bound for best arm identification, this implies that the
algorithm requires $\Omega(MHK/\epsilon^2)$ samples to identify a
near-optimal policy.

We rely on a fairly standard lower bound for best arm identification.
We reproduce the formal statement
from~\citet{krishnamurthy2016contextual}, although the proof is based
on earlier lower bounds due to~\citet{auer2002nonstochastic}.

\begin{proposition}
For any $K \ge 2$ and $\tau \le \sqrt{1/8}$ and any best arm
identification algorithm that produces an estimate $\hat{a}$, there
exists a multi-arm bandit problem for which the best arm $a^\star$ is
$\tau$ better than all others, but $\PP[\hat{a} \ne a^\star] \ge 1/3$
unless the number of samples $T$ is at least $\frac{K}{72\tau^2}$.
\label{prop:best_arm_id}
\end{proposition}

In particular, the problem instance used in this lower bound is one
where the best arm $a^\star$ has reward $\textrm{Ber}(1/2+\epsilon)$,
while all other arms have reward $\textrm{Ber}(1/2)$.  The
construction embeds precisely these instances into the MDP.

\begin{proof}
We construct an MDP with $M$ states per level, $H$ levels, and $K$
actions per state.  At each level, we allocate three special states,
$w_h, g_h$, and $b_h$ for ``waiting", ``good", and ``bad." The remaining
$M-3$ ``bandit" states are denoted $s_{h,i}, i \in [M-3]$.  Each
bandit state has an unknown optimal action $a_{h,i}^\star$.

The dynamics are as follows.
\begin{itemize}
\item For waiting states $w_h$, all actions are equivalent and with
  probability $1-1/H$ they transition to the next waiting state
  $w_{h+1}$. With the remaining $1/H$ probability, they transition
  randomly to one of the bandit state $s_{h+1,i}$ so each subsequent
  bandit state is visited with probability $\frac{1}{H(M-3)}$.
\item For bandit states $s_{h,i}$, the optimal action $a^\star_{h,i}$
  transitions to the good state $g_{h+1}$ with probability $1/2+\tau$
  and otherwise to the bad state $b_{h+1}$. All other actions
  transition to $g_{h+1}$ and $b_{h+1}$ with probability $1/2$. Here
  $\tau$ is a parameter we set toward the end of the proof.
\item Good states always transition to the next good state and bad
  states always transition to bad states.
\item The starting state is $w_1$ with probability $1-1/H$ and
  $s_{1,i}$ with probability $\frac{1}{H(M-3)}$ for each $i \in
  [M-3]$.
\end{itemize}

The reward at all states except $g_H$ is zero, and the reward at
$g_H$ is one.  Clearly the optimal policy takes actions
$a_{h,i}^\star$ for each bandit state, and takes arbitrary actions at
the waiting, good, and bad states.

This construction embeds $H(M-3)$ best arm identification problems
that are identical to the one used in
Proposition~\ref{prop:best_arm_id} into the MDP. Moreover, these
problems are independent in the sense that samples collected from one
provides no information about any others.  Appealing to
Proposition~\ref{prop:best_arm_id}, for each bandit state $(h,i)$,
unless $\frac{K}{72\tau^2}$ samples are collected from that state, the
learning algorithm fails to identify the optimal action
$a^\star_{h,i}$ with probability at least $1/3$.

After the execution of the algorithm, let $B$ be the set of $(h,s)$
pairs for which the algorithm identifies the correct action.  Let $C$
be the set of $(h,s)$ pairs for which the algorithm collects fewer
than $\frac{K}{72\tau^2}$ samples. For a set $S$, we use $S^C$ to
denote the complement.
\begin{align*}
\EE[|B|] &= \EE\left[\sum_{(h,s)} \mathbf{1}[a_{h,s} = a^\star_{h,s}]\right] \\
&\le ((M-3)H - |C|)  + \sum_{(h,s) \in C} \EE \mathbf{1}[a_{h,s} = a^\star_{h,s}]\\
& \le ((M-3)H - |C|) + \frac{2}{3}|C| = (M-3)H - |C|/3
\end{align*}
The second inequality is based on Proposition~\ref{prop:best_arm_id}. 
Now, by the pigeonhole principle, if $n \le \frac{(M-3)H}{2} \times \frac{K}{72\tau^2}$, then the algorithm can collect $\frac{K}{72\tau^2}$ samples from at most half of the bandit problems.
Thus $|C| \ge (M-3)H/2$, which implies,
\begin{align*}
\EE[|B|] \le \frac{5}{6} (M-3)H
\end{align*}
By Markov's inequality,
\begin{align*}
\PP\left[|B| \ge \frac{11}{12} (M-3)H\right] \le \frac{\EE[|B|]}{\frac{11}{12}(M-3)H} \le \frac{5/6}{11/12} = 10/11
\end{align*}

Thus with probability at least $1/11$ we know that $|B| \le \frac{11}{12}
(M-3)H$, so the algorithm failed to identify the optimal action on
$1/12$ fraction of the bandit problems.
Under this event, the suboptimality of the policy produced by the algorithm is,
\begin{align*}
V^\star - V(\hat{\pi}) &= \PP[\textrm{visit } B^C] \times \tau = \PP[\bigcup_{(h,i) \in B^C} \textrm{visit } (h,i)]\times \tau = \sum_{(h,i) \in B^C} \PP[\textrm{visit } (h,i)]\times \tau\\
& = \sum_{(h,i) \in B^C} \frac{1}{H(M-3)} (1-1/H)^{h-1} \tau \ge \sum_{(h,i) \in B^C} \frac{1}{H(M-3)} (1 - 1/H)^{H} \tau\\
& \ge \sum_{(h,i) \in B^C} \frac{1}{H(M-3)} \frac{1}{4} \tau
\ge \frac{H(M-3)}{12}\frac{1}{H(M-3)} \frac{1}{4} \tau = \frac{\tau}{48}
\end{align*}
Here we use the fact that the probability of visiting a bandit state
is independent of the policy and that the policy can only visit one
bandit state per episode, so the events are disjoint.  Moreover, if we
visit a bandit state for which the algorithm failed to identify the
optimal action, the difference in value is $\tau$, since the optimal
action visits the good state with $\tau$ more probability than a
suboptimal one.  The remainder of the calculation uses the transition
model, the fact that $H \ge 2$, and finally the fact that $|B| \le
\frac{11}{12} (M-3)H$.  Setting $\tau = 48\epsilon$ and using the requirement
on $\tau$ gives a stricter requirement on $\epsilon$ and proves the
result.
\end{proof}

\section{Models with Low \brankbig}
\label{app:low_brank_proof}

\subsection{Proof of Propositon~\ref{prop:mdp_finite}}
\label{app:mdp_finite}
Let $M=|\Scal|$ and each element of $\pvec_h(\cdot)$ and
$\bvec_h(\cdot)$ be indexed by $s\in\Scal$. We explicitly construct
$\pvec_h$ and $\bvec_h$ as follows: let $[\pvec_h(f')]_s = \Pr\,[ x_h
  = (s, h) \,|\, a_{1:h-1} \sim \pi_{f'} ] $, and $[\bvec_h(f)]_s = \EE
\, \big[f(x_h, a_h) - r_h - f(x_{h+1},a_{h+1}) ~\big|~ x_h = (s,
  h),\, a_{h:h+1} \sim \pi_f \big]$. In other words, $\pvec_h(f')$ is
the distribution over states induced by $\pi_{f'}$ at time step $h$,
and the $s$-th element of $\bvec_h$ is the traditional notion of
Bellman error for state $s$. It is easy to verify that
Eq.~\eqref{eq:bellman_decomposition} holds. For the norm constraint,
since $\|\pvec_h(\cdot)\|_1 = 1$ and $\|\bvec_h(\cdot)\|_\infty \le
2$, we have $\|\pvec_h(\cdot)\|_2 \le 1$ and $\|\bvec_h(\cdot)\|_2 \le
2\sqrt{M}$, hence $\normbound= 2\sqrt{M}$ is a valid upper bound on
the product of vector norms.

\subsection{Generalization of \citet{li2009unifying}'s Setting} \label{app:li}
 \citet[Section 8.2.3]{li2009unifying}
 considers the setting where the learner is given an abstraction
 $\phi$ that maps the large state space $\Scal$ in an MDP to some finite
 abstract state space $\bar \Scal$ in an MDP. $|\bar \Scal|$ is potentially much
 smaller than $|\Scal|$, and it is guaranteed that $Q^\star$ can be
 expressed as a function of $(\phi(s), a)$.  \citeauthor{li2009unifying} shows that when
 delayed Q-learning is applied to this setting, the sample complexity
 has polynomial dependence on $|\bar \Scal|$ with no direct
 dependence on $|\Scal|$.

%While the result is for the discounted setting, it is natural to expect that finite-horizon MDPs with a $Q^\star$-irrelevant abstraction can also be sample-efficiently learned. 
In the next proposition, we show that a similar setting for
finite-horizon problems admits \bfac with low \brank. In particular, we subsume
\citeauthor{li2009unifying}'s setting by viewing it as a POMDP, where
$\phi$ is a deterministic emission process that maps hidden state
$s\in\Scal$ to discrete observations $\phi(s) \in \bar\Scal = \Ocal$,
and the candidate value functions are reactive so they depend on
$\phi(s)$ but not directly on $s$ or any previous state. More
generally, Proposition~\ref{prop:qstar-irrelevance} claims that for
POMDPs with large hidden-state spaces and finite observation spaces,
the \brank is polynomial in the number of observations if the function
class is reactive.

\begin{proposition}[A generalization of \citep{li2009unifying}'s setting] \label{prop:qstar-irrelevance}
Consider a POMDP introduced in Example~\ref{exm:pomdp} with $|\Ocal| <
\infty$, and assume that rewards can only take $C_R$ different
discrete values.\footnote{The discrete reward assumption is made to
  simplify presentation and can be relaxed. For arbitrary rewards, we
  can always discretize the reward distribution onto a grid of
  resolution $C_R$, which incurs $\slackM = O(1/C_R)$ approximation
  error in Definition~\ref{def:approx_brank}.} The CDP $(
\Xcal, \Acal, H, P)$ induced by letting $\Xcal = \Ocal \times
     [H]$ and $x_h = (o_h, h)$, with any $\Fcal: \Xcal\times\Acal
     \to [0,1]$, admits a \bfac with $M = |\Ocal|^2 C_R K$ and
     $\normbound = 2|\Ocal| K \sqrt{C_R}$.
\end{proposition}

\begin{proof}%[Proof of Proposition~\ref{prop:qstar-irrelevance}]
For any $f,f'\in\Fcal, h\in[H]$, let $\pvec_h(f')$ and $\bvec_h(f)$ be vectors of length $|\Ocal|^2 C_R K$. Let the entry of $\pvec_h(f')$  indexed by $(o_h, a_h, r_h, o_{h+1})$ be
$$
P [o_h, r_h, o_{h+1} ~|~ a_{1:h-1} \sim \pi_{f'}, a_h],
$$ 
interpreted as the following: conditioned on the fact that the
first $h-1$ actions are chosen according to $\pi_{f'}$, what is the
probability of seeing a particular tuple of $(o_h, r_h, o_{h+1})$ when
taking a particular action for $a_h$? For $\bvec_h(f)$, let the
corresponding entry be (with $x_h = (o_h,h)$ and $x_{h+1} =
(o_{h+1},h+1)$ as the corresponding contexts in the CDP)
$$
\mathbf{1}[a_h = \pi_f(x_h)] \big(f(x_h, a_h) - r_h - f(x_{h+1}, \pi_f(x_{h+1}))\big).
$$
It is not hard to verify that $\berr{f,\pi_{f'},h} = \langle \pvec_h(f'), \bvec_h(f)\rangle$. Since fixing $a_h$ to any non-adaptive choice of action induces a valid distribution over $(o_h, r_h, o_{h+1})$, we have $\|\pvec_h(f')\|_1 = K$ and $\|\pvec_h(f')\|_2 \le K$. On the other hand, $\|\bvec_h(f)\|_\infty \le 2$ but the vector only has $|\Ocal|^2 C_R$ non-zero entries, so $\|\bvec_h(f)\|_2 \le 2|\Ocal| \sqrt{C_R}$. Together the norm bound follows.
\end{proof}

\subsection{POMDP-like Models}
\label{app:pomdp_like}
Here we first state the formal version of Proposition~\ref{prop:mdp_low_rank_dynamics}, and prove Propositions~\ref{prop:mdp_low_rank_dynamics}
and~\ref{prop:reactive_pomdp} together by studying a slightly more general
model (See Figure~\ref{fig:ctx_pomdp}). 

\begin{proposition}[Formal version of Propposition~\ref{prop:mdp_low_rank_dynamics}]
\label{prop:mdp_low_rank_dynamics_formal}
Consider an MDP introduced in Example~\ref{exm:mdp}. With a slight
abuse of notation let $\Gamma$ denote its transition matrix of size
$|\Scal\times\Acal| \times |\Scal|$, whose element indexed by $((s,a),
s')$ is $\Gamma(s'|s,a)$. Assume that there are two row-stochastic
matrices $\Gamma^{(1)}$ and $\Gamma^{(2)}$ with sizes
$|\Scal\times\Acal|\times M$ and $M \times |\Scal|$ respectively, such
that
$
\Gamma = \Gamma^{(1)} \Gamma^{(2)}.
$ 
Recall that we convert an MDP into a CDP by letting $\Xcal = \Scal\times[H]$, $x_h = (s_h, h)$. For any $\Fcal \subset \Xcal\times \Acal\to [0,1]$, this model admits a \bfac with \brank $M$ and $\normbound=2\sqrt{M}$. %the \brank is at most $M$, with $\normbound=2\sqrt{M}$.
\end{proposition}

The model that we use to study Proposition~\ref{prop:mdp_low_rank_dynamics}
and~\ref{prop:reactive_pomdp} simultaneously behaves like a
POMDP except that both the transition function and the reward depends
also on the observation, that is $\Gamma: \Scal \times \Ocal \times
\Acal \rightarrow \Delta(\Scal)$ and $R: \Scal \times \Ocal \times
\Acal \rightarrow \Delta([0,1])$.
Clearly this model generalizes standard POMDPs, where the transition and reward are both assumed to be independent of the current observation.

This model also generalizes the MDP with low-rank dynamics described
in Proposition~\ref{prop:mdp_low_rank_dynamics}: if the future
hidden-state is independent of the current hidden-state conditioned on
the observation (i.e., $\Gamma(s'|s,o,a)$ does not depend on $s$), the observations themselves become Markovian, and we can treat $o$ as the observed state $s$ in Proposition~\ref{prop:mdp_low_rank_dynamics_formal}, and the hidden-state $s$ as the low-rank factor
in Proposition~\ref{prop:mdp_low_rank_dynamics_formal} (see
Figure~\ref{fig:reactive_pomdp}). 
%The low-rank dynamics requires that
%for $x$ at level $h$, $\Gamma(x'|x,a) = \langle q(x,a), w_h(x')
%\rangle$ for two $M$-dimensional vectors.  If we write $[q(x,a)]_{s'}
%= \Gamma(s' | x,a)$ and $[w_h(x')]_{s'} = D_{s'}(x')$ in the above
%model, we have identified an $M$-dimensional decomposition of the
%transition dynamics.
Hence, Proposition~\ref{prop:mdp_low_rank_dynamics}  follows as a special case
of the analysis for this more general model.

As in Proposition~\ref{prop:reactive_pomdp}, we consider a class
$\Fcal$ reactive value functions. Observe that for the MDP with low
rank dynamics, this provides essentially no loss of generality, since
the optimal value function is reactive.

\begin{proposition}
Let $(\Xcal, \Acal, H, P)$ be the CDP induced by the above model which generalizes POMDPs, with $\Xcal = \Ocal \times [H]$ and $x_h = (o_h, h)$. Given any $\Fcal : \Xcal \times \Acal \to [0,1]$, the \brank $M \le |\Scal|$ with $\normbound = 2\sqrt{|\Scal|}$.
\end{proposition}

\begin{proof}
For any $f, f'\in\Fcal, h\in[H]$, consider
$$
a_{1:h-1} \sim \pi_{f'},~~ a_{h:h+1} \sim \pi_f,
$$
which is how actions are chosen in the definition of $\berr{f,\pi_{f'},h}$ (see Definition~\ref{def:berr}). Such a decision-making strategy induces a distribution over the following set of variables
$$
(s_h, o_h, a_h, r_h, o_{h+1}, a_{h+1}).
$$
We use $\mu_{f,f'}$ to denote this distribution, and the subscript emphasizes its dependence on both $f$ and $f'$. Note that the marginal distribution of $s_h$ only depends on $f'$ and has no dependence on $f$, which we denote as $\mu_{f'}$. Then, sampling from $\mu_{f,f'}$ is equivalent to the following sampling procedure:  (recall that $x_h = (o_h, h)$)
\begin{align*}
& s_h \sim \mu_{f'}, ~ o_h \sim D_{s_h},~ a_h = \pi_f(x_h),~ r_h \sim R(s_h, o_h, a_h),\\
& s_{h+1} \sim \Gamma(s_h, o_h, a_h), ~ o_{h+1} \sim D_{s_{h+1}}, a_{h+1} = \pi_f(x_h).
\end{align*}
That is, we first sample $s_h$ from the marginal $\mu_{f'}$, and then sample the remaining variables conditioned on $s_h$. Notice that once we condition on $s_h$, the sampling of the remaining variable has no dependence on $f'$, so we denote the joint distribution over the remaining variables (conditioned on the value of $s_h$) $\mu_{f|s_h}$.

Finally, we express the factorization of $\berr{f,\pi_{f'},h}$ as follows:
\begin{align*}
\berr{f,\pi_{f'},h}
= &~ \EE_{\mu_{f,f'}} [ f(x_h, a_h) - r_h - f(x_{h+1}, a_{h+1}) ] \\
= &~ \EE_{s_h \sim \mu_{f'}} \EE_{\mu_{f|s_h}} [ f(x_h, a_h) - r_h - f(x_h, a_{h+1}) ] \\
= &~ \sum_{s\in\Scal} \mu_{f'}(s) \cdot \EE_{\mu_{f|s}} [ f(x_h, a_h) - r_h - f(x_h, a_{h+1}) ].
\end{align*}
We define $\pvec_h(\cdot)$ and $\bvec_h(\cdot)$ explicitly with dimension $M = |\Scal|$: given $f$ and $f'$, we index the elements of $\pvec_h(f')$ and those of $\bvec_h(f)$ by $s\in\Scal$, and let $[\pvec_h(f')]_s = \mu_{f'}(s)$, $[\bvec_h(f)]_s = \EE_{\mu_{f|s}} [ f(x_h, a_h) - r_h - f(x_h, a_{h+1}) ]$. $\normbound = 2\sqrt{M}$ follows from the fact that $\|\pvec_h(f')\|_1 =1$ and $\|\bvec_h(f)\|_{\infty} \le 2$.
\end{proof}

\subsection{Predictive State Representations} \label{app:psr}
In this subsection we state and prove the formal version of
Proposition~\ref{prop:psr}. We first recall the definitions and some
basic properties of PSRs, which can be found in
\citet{singh2004predictive,boots2011closing}. Consider dynamical
systems with discrete and finite observation space $\Ocal$ and action
space $\Acal$. Such systems can be fully specified by moment matrices
$P_{\Tcal|\Hcal}$, where $\Hcal$ is a set of \emph{histories} (past
events) and $\Tcal$ is a set of \emph{tests} (future events). Elements
of $\Tcal$ and $\Hcal$ are sequences of alternating actions and
observations, and the entry of $P_{\Tcal|\Hcal}$ indexed by
$t\in\Tcal$ on the row and $\tau \in \Hcal$ on the column is
$P_{t|\tau}$, the probability that the test $t$ \emph{succeeds}
conditioned on a particular past $\tau$. For example, if $t=aoa'o'$,
success of $t$ means seeing $o$ and $o'$ in the next two steps after
$\tau$ is observed, if interventions $a$ and $a'$ \emph{were} to be
taken.

Among all such systems, we are concerned about those that have finite
\emph{linear dimension}, defined as $\sup_{\Tcal, \Hcal}
\textrm{rank}(P_{\Tcal|\Hcal})$. As an example, the linear dimension
of a POMDP is bounded by the number of hidden-states. Systems with
finite linear dimension have many nice properties, which allow them to
be expressed by compact models, namely PSRs. In particular, fixing any
$\Tcal$ and $\Hcal$ such that $\textrm{rank}(P_{\Tcal|\Hcal})$ is
equal to the linear dimension (such $(\Hcal, \Tcal)$ are called
\emph{core} histories and \emph{core} tests), we have:
\begin{enumerate}
\item For any history $\tau\in(\Acal\times\Ocal)^*$, the conditional predictions of core tests $P_{\Tcal | \{\tau\}}$ (we also write $P_{\Tcal | \tau}$) is always a \emph{state}, that is, a sufficient statistics of history. This gives rise to the name ``predictive state representation''.
\item Based on $P_{\Tcal |\tau}$, the conditional prediction of any test $t$ can be computed from a PSR model, parameterized by square matrices $\{B_{ao}\}$ and a vector $b_\infty$ with dimension $|\Tcal|$. Letting $t^{(i)}$ be the $i$-th (action, observation) pair in $t$, and $|t|$ be the number of such pairs, the prediction rule is
\begin{align} \label{eq:psr_pred}
P_{t | \tau} = b_{\infty}^\trans B_{t^{(|t|)}} \cdots B_{t^{(1)}} P_{\Tcal |\tau}.
\end{align}
And these parameters can be computed as
\begin{align} \label{eq:psr_learn}
B_{ao} = P_{\Tcal, ao, \Hcal} P_{\Tcal, \Hcal}^\dagger~, \qquad
b_{\infty}^\trans = P_{\Hcal}^\trans P_{\Tcal, \Hcal}^\dagger
\end{align}
where 
\begin{itemize}
\item $P_{\Tcal, \Hcal}$ is a matrix whose element indexed by $(t\in\Tcal, \tau\in\Hcal)$ is $P_{\tau t|\varnothing}$, where $\tau t$ is the concatenation of $\tau$ and $t$ and $\varnothing$ is the null history.
\item $P_{\Hcal} = P_{\{\varnothing\},\Hcal}$.
\item $P_{\Tcal, ao, \Hcal} = P_{\Tcal, \Hcal_{ao}}$, where $\Hcal_{ao} = \{\tau ao: \tau\in\Hcal\}$.
\end{itemize}
\end{enumerate}
Now we are ready to state and prove the formal version of Proposition~\ref{prop:psr}.
%Below we use these properties to prove Proposition~\ref{prop:psr}.

\begin{proposition}[Formal version of Proposition~\ref{prop:psr}] \label{prop:psr_formal}
Consider a partially observable system with observation space $\Ocal$, and the induced CDP $( \Xcal, \Acal, H, P)$  with $x_h = (o_h,h)$. To handle some subtleties, we assume that
\begin{enumerate}
\item $|\Ocal| <\infty$ (classical PSR results assume discrete observations).
\item $o_1$ is deterministic (PSR trajectories always start with an action), and $r_h$ is a deterministic function of $o_{h+1}$ (reward is usually omitted or assumed to be part of the observation).
\end{enumerate}
If the linear dimension of the original system is at most $L$, then with any $\Fcal: \Xcal\times\Acal \to [0,1]$, this model admits a \bfac with $M = LK$. Assuming further that the PSR's parameters are non-negative under some choice of core histories and tests $(\Hcal, \Tcal)$ of size $|\Hcal| = |\Tcal| = L$, then we have $\normbound \le 2 K^2 L^3\sqrt{L} / \sigma_{\min}^3$, where $\sigma_{\min}$ is the minimal non-zero singular value of $P_{\Tcal,\Hcal}$.
\end{proposition}

\begin{proof}%[Proof of Proposition~\ref{prop:psr}]
For any $f,f'\in\Fcal, h\in[H]$, define
\begin{enumerate}
\item $\mu_{f',h}$ as the distribution vector over $(a_1, o_2, \ldots, o_{h-1}, a_{h-1})\in (\Acal\times\Ocal)^{h-2} \times \Acal$ induced by $a_{1:h-1}\sim \pi_{f'}$. (Recall that $o_1$ is deterministic.)
\item $P_{2|h-1}$ as a moment matrix whose element with column index $(o_h, a_h, o_{h+1}) \in \Ocal\times\Acal\times\Ocal$ and \\ row index $(a_1, o_2, \ldots, o_{h-1}, a_{h-1}) \in (\Acal\times\Ocal)^{h-2} \times \Acal$ is 
$$
P [ o_h, o_{h+1} ~\|~ a_{h-1}, a_h ~|~ a_1, o_2, \ldots, o_{h-1}].\footnote{PSR literature often emphasizes the intervention aspect of the actions in tests via the uses ``$\|$'' symbol; mathematically they can be treated as the conditioning operator in most cases.}
$$
\item $F_{f,h}$ as a vector whose element indexed by $(o_h, a_h, o_{h+1}) \in \Ocal\times\Acal\times\Ocal$ is (recall that $x_h = (o_h, h)$ and $r_h$ is function of $o_{h+1}$)
$$
\mathbf{1}[a_h \ne \pi_f(x_h)] \,\big( f(x_h, a_h) - r_h - f(x_{h+1}, \pi_f(x_{h+1})) \big).
$$
\end{enumerate}
First we verify that 
$$
\berr{f, \pi_{f'}, h} = \mu_{f', h}^\trans P_{2|h-1} F_{f,h}.
$$
To show this, first observe that $\mu_{f', h}^\trans P_{2|h-1}$ is a row vector whose element indexed by $(o_h, a_h, o_{h+1})$ is
$$
P [ o_h, o_{h+1} ~\|~ a_h ~|~ a_{1:h-1} \sim \pi_{f'}].
$$
%\begin{align*}
%\sum_{(a_1, o_2, \ldots, o_{h-1}, a_{h-1}')} \Pr [a_1, o_2, \ldots, o_{h-1}, a_{h-1} \,|\, a_{1:h-1} \sim \pi_{f'} ]  \cdot 
%\end{align*}
Multiplied by $F_{f,h}$, we further get
$$
\EE [ f(x_h,a_h) - r_h - f(x_{h+1}, \pi_f(x_{h+1})) ~|~ a_{1:h-1} \sim \pi_{f'}, a_h \sim \pi_f] = \berr{f, \pi_{f'}, h}.
$$
%In the next step we show that the rank of $P_{2|h-1}$ is at most $LK$ by showing that it can be decomposed into $K$ sub-matrices according to the value of $a_{h-1}$ on the row index, each of rank at most $L$. For every $a\in\Acal$, the submatrix that corresponds to $a_{h-1} = a$ can be viewed as a matrix with row index $(a_1, o_2, \ldots, o_{h-1})$, column index $(a_{h-1}, o_h, a_h, o_{h+1})$ ($a_{h-1}$ is restricted to only taking value $a$, while all the other symbols can take from $\Ocal$ and $\Acal$ for observations and actions respectively), and entry
%$$
%P [ o_h, o_{h+1} ~\|~ a_{h-1}, a_h ~|~ a_1, o_2, \ldots, o_{h-1}].
%$$
%Such a matrix is exactly a sub-matrix of the \emph{system-dynamics matrix} defined in PSRs (more precisely, the conditional version in \citet{singh2004predictive}), and by definition its rank is bounded by $L$. Since there are $K$ such row blocks in $P_{2|h-1}$, the rank of $P_{2|h-1}$ is bounded by $LK$, hence $\berr{f, \pi_{f'}, h}$ has a rank-$LK$ factorization. However, we still have to show that in the factorization the left vector has no dependence on $f$ and the right vector has no dependence on $f'$, and the product of their norms has the desired upper bound.

Next, we explicit construct $\bvec_h(f)$ and $\pvec_h(f')$ by  factorizing $P_{2|h-1} = P_1 \times P_2$, where both $P_1$ and $P_2$ have no dependence on either $f$ or $f'$. Recall that for PSRs, any history $(a_1, o_2, \ldots, o_{h-1})$ has a sufficient statistics $P_{\Tcal|a_1, o_2, \ldots, o_{h-1}}$, that is a vector of predictions over the selected core tests $\Tcal$ conditioned on the observed history. $P_1$ consists of row vectors of length $LK$, and for the row indexed by $(a_1, o_2, \ldots, o_{h-1}, a_{h-1})$ the vector is
$$
\textrm{Pad}_{a_{h-1}} \big( P_{\Tcal|a_1, o_2, \ldots, o_{h-1}}^\trans \big),
$$
where $\textrm{Pad}_{a}(\cdot)$ is a function that takes a $L$-dimensional vector, puts it in the $a$-th block of a vector of length $LK$, and fills the remaining entries with $0$.

We construct $P_2$ to be a matrix whose column vector indexed by $(o_h, a_h, o_{h+1})$ is 
$$
\begin{bmatrix}
B_{a^{(1)}, o_h}^\trans B_{a_h, o_{h+1}}^\trans b_{\infty} \\
\ldots \\
B_{a^{(K)}, o_h}^\trans B_{a_h, o_{h+1}}^\trans b_{\infty}
\end{bmatrix},
$$
where $\Acal = \{a^{(1)}, \ldots, a^{(K)}\}$. 
It is easy to verify that $P_{2|h-1} = P_1 \times P_2$ by recalling the prediction rules of PSRs in Eq.~\eqref{eq:psr_pred}:
\begin{align*}
P [ o_h, o_{h+1} ~\|~ a_{h-1}, a_h ~|~ a_1, o_2, \ldots, o_{h-1}]
= &~ b_{\infty}^\trans B_{a_h, o_{h+1}}  B_{a_{h-1}, o_h}  P_{\Tcal|a_1, o_2, \ldots, o_{h-1}} \\
= &~ P_{\Tcal|a_1, o_2, \ldots, o_{h-1}}^\trans  (B_{a_{h-1}, o_h}^\trans B_{a_h, o_{h+1}}^\trans b_{\infty}).
\end{align*}

Given this factorization, we can write
$$
\berr{f, \pi_{f'}, h} = (\mu_{f', h}^\trans P_1) \times (P_2 F_{f,h}).
$$
So we let $\pvec_h(f') = P_1^\trans \mu_{f', h}$ and $\bvec_h(f) =  P_2 F_{f,h}$. It remains to be shown that we can bound their norms. Notice that the entries of a state vector $P_{\Tcal|(\cdot)}$ are predictions of probabilities, so $\|P_1\|_\infty \le 1$. Since $\mu_{f',h}$ is a probability vector, its dot product with every column in $P_1$ is bounded by $1$, hence $\|\pvec_h(f')\|_2 \le \sqrt{LK}$.

At last, we consider bounding the norm of $P_2 F_{f,h}$.  We upper bound each entry of $P_2 F_{f,h}$ by providing an $\ell_1$ bound on the row vectors of $P_2$, and then applying the H\"older's inequality with $\|F_{f,h}\|_\infty \le 2$. 
Since we assumed that all model parameters of the PSRs are non-negative, $P_2$ is a non-negative matrix, and bounding the $\ell_1$ norm of its row vectors is equivalent to bounding each entry of the vector $P_2 \,\mathbf{1}$, where $\mathbf{1}$ is an all-1 vector. This vector is equal to
\begin{align} \label{eq:psr_p2_1norm}
P_2 \,\mathbf{1}
= &~ \begin{bmatrix} 
\sum_{(o_h, a_h, o_{h+1})} B_{a^{(1)}, o_h}^\trans B_{a_h, o_{h+1}}^\trans b_{\infty} \\
\ldots \\
\sum_{(o_h, a_h, o_{h+1})} B_{a^{(K)}, o_h}^\trans B_{a_h, o_{h+1}}^\trans b_{\infty} 
\end{bmatrix} 
=  \begin{bmatrix} 
\left(\sum_{o_h} B_{a^{(1)}, o_h}^\trans \right) \left(\sum_{(a_h, o_{h+1})}  B_{a_h, o_{h+1}}^\trans \right) b_{\infty} \\
\ldots \\
\left(\sum_{o_h} B_{a^{(K)}, o_h}^\trans \right) \left(\sum_{(a_h, o_{h+1})}  B_{a_h, o_{h+1}}^\trans \right) b_{\infty}.
\end{bmatrix}
%\textrm{Rep}_K\left( \sum_{(a_{h-1}, o_h)} B_{a_{h-1}, o_h}^\trans \sum_{(a_h, o_{h+1})} B_{a_h, o_{h+1}}^\trans b_{\infty} \right)
%= &~ \textrm{Rep}_K\left( \left(\sum_{a, o} B_{ao}^\trans\right)^2 b_{\infty} \right).
\end{align}
Since we care about the $\ell_\infty$ norm of this vector, we can bound the $\ell_\infty$ norm of each component vector. 
Using the PSR learning equations, we have
\begin{align*}
\sum_{a,o} B_{ao} = \sum_{a,o} P_{\Tcal, ao, \Hcal} P_{\Tcal, \Hcal}^{\dagger} 
= \left(\sum_{a,o} P_{\Tcal, ao, \Hcal}\right) P_{\Tcal, \Hcal}^{\dagger}.
\end{align*}
Note that for any fixed $a = a^{(i)}$, every entry of $\sum_{o} P_{\Tcal, ao, \Hcal}$ is the probability that the event  $t\in\Tcal$ happens after $h\in\Hcal$ happens with a one step delay in the middle, where $a$ is intervened in that delayed time step. Such entries are predicted probabilities of events, hence lie in $[0,1]$. Consequently, $\|\sum_{a,o} P_{\Tcal, ao, \Hcal}\|_\infty \le K$, and we can upper bound the matrix $\ell_2$ norm by Frobenius norm: $\|\sum_{a,o} P_{\Tcal, ao, \Hcal}\|_2 \le \|\sum_{a,o} P_{\Tcal, ao, \Hcal}\|_F \le K L$. Hence,
\begin{align*}
\left\|\sum_{a,o} B_{ao} \right\|_2 \le 
\left\|\sum_{a,o} P_{\Tcal, ao, \Hcal}\right\|_2 \cdot \left\|P_{\Tcal,\Hcal}^\dagger \right\|_2 
\le KL / \sigma_{\min}.
\end{align*}
Using a similar argument, for any fixed $a=a^{(i)}$, $\left\|\sum_{o} B_{ao} \right\|_2 \le L / \sigma_{\min}$.
We also recall the definition of $b_{\infty}$ and bound its norm similarly:
\begin{align*}
\|b_{\infty}\|_2 = 
\left\|P_{\Hcal}^\trans P_{\Tcal, \Hcal}^\dagger \right\|_2 
\le \sqrt{L} / \sigma_{\min}.
\end{align*}
Finally, we have 
\begin{align*}
\|P_2 \,\mathbf{1}\|_\infty \le &~  
\max_{a\in\Acal} \left\| \left(\sum_{o_h} B_{a, o_h}^\trans \right) \left(\sum_{(a_h, o_{h+1})}  B_{a_h, o_{h+1}}^\trans \right) b_{\infty} \right\|_\infty  \tag{Eq.~\eqref{eq:psr_p2_1norm}}\\
\le &~ \max_{a\in\Acal} \left\| \left(\sum_{o_h} B_{a, o_h}^\trans \right) \left(\sum_{(a_h, o_{h+1})}  B_{a_h, o_{h+1}}^\trans \right) b_{\infty} \right\|_2 \\
\le &~ \left(\max_{a\in\Acal} \left\|\sum_{o} B_{ao} \right\|_2\right) \left\|\sum_{a,o} B_{ao} \right\|_2 \|b_{\infty}\|_2 \le K L^2\sqrt{L} / \sigma_{\min}^3.
\end{align*}
So each row of $P_2$ has $\ell_1$ norm bounded by the above expression. Applying H\"older's inequality we have each entry of $P_2 F_{f,h}$ bounded by $2 K L^2\sqrt{L} / \sigma_{\min}^3$, hence $\|\bvec_h(f) \|_2 = \|P_2 F_{f,h}\|_2 \le 2 L^3 K \sqrt{K} / \sigma_{\min}^3$. Combined with the bound on  $\|\pvec_h(f')\|_2$ the proposition follows.
\end{proof}

\subsection{Linear Quadratic Regulators}
\label{app:lqr}

In this subsection we prove that Linear Quadratic Regulators (LQR)
(See e.g.,~\citet{anderson2007optimal} for a standard reference) admit \bfac with 
low \brank.  We study a finite-horizon, discrete-time LQR, governed
by the equations:
\begin{align*}
x_1 &= \epsilon_0, \qquad
x_{h+1} = Ax_h + B a_h + \epsilon_h, \qquad \textrm{ and } \qquad c_h = x_h^\trans Qx_h + a_h^\trans a_h + \tau_h,
\end{align*}
where $x_h \in \RR^d$, $a_h \in \RR^K$ and the noise variables are
centered with $\EE[\epsilon_h \epsilon_h^\trans ] = \Sigma$, and
$\EE\tau_h^2 = \sigma^2$.  We operate with \emph{costs} $c_h$ and the
goal is to minimize cumulative cost.  We assume that all parameters
$A,B,\Sigma, Q, \sigma^2$ are bounded in spectral norm by some $\Theta
\geq 1$, that $\lambda_{\min}(B^\trans B) \ge \kappa > 0$, and that
$Q$ is strictly positive definite. Other formulations of LQR replace
$a_h^\trans a_h$ in the cost with $a_h^\trans R a_h$ for a positive
definite matrix $R$, which can be accounted for by a change of
variables. Generalization to non-stationary parameters is
straightforward.

This model describes an MDP with continuous state and action spaces,
and the corresponding CDP has context space $\RR^d \times [H]$,
although we always explicitly write both parts of the context in
this section.  It is well known that in a discrete time LQR, the
optimal policy is a non-stationary linear policy $\pi^\star(x,h) =
P_{\star,h}x$ \citep{anderson2007optimal}, where $P_{\star,h} \in \RR^{K\times d}$ is
a $h$-dependent control matrix.  Moreover, if all of the parameters
are known to have spectral norm bounded by $\Theta$ then the optimal
policy has matrices with bounded spectral norm as well, as we see
in the proof.

The arguments for LQR use decoupled policy and value function classes as in Section~\ref{sec:policy_vval}.
We use a policy class and value function class defined below for
parameters $B_1, B_2, B_3$ that we set in the proof.
\begin{align*}
  \Pi &= \{\pi_{\vec{P}}: \pi_{\vec{P}}(x,h) = P_h x,  \vec{P} \in\prod_{i=1}^H \RR^{K\times d}, \|P_h\|_2 \le B_1 \}\\
  \Gcal &= \{f_{\vec{\Lambda},\vec{O}}: f_{\vec{\Lambda},\vec{O}}(x,h) = x^\trans \Lambda_hx + O_h, \vec{\Lambda} \in \prod_{i=1}^H\RR^{d\times d}, \|\Lambda_h\|_2 \le B_2, \vec{O}\in \RR^{H}, |O_h| \le B_3\}
\end{align*}
The policy class consists of linear non-stationary policies, while the
value functions are nonstationary quadratics with constant offset.

\begin{proposition}[Formal version of Proposition~\ref{prop:lqr}]
  \label{prop:lqr_detailed}
  Consider an LQR under the assumptions outlined above.

  Let $\Gcal$ be a class of non-stationary quadratic value functions
  with offsets and let $\Pi$ be a class of linear non-stationary
  policies, defined above.  Then, at level $h$, for any $(\pi,g)$ pair and any roll-in policy $\pi' \in \Pi$, the average Bellman error can be written as
  \begin{align*}
    \berr{g,\pi,\pi',h} = \langle \bvec_h(\pi,g), \pvec_h(\pi') \rangle,
  \end{align*}
  where $\pvec,\bvec \in \RR^{d^2+1}$.  
  If $\Pi,\Gcal$ are defined as above
  with bounds $B_1, B_2,B_3$ and if all problem parameters have
  spectral norm at most $\Theta$, then
  \begin{align*}
    \|\bvec_h(\pi,g)\|_2^2 &\le d\left(B_2 + \Theta + B_1^2 - (\Theta + \Theta B_1)^2B_2\right) + 4B^2_3 + d^2\Theta^2 B^2_2\\
    \|\pvec_h(\pi')\|_2^2 &\le d^{H+1} \Theta (\Theta B_1)^{2H} + 1.
  \end{align*}
  Hence, the problem admits \bfac with \brank at most $d^2+1$ and
  $\normbound$ that is exponential in $H$ but polynomial in all other
  parameters.  Moreover, if we set $B_1, B_2, B_3$ as,
  \begin{align*}
  B_1 = \Theta^2/\kappa, B_2 =
  \left(\frac{6\Theta^6}{\kappa^2}\right)^H\Theta, B_3 =
  \left(\frac{6\Theta^6}{\kappa^2}\right)^HdH\Theta^2,
  \end{align*}
  then the optimal policy and value function belong to $\Pi,\Gcal$
  respectively.
\end{proposition}

We prove the proposition in several components.  First, we study the
relationship between policies and value functions, showing that linear
policies induce quadratic value functions.  Then, we turn to the
structure of the optimal policy, showing that it is linear.  Next, we
derive bounds on the parameters $B_1, B_2, B_3$ which ensure that the
optimal policy and value function belong to $\Pi, \Gcal$.
Lastly, we demonstrate the \bfac. 

The next lemma derives a relationship between linear policies and
quadratic value functions.
\begin{lemma}
\label{lem:lqr_values}
If $\pi$ is a linear non-stationary policy, $\pi_h(x) = P_{\pi,h}x$,
then $\Vpi{\pi}(x,h) = x^\trans \Lambda_{\pi,h}x + O_{\pi,h}$ where
$\Lambda_{\pi,h} \in \RR^{d\times d}$ depends only on $\pi$ and $h$
and $O_{\pi,h} \in \RR$.
These parameters are defined inductively by,
\begin{align*}
  \Lambda_{\pi,H} &= Q + P_{\pi,H}^\trans P_{\pi,H}, \qquad O_{\pi,H} = 0\\
  \Lambda_{\pi,h} &= Q + P_{\pi,h}^\trans P_{\pi,h} + (A+BP_{\pi,h})^\trans \Lambda_{\pi,h+1}(A + BP_{\pi,h})\\
  O_{\pi,h} & = \tr(\Lambda_{\pi,h+1} \Sigma) + O_{\pi,h+1},
\end{align*}
where we recall that $\Sigma$ is the covariance matrix of the
$\epsilon_h$ random variables.
\end{lemma}

\begin{proof}
  The proof is by backward induction on $h$, starting from level $H$. Clearly,
  \begin{align*}
    \Vpi{\pi}(x,H) = x^\trans Qx + \pi_H(x)^\trans \pi_H(x) = x^\trans Qx + x^\trans P_{\pi,H}^\trans P_{\pi,H}x \triangleq x^\trans \Lambda_{\pi,H}x
  \end{align*}
  so $\Vpi{\pi}(\cdot,H)$ is a quadratic function. 

  For the inductive step, consider level $h$ and assume that for all $x, \Vpi{\pi}(x,h+1) = x^\trans \Lambda_{\pi,h+1}x + O_{\pi,h+1}$.
  Then, expanding definitions,
  \begin{align*}
    \Vpi{\pi}(x,h) &= x^\trans Qx + \pi_h(x)^\trans \pi_h(x) +
    \EE_{x'\sim (x,\pi_h(x))}\Vpi{\pi}(x',h+1)\\ 
    & = x^\trans Qx + x^\trans P_{\pi,h}^\trans P_{\pi,h}x +
    \EE_{x'\sim (x,\pi_h(x))} \left[(x')^\trans \Lambda_{\pi,h+1}(x') +
      O_{\pi,h+1}\right]\\ 
    & = x^\trans Qx + x^\trans P_{\pi,h}^\trans P_{\pi,h}x +
    \EE_{\epsilon_h} \left[(Ax + B\pi_h(x) + \epsilon_h)^\trans
      \Lambda_{\pi,h+1}(Ax + B\pi_h(x) + \epsilon_h) +
      O_{\pi,h+1}\right] \\ 
    & = x^\trans Qx + x^\trans P_{\pi,h}^\trans P_{\pi,h}x +
      \EE_{\epsilon_h} \left[(Ax + BP_{\pi,h}x + \epsilon_h)^\trans
        \Lambda_{\pi,h+1}(Ax + BP_{\pi,h}x + \epsilon_h) +
        O_{\pi,h+1}\right] \\ 
    & = x^\trans Qx + x^\trans P_{\pi,h}^\trans P_{\pi,h}x + x^\trans
        (A + BP_{\pi,h})^\trans \Lambda_{\pi,h+1}(A + BP_{\pi,h})x +
        \EE_{\epsilon_h}\epsilon_h^\trans \Lambda_{\pi,h+1}\epsilon_h
        + O_{\pi,h+1}\\ 
    & = x^\trans Qx + x^\trans P_{\pi,h}^\trans P_{\pi,h}x + x^\trans
        (A + BP_{\pi,h})^\trans \Lambda_{\pi,h+1}(A + BP_{\pi,h})x +
        \tr(\Lambda_{\pi,h+1}\Sigma) + O_{\pi,h+1} 
  \end{align*}
  Thus, setting,
  \begin{align*}
    \Lambda_{\pi,h} &= Q + P_{\pi,h}^\trans P_{\pi,h} + (A+BP_{\pi,h})^\trans \Lambda_{\pi,h+1}(A + BP_{\pi,h})\\
    O_{\pi,h} & = \tr(\Lambda_{\pi,h+1} \Sigma) + O_{\pi,h+1}
  \end{align*}
  We have shown that $\Vpi{\pi}(x,h)$ is a quadratic function of $x$.
\end{proof}

The next lemma shows that the optimal policy is linear.
\begin{lemma}
  \label{lem:lqr_opt_params}
  In an LQR, the optimal policy $\pi^\star$ is a non-stationary linear
  policy given by $\pi^\star(x,h) = P_{\star,h}x$, with parameter
  matrices $P_{\star,h} \in \RR^{K \times d}$ at each level $h$.
  The optimal value function $V^\star$ is a non-stationary quadratic
  function given by $V^\star(x,h)= x^\trans \Lambda_{\star,h}x +
  O_{\star,h}$ with parameter matrix $\Lambda_{\star,h} \in
  \RR^{d\times d}$ and offset $O_{\star,h} \in \RR$.  The optimal
  parameters are defined recursively by,
  \begin{align*}
    P_{\star,H} &= 0 \qquad \Lambda_{\star,H} = Q \qquad O_{\star,H} = 0\\
    P_{\star,h} &= (I +B^\trans \Lambda_{\star,h+1}B)^{-1}B^\trans \Lambda_{\star,h+1}A\\
    \Lambda_{\star,h} & = Q + P_{\star,h}^\trans P_{\star,h} + (A+BP_{\star,h})^\trans \Lambda_{\star,h+1}(A+BP_{\star,h})\\
    O_{\star,h} &= \tr(\Lambda_{\star,h+1}\Sigma) + O_{\star,h+1}.
  \end{align*}
\end{lemma}
\begin{proof}
We explicitly calculate the optimal policy $\pi_\star$ and
demonstrate that it is linear.  Then we instantiate these
matrices in Lemma~\ref{lem:lqr_values} to compute the optimal value
function.

For the optimal policy, we use backward induction on $H$.  At the last
level, we have,
\begin{align*}
  \pi^\star(x,H) = \argmin_{a} x^\trans Qx + a^\trans a = 0.
\end{align*}
Recall that we are working with costs, so the optimal policy minimizes
the expected cost.  Thus $P_{\star,H} = 0 \in \RR^{K\times d}$ and
$\pi^\star(x,H)$ is a linear function of $x$.

Plugging into Lemma~\ref{lem:lqr_values} the value function has parameters,
\begin{align*}
  \Lambda_{\star,H} = Q, \qquad O_{\star,H} = 0
\end{align*}

For the induction step, assume that $\pi^\star(x,h+1) =
P_{\star,h+1}x$ is linear and $V^\star(x,h+1)$ is quadratic with
parameter $\Lambda_{\star,h+1} \succ 0$ and $O_{\star,h+1}$.
We then have,
\begin{align*}
  \pi^\star(x,h) &= \argmin_a x^\trans Qx + a^\trans a + \EE_{x' \sim (x,a)} V^\star(x',h+1)\\
  &= \argmin_a x^\trans Qx + a^\trans a + \EE_{\epsilon_h}(Ax+Ba+\epsilon_h)^\trans \Lambda_{\star,h+1}(Ax+Ba+\epsilon_h) + O_{\star,h+1}\\
  &= \argmin_a a^\trans (I + B^\trans \Lambda_{\star,h+1}B)a + 2\langle \Lambda_{\star,h+1}Ax, Ba\rangle
\end{align*}
This follows by applying definitions and eliminating terms that are independent of $a$. 
Since $R,\Lambda_{\star,h+1} \succ 0$ by assumption and using the inductive hypothesis we can analytically minimize. 
Setting the derivative equal to zero gives,
\begin{align*}
  a = (I +B^\trans \Lambda_{\star,h+1}B)^{-1}B^\trans \Lambda_{\star,h+1}Ax
\end{align*}
Thus $P_{\star,h} = (I +B^\trans \Lambda_{\star,h+1}B)^{-1}B^\trans \Lambda_{\star,h+1}A$. 
\end{proof}

As a consequence, we can now derive bounds on the policy and value
function parameters.  Recall that we assume that all system parameters
are bounded in spectral norm by $\Theta \geq 1$ and that $(B^\trans B)^{-1}$
has minimum eigenvalue at least $\kappa$.
\begin{corollary}
  \label{cor:lqr_norm_bounds}
  With $\Theta$ and $\kappa$ defined above, we have
  \begin{align*}
    \|P_{\star,h}\|_f \le \frac{\Theta^2}{\kappa}, \qquad \|\Lambda_{\star,h}\| \le \left(\frac{6\Theta^6}{\kappa^2}\right)^{H-h}\Theta, \qquad  |O_{\star,h}| \le (H-h)\left(\frac{6\Theta^6}{\kappa^2}\right)^{H-h}d\Theta^2.
  \end{align*}
\end{corollary}
\begin{proof}
  Again we proceed by backward induction, using
  Lemma~\ref{lem:lqr_opt_params}.  Clearly $\|P_{\star,H}\|_F = 0$,
  $\|\Lambda_{\star,H}\|_F \le \Theta$, $|O_{\star,H}| = 0$.

 For the inductive step we can actually compute $P_{\star,h}$ without
 any assumption on $\Lambda_{\star,h+1}$, except for the fact that it
 is symmetric positive definite, which follows from
 Lemma~\ref{lem:lqr_opt_params}.  First, we consider just the matrix
 $B^\trans \Lambda_{\star,h+1}A$. Diagonalizing $\Lambda_{\star,h+1} =
 U^\trans DU$ where $U$ is orthonormal and $D$ is diagonal, gives,
  \begin{align*}
    B^\trans \Lambda_{\star,h+1}A &= (UB)^\trans D (UA) = (UB)^\trans D
    (UB)(B^\trans U^\trans UB)^{-1}(UB)^\trans (UA)\\ & = (UB)^\trans D(UB)(B^\trans B)^{-1}B^\trans A =
    B^\trans \Lambda_{\star,h+1} \Pi_B A
  \end{align*}
  Here $\Pi_B = B(B^\trans B)^{-1}B^\trans $ is an orthogonal projection operator.
  This derivation uses the fact that since $(UB)^\trans D$ has rows in the
  column space of $UB$, we can right multiply by the projector onto
  $UB$.  We also use that $U^\trans U = I$ since $U$ has orthonormal rows
  and columns.

  Thus, by the submultiplicative property of spectral norm, we obtain
  \begin{align*}
    \|(I + B^\trans \Lambda_{\star,h+1} B)^{-1}B^\trans \Lambda_{\star,h+1} A\|_2 & \le \|(I+B^\trans \Lambda_{\star,h+1} B)^{-1}B^\trans \Lambda_{\star,h+1} B\|_2\|(B^\trans B)^{-1}B^\trans A\|_2\\
    & \le \|(B^\trans B)^{-1}B^\trans A\|_2 \le \Theta^2/\kappa
  \end{align*}
  Here $\kappa$ is a lower bound on the minimum eigenvalue of $B^\trans B$. 

  Using this bound on $\|P_{\star,h}\|$, we can now bound the optimal value function,
  \begin{align*}
    \|\Lambda_{\star,h}\| \le \Theta + \Theta^4/\kappa^2 + (\Theta + \Theta^3/\kappa)^2\|\Lambda_{\star,h+1}\|
    & \le 6 \Theta^6/\kappa^2 \|\Lambda_{\star,h+1}\|
  \end{align*}
  The last bound uses the fact we apply a bound for
  $\|\Lambda_{\star,h+1}\|_2$ that is larger than one, so the last
  term dominates.  We also use the inequalities $\Theta^2/\kappa \geq 1$
  and $\Theta \geq 1$. This recurrence yields,
  \begin{align*}
    \|\Lambda_{\star,h}\|_2 \le \left(\frac{6\Theta^6}{\kappa^2}\right)^{H-h}\Theta.
  \end{align*}
  A naive upper bound on $O_{\star,h}$ gives,
  \begin{align*}
    O_{\star,h} \le \|\Lambda_{\star,h+1}\|\tr(\Sigma) + |O_{\star,h+1}| \le (H-h)\left(\frac{6\Theta^6}{\kappa^2}\right)^{H-h}d\Theta^2. \tag*{\qedhere}
  \end{align*}
\end{proof}

The final component of the proposition is to demonstrate the \bfac.

\begin{proof}[Proof of Proposition~\ref{prop:lqr_detailed}]
Fix $h$ and a value function $g$ parametrized by matrices $\Lambda$
and offset $O$ at time $h$ and $\Lambda',O'$ at time $h+1$. Also fix
$\pi$ which uses operator $P_\pi$ at time $h$.
\begin{align*}
  \berr{\pi,g,\pi',h} &= \EE_{x\sim (\pi',h)} x^\trans \Lambda x + O - x^\trans Qx - x^\trans P_\pi^\trans P_\pi x - \EE_{x' \sim (x,\pi(x))} (x')^\trans \Lambda'x' + O'\\
  & = \EE_{x \sim (\pi',h)}x^\trans \Lambda x + O - x^\trans Qx - x^\trans P_\pi^\trans P_\pi x - \EE_{\epsilon} (Ax + BP_\pi x + \epsilon)^\trans \Lambda'(Ax + BP_\pi x + \epsilon) + O'\\
  & = \tr\left[\left(\Lambda - Q - P_\pi^\trans P_\pi - (A+BP_\pi)^\trans \Lambda'(A+BP_\pi)\right) \EE_{x \sim (\pi',h)} xx^\trans \right] + O-O'-\tr(\Lambda\Sigma)
\end{align*}
Thus we may write $\bvec_h(\pi,g) = \textrm{vec}(\Lambda - Q -
P_\pi^\trans P_\pi - (A+BP_\pi)^\trans \Lambda'(A+BP_\pi))$ in the
first $d^2$ coordinates and $O-O'-\tr(\Lambda\Sigma)$ in the last
coordinate. We also write $\pvec_h(\pi') = \textrm{vec}(\EE_{x\sim
  (\pi',h)} xx^\trans )$ in the first $d^2$ coordinates and $1$ in the
last coordinate.

The norm bound on $\bvec$ is straightforward, since all terms in its
decomposition have an exponential in $H$ bound.

For $\pvec$, since the distribution is based on applying a bounded
policy $\pi'$ at level $h-1$ iteration, we can write $x =
A\tilde{x}+BP_{\pi'}\tilde{x} + \epsilon$ where $\tilde{x}$ is
obtained by rolling in with $\pi'$ for $h-1$ steps.  If $(\pi',h-1)$
denotes the distribution at the previous level, this gives,
\begin{align*}
  \|\EE_{x\sim (\pi',h)} xx^\trans \|_F &\le \|\Sigma\|_F + \tr\left((A+BP)^\trans (A+BP)\EE_{\tilde{x} \sim (\pi',h-1)} \tilde{x}\tilde{x}^\trans \right)\\
  & \le \|\Sigma\|_F + d(\Theta + \Theta B_1)^2 \|\EE_{\tilde{x} \sim (\pi',h-1)}\tilde{x}\tilde{x}^\trans \|_F
\end{align*}
Since at level one we have that the norm is at most $\|\Sigma\|_F$, we obtain a recurrence which produces a bound, at level $h$ of,
\begin{align*}
  \|\EE_{x \sim (\pi',h)} xx^\trans \|_F &\le \|\Sigma\|_F\sum_{i=1}^h d^{i-1}(\Theta + \Theta B_1)^{2(i-1)}
  \le \|\Sigma\|_F Hd^H(\Theta B_1)^{2H}
\end{align*}
if $\Theta, B_1 \ge 1$, which is the regime of interest. 
\end{proof}

\section{Auxiliary Proofs of the Main Lemmas}
\label{app:main-lemmas}

In this appendix we give the full proofs of the lemmas sketched in
Section~\ref{sec:sketch}. Rather than directly analyze \cdplearn and
prove Theorem~\ref{thm:cdp_complexity}, we instead focus the analysis
on the robust variant, \cdplearnslack, introduced in
Section~\ref{sec:robust}. \cdplearnslack (Algorithm~\ref{alg:slack})
with parameters $\slack = 0$ and $\slackM = 0$ is precisely \cdplearn,
and the two analyses are identical. To avoid repetition, in this
appendix we analyze \cdplearnslack (Algorithm~\ref{alg:slack}) and
prove the versions of the lemmas that can be used for
Theorem~\ref{thm:slack}. Readers can easily recover the detailed
proofs of the lemmas in Section~\ref{sec:sketch} for \cdplearn by
letting $\slack=0,~ \slackM = 0,~ \epsilon' = \epsilon,~
\fstarslack=f^\star,~ \vstarslack = \vstar$.

%% While we ought to analyze \cdplearn
%% (Algorithm~\ref{alg:simple}), the algorithm and its analysis are
%% completely subsumed by the robust variant \cdplearnslack introduced in
%% Section~\ref{sec:robust}: by letting $\slack = 0$ and $\slackM = 0$,
%% \cdplearnslack (Algorithm~\ref{alg:slack}) becomes precisely
%% \cdplearn, and the analyses are the same.  

To facilitate understanding we break up the proofs into 3 parts. The
main proofs appear in \ref{app:main-proofs}, and two types of
technical lemmas are invoked from there: (1) a series of lemmas
that adapt the work of \citet{todd1982minimum} for the purpose, which
are given in \ref{app:todd}; (2) deviation bounds, which are given in
\ref{app:dev_bounds}.

\subsection{Main Proofs} \label{app:main-proofs}
\textbf{Lemma}~~(\emph{Restatement of Lemma~\ref{lem:bellman_decomposition}
  from main text for convenience}) With $\Vf{f} = \EE [f(x_1,\pi_f(x_1))]$, we have
%\label{lem:bellman_decomposition_app}
%Fixing a CDP $\langle \Xcal, \Acal, H, P\rangle$, define $\Vf{f} = \EE [f(x_1, \pi_f(x_1))]$. Then $\forall f: \Xcal\times\Acal \to [0,1]$,
\begin{align}
\Vf{f} - \Vpi{\pi_f} = \sum_{h=1}^H \berr{f,\pi_f,h}.
\label{eq:bellman_decomposition_app}
\end{align}
%\end{lemma}

\begin{proof}
Recall from Definition~\ref{def:berr} that the average Bellman errors are defined as
\begin{align*}
\berr{f,\pi,h} = \EE \, \big[f(x_h, a_h) - r_h - f(x_{h+1},a_{h+1}) ~\big|~ a_{1:h-1} \sim \pi,~ a_{h:h+1} \sim \pi_f \big].
\end{align*}
Expanding RHS of Eq.~\eqref{eq:bellman_decomposition_app}, we get
\[
\sum_{h=1}^H \EE \, \big[f(x_h, a_h) - r_h - f(x_{h+1},a_{h+1}) ~\big|~ a_{1:h-1} \sim \pi_f,~ a_{h:h+1} \sim \pi_f \big].
\]
Since all $H$ expected values share the same distribution over trajectories, which is the one induced by $a_{1:H} \sim \pi_f$,  the above expression is equal to
\begin{align*}
&~ \sum_{h=1}^H \EE \, \big[f(x_h, a_h) - r_h - f(x_{h+1},a_{h+1}) ~\big|~ a_{1:H} \sim \pi_f \big] \\
= &~ \EE \, \left[ \sum_{h=1}^H  \Big( f(x_h, a_h) - r_h - f(x_{h+1},a_{h+1}) \Big) ~\big|~ a_{1:H} \sim \pi_f \right] \\
= &~ \EE \, \big[ f(x_1, \pi_f(x_1)) \big]  -  \EE \, \big[ r_h  ~\big|~ a_{1:H} \sim \pi_f \big] = \Vf{f} - \Vpi{\pi_f}. \qedhere
\end{align*}
\end{proof}

\begin{lemma}[Optimism drives exploration, analog of Lemma~\ref{lem:optimism_explore}]
\label{lem:optimism_explore_app}
If the estimates $\empVf{f}$ and $\dtctberr{\ft{t}, \pit{t}, h}$ in Line~\ref{lin:vf_slack} and \ref{lin:detect_slack} of Algorithm~\ref{alg:slack} always satisfy 
\begin{align}
\label{eq:vf_vpi_app}
|\empVf{f} - \Vf{f}| \le \epsilon'/8, \qquad |\dtctberr{\ft{t},  \pit{t}, h} - \berr{\ft{t}, \pit{t}, h}| \le \frac{\epsilon'}{8H}
\end{align}
throughout the execution of the algorithm (recall that $\epsilon'$ is defined on Line~\ref{lin:def_eps_prime_slack}), and $\fstarslack$ is never eliminated, then in any iteration $t$, either the algorithm does not terminate and
\begin{align}
\label{eq:ht_app}
\berr{\ft{t}, \pit{t}, h_t} \ge \frac{\epsilon'}{2H},
\end{align}
or the algorithm terminates and the output policy $\pit{t}$ satisfies $\Vpi{\pit{t}}  \ge \vstarslack - \epsilon' - H\slack$.
\end{lemma}
\begin{proof}
Eq.~\eqref{eq:ht_app} follows directly from the termination criterion
and Eq.~\eqref{eq:vf_vpi_app}. Suppose the algorithm terminates in
iteration $t$. Let $f_{\max}:= \argmax_{f\in\Fcal_{t-1}} \Vf{f}$, and we
have
\begin{align*}
\Vpi{\pit{t}}  = &~ \Vf{\ft{t}} - \sum_{h=1}^H \berr{\ft{t}, \pit{t}, h} \tag{Lemma~\ref{lem:bellman_decomposition}} \\
\ge &~ \empVf{\ft{t}} - \sum_{h=1}^H \dtctberr{\ft{t}, \pit{t}, h} - \epsilon'/4  \tag{Eq.~\eqref{eq:vf_vpi_app}} \\
\ge &~ \empVf{\ft{t}} -  7\epsilon'/8  \tag{termination criterion} \\
\ge &~ \empVf{f_{\max}} - 7\epsilon'/8 \tag{$\ft{t}$ is the maximizer of $\empVf{f}$} \\
\ge &~ \Vf{f_{\max}} - \epsilon' \ge \Vf{\fstarslack} - \epsilon' \tag{$\fstarslack$ is not eliminated} \\
\ge &~ \vstarslack - H\slack - \epsilon'. \tag*{(Lemma~\ref{lem:bellman_decomposition}) \qedhere}
\end{align*}
The last inequality uses Lemma~\ref{lem:bellman_decomposition} on
$\Vf{\fstarslack}$ and the definition of $\vstarslack$, which is the
reward for policy
$\pi_{\fstarslack}$. Lemma~\ref{lem:bellman_decomposition} relates
these two quantities to the average Bellman errors, which, since
$\fstarslack$ is $\slack$-valid are each upper bounded by $\slack$.
\end{proof}

\begin{lemma}[Volumetric argument, analog of Lemma~\ref{lem:vol}]
\label{lem:vol_app}
If $\empberr{f, \pit{t}, h_t}$ in Eq.~\eqref{eq:estm_berr_slack} always satisfies
\begin{align}
|\empberr{f, \pit{t}, h_t} - \berr{f,\pit{t}, h_t}| \le \phi
\label{eq:learning_works_app}
\end{align}
throughout the execution of the algorithm ($\phi$ is the threshold in the elimination criterion), then $\fstarslack$ is never eliminated. Furthermore, for any particular level $h$, if whenever $h_t=h$, we have
\begin{align}
|\berr{\ft{t}, \pit{t}, h_t }| \ge 3\sqrt{M}(2\phi + \slack+\slackM) + \slackM, \label{eq:large_violation_app},
\end{align}
then the number of iterations that $h_t=h$ is at most
\begin{align}
%% \left(\frac{M}{2}\log M + M \log\frac{1}{\phi}\right) / \log\frac{5}{3}.
M\log \frac{\normbound}{2\phi} / \log\frac{5}{3}.
%\left(\frac{\primalbound\dualbound\circumscribe}{2\phi\inscribe}\right) / \log\frac{5}{3}.
\end{align}
\end{lemma}
\begin{proof}
The first claim that $\fstarslack$ is never eliminated follows
directly from the fact $|\berr{\fstarslack,\pit{t}, h_t}| \le \slack$
(Definition~\ref{def:approx_valid}), Eq.~\eqref{eq:learning_works_app},
and the elimination threshold $\phi + \slack$. Below we prove the
second claim.

For any particular level $h$, suppose $i_1 < \cdots < i_{\tick}
<\cdots < i_{\Th}$ are the iteration indices with $h_t=h$, $\{t: h_t =
h\}$ ordered from first to last, and $\Th = |\{t: h_t = h\}|$. For
convenience define $i_0 = 0$. The goal is to prove an upper bound on
$\Th$.

Define notations:
\begin{itemize}
\item $p_1, \ldots, p_{\Th}$. $p_{\tick} := \pvec_h(\ft{i_{\tick}})$ where $\pvec_h(\cdot)$ is given in Definition~\ref{def:approx_brank}. Recall that $\ft{i_{\tick}}$ is the optimistic function used for exploration in iteration $t=i_{\tick}$.
\item $\qFht{0}, \qFht{1}, \ldots, \qFht{\Th}$. $\qFht{\tick} = \{\bvec_h(f): f\in\Fht{\tick}\}$ where $\bvec_h(f) \in \RR^M$ is given in Definition~\ref{def:approx_brank}. 
\item $\primalbound = \sup_{f\in\Fcal} \|\pvec_h(f)\|_2$, and $\dualbound = \sup_{f\in\Fcal} \|\bvec_h(f)\|_2$. By Definition~\ref{def:approx_brank}, $\primalbound \cdot \dualbound \le \normbound$.
\item $V_0, V_1, \ldots, V_{\Th}$. $V_0:= \{v: \|v\|_2 \le \dualbound\}$, and 
$
V_{\tick} := \{v\in V_{\tick-1}: |p_{\tick}^\trans v| \le 2\phi + \slack + \slackM\}.
$
\item $B_0, B_1, \ldots, B_{\Th}$. $B_{\tick}$ is a minimum volume enclosing ellipsoid (MVEE) of $V_{\tick}$.
\end{itemize} 
%We have the following properties regarding the relationship among $\berr{\Fcal_t}, V_t, B_t$ for every $t=0, \ldots, T$:
%\begin{enumerate}
%\item $\berr{\Fcal_{t+1}} \subseteq \berr{\Fcal_{t}}, V_{t+1} \subseteq V_{t}$.
%\item $\berr{\Fcal_t} \subseteq V_t \subseteq B_t$.
%\item (Minimal volume) $[-2\phi, 2\phi]^M \subseteq V_t$.
%\end{enumerate}
For every $\tick=0, \ldots, \Th$, we first show that $\qFht{\tick} \subseteq V_{\tick}$. When $\tick = 0$ this is obvious. For $\tick \ge 1$,  we have $\forall f\in\Fht{\tick}$, 
$$
|\berr{f, \pi_{\ft{i_{\tick}}}, h}| \le 2\phi + \slack.
$$
by the elimination criterion and Eq.~\eqref{eq:learning_works_app}. 
By Definition~\ref{def:approx_brank}, this implies that, $\forall v\in\qFht{\tick}$, 
\[
|p_{\tick}^\trans v| \le 2\phi + \slack +\slackM,
\]
so $\qFht{\tick} \subseteq V_{\tick}$.  
%On the other hand, $\berr{f^\star, D(\pi_{\ft{t}})}= 0$ so $|\empberr{f^\star, D(\pi_{\ft{t}})}| \le \phi$, hence $f^\star$ is never eliminated.

Next we show that $\exists v\in V_{\tick-1}$ such that $|p_{\tick}^\trans v |\ge 3\sqrt{M}(2\phi + \slack+\slackM)$. In fact, Eq.~\eqref{eq:large_violation_app} and the fact that $\ft{i_t}$ was chosen (implying that it survived) implies that this $v$ can be chosen as
$$
v = \bvec_h(\ft{i_{\tick}}) \in\Ucal(\Fcal_{i_{\tick}-1}) \subseteq \qFht{\tick-1} \subseteq V_{\tick-1}.
$$
(The first ``$\subseteq$'' follows from the fact that $\Fcal_t$ shrinks monotonically in Algorithm~\ref{alg:slack}, since the learning steps between $t=i_{\tick-1}+1$ and $t=i_{\tick}-1$ on other levels can only eliminate functions.) We verify that this $v$ satisfies the desired property, given by Definition~\ref{def:approx_brank} and Eq.~\eqref{eq:large_violation_app}:
\begin{align*}
|p_{\tick}^\trans v | = |\langle \pvec_h(f_{i_{\tick}}), \bvec_h(\ft{i_{\tick}}) \rangle|
\ge |\berr{f_{i_{\tick}}, \pi_{i_{\tick}}, h} | -\slackM \ge 3\sqrt{M}(2\phi + \slack+\slackM).
\end{align*}

Observing that $V_t$ is centrally symmetric and consequently so is $B_t$ \citep{todd2007khachiyan}, we apply Lemma~\ref{lem:todd_use} and Fact~\ref{lem:0.6} with the variables set to $d:= M, B:= B_{\tick-1}, \kappa:=3\sqrt{M}(2\phi + \slack + \slackM), \cutto:=2\phi + \slack + \slackM$. We obtain that
$$
\frac{vol(B_+)}{vol(B_{t-1})} \le 0.6,
$$
where $B^+$ is the MVEE of $V_{\tick}':=\{v \in B_{{\tick}-1}: |p_{\tick}^\trans v | \le 2\phi+\slack+\slackM\}$. Note that $V_{\tick} = \{v\in V_{{\tick}-1}: |p_{\tick}^\trans v| \le 2\phi + \slack+ \slackM\} \subseteq V_{\tick}'$ given that $V_{{\tick}-1} \subseteq B_{{\tick}-1}$.
Since $B_+$ is \emph{an} enclosing ellipsoid of $V_{\tick}$, and $B_{\tick}$ is the MVEE of $V_{\tick}$, we have $vol(B_{\tick}) \le vol(B_+)$. Altogether we claim that
$$
\frac{vol(B_{\tick})}{vol(B_{\tick-1})} \le 0.6.
$$

This result shows that the volume of $B_{\tick}$ shrinks exponentially with $\tau$. To prove that $\Th$ is small, it suffices to show that the volume of $B_0$ is not too large, and that of $B_{\Th}$ is not too small. Let $c_M$ be the volume of Euclidean sphere with unit radius in $\mathbb{R}^M$. 
By definition, $vol(B_0) = c_M (\dualbound)^M$. 

For $vol(B_{\Th})$, since $\|p_{\tick}\|_2 \le \primalbound$ always holds, we can guarantee that 
\begin{align*}
V_T \supseteq &~ \left\{q \in \RR^M: \bigcap_{p \in \RR^M: \|p\|_2 \le \primalbound} |\langle p,q\rangle| \le 2\phi + \slack + \slackM \right\} \\
\supseteq &~ \left\{q \in \RR^M: \|q\|_2 \le (2\phi + \slack + \slackM) / \primalbound \right\}  \tag{H\"older's inequality}\\
\supseteq &~ \left\{q \in \RR^M: \|q\|_2 \le 2\phi / \primalbound \right\}.
%B(\|\cdot\|_\star, 2\phi/\primalbound) \supseteq B(\|\cdot\|_2, 2\phi\inscribe/\primalbound).
\end{align*}
Hence,  $vol(B_{\Th}) \ge c_M \left(2\phi/\primalbound\right)^M$, and
%% we know that $[-2\phi, 2\phi]^M \subseteq V_T$, so
%% $$
%% vol(B_{\Th}) \ge vol(V_{\Th}) \ge c_M (2\phi)^M.
%% $$
%Here we relaxed the a hypercube to its inscribed hypersphere in order to cancel the constant $c_M$ in $vol(B_0)$. 
$$
\frac{c_M (2\phi /\primalbound)^M}{c_M (\dualbound)^M} \le \frac{vol(B_{\Th})}{vol(B_0)} = \prod_{t=1}^{\Th} \frac{vol(B_t)}{vol(B_{t-1})} \le 0.6^{\Th}.
$$
Algebraic manipulations give
\begin{align*}
%% \frac{M}{2} \log M +　M \log \frac{1}{\phi}  \ge {\Th} \log \frac{5}{3}. \qedhere
M \log\left( \frac{\primalbound\dualbound}{2\phi}\right) \ge {\Th}\log\frac{5}{3}.
\end{align*}
The second claim of the lemma statement follows by recalling that $\primalbound \dualbound \le \normbound$.
\end{proof}

\subsection{Lemmas for the Volumetric Argument} \label{app:todd}

We adapt the work of \citet{todd1982minimum} to derive lemmas that we
use in \ref{app:main-proofs}. The main result of this section is
Lemma~\ref{lem:todd_use}. As this section focuses on generic geometric
results, we adopt notation more standard for these arguments unlike
the notation used in the rest of the paper.

\begin{theorem}[Theorem 2 of \citet{todd1982minimum}]
\label{thm:todd}
Define $E = \{w \in \RR^d : w^\trans w \le 1\}$ and $E_{\beta} = \{w \in E: |e_1^\trans w| \le \beta\}$ for $0 < \beta \le d^{-1/2}$.
The ellipsoid,
\begin{align}
E_+ = \{w \in \RR^d \,|\, w^\trans  (\rho(I - \sigma e_1e_1^\trans ))^{-1} w \le 1\},
\end{align}
is a minimum volume enclosing ellipsoid (MVEE) for $E_\beta$ if
\begin{align*}
\sigma = \frac{1-d\beta^2}{1-\beta^2} \qquad \textrm{and} \qquad \rho = \frac{d(1-\beta^2)}{d-1}.
\end{align*}
\end{theorem}

\begin{fact}
\label{fact:todd}
With $E,E_+,\sigma,\rho$ as in Theorem~\ref{thm:todd}, we have
\begin{align}
\frac{\textrm{Vol}(E_+)}{\textrm{Vol}(E)} = \sqrt{d} \beta\left(\frac{d}{d-1}\right)^{(d-1)/2}\left(1 - \beta^2\right)^{(d-1)/2}.
\end{align}
\end{fact}
\begin{proof}
For convenience, let us define $\elip = \rho(I - \sigma e_1e_1^\trans )$ so that $E_+ = \{w \in \RR^d : w^\trans  \elip^{-1} w \le 1\}$. 
Notice that $E$ can be obtained from $E_+$ by the affine transformation $v = \elip^{-1/2} w$, which means that if $w \in E_+$ then $v = \elip^{-1/2} w \in E$. 
Via change of variables this implies that
\begin{align*}
\frac{\textrm{Vol}(E_+)}{\textrm{Vol}(E)} = \det(\elip^{1/2}).
\end{align*}
The determinant is simply the product of the eigenvalues, which is easy to calculate since $\elip$ is diagonal,
\begin{align*}
\det(\elip^{1/2}) = \rho^{(d-1)/2} (\rho(1-\sigma))^{1/2}.
\end{align*}
Plugging in the definitions of $\rho, \sigma$ from Theorem~\ref{thm:todd} proves the statement. 
\end{proof}

\begin{lemma}
\label{lem:todd_use}
Consider a closed and bounded set $V \subset \RR^d$ and a vector $p
\in \RR^d$. Let $B$ be any enclosing ellipsoid of $V$ that is centered at the origin, and we abuse the same symbol for the symmetric positive definite matrix that defines the ellipsoid, i.e., $B = \{v \in \RR^d: v^\top B^{-1} v \le 1\}$.  Suppose there
exists $v \in V$ with $|p^\trans v| \ge \kappa$ and define $B_+$ as
the minimum volume enclosing ellipsoid of $\{v\in B : |p^\trans v| \le
\cutto\}$.  If $\cutto/\kappa \le 1/\sqrt{d}$, we have
\begin{align}
\frac{\textrm{vol}(B_+)}{\textrm{vol}(B)} \le \sqrt{d} \frac{\cutto}{\kappa}\left(\frac{d}{d-1}\right)^{(d-1)/2}\left(1 - \frac{\cutto^2}{\kappa^2}\right)^{(d-1)/2}.
\label{eq:todd_use}
\end{align}
\end{lemma}

\begin{proof}
The first claim is to prove a bound on $p^\trans Bp$.
\begin{align*}
\kappa \le |p^\trans v| = |p^\trans B^{1/2}B^{-1/2}v| \le \sqrt{p^\trans Bp} \sqrt{v^\trans B^{-1}v} \le \sqrt{p^\trans Bp}.
\end{align*}
The last inequality applies since $v \in B$ so that $v^\trans B^{-1}v \le 1$. 
Now we proceed to work with the ellipsoids, let $L = \{v: |v^\trans p| \le \cutto\}$.
Set $B_{+} = MVEE(B \bigcap L)$.
We apply two translations of the coordinate system so that $B$ gets mapped to the unit ball and so that $p$ gets mapped to $\alpha e_1$ (i.e. a scaled multiple of the first standard basis vector). 
The first translation is done by setting $w = B^{-1/2}v$ where $w$ is in the new coordinate system and $v$ is in the old coordinate system. 
Let $p_1 = B^{1/2}p$ so that we can equivalently write $L = \{w: |w^\trans p_1| \le \cutto\}$.
The second translation maps $p_1$ to $\alpha e_1$ via a rotation matrix $R$ such that $RB^{1/2}p = Rp_1 = \alpha e_1$. 
We also translate $w$ to $Rw$ but this doesn't affect the now spherically symmetric ellipsoid, so we do not change the variable names. 

To summarize, after applying the scaling and the rotation, we are interested in $MVEE(I \bigcap \{w: |w^\trans e_1| \le \cutto/\alpha\})$ and specifically, since volume ratios are invariant under affine transformation, we have
\begin{align*}
\frac{\textrm{Vol}(B_+)}{\textrm{Vol}(B)} = \frac{\textrm{Vol}(MVEE(I \bigcap \{w: |w^\trans e_1| \le \cutto/\alpha\}))}{\textrm{Vol}(I)}.
\end{align*}
Here $I$ is the unit ball (i.e. the ellipsoid with identity
matrix). Further applying Fact~\ref{fact:todd}, we obtain  
\begin{align*}
\frac{\textrm{Vol}(B_+)}{\textrm{Vol}(B)} &=
\sqrt{d}\frac{\cutto}{\alpha} \left(\frac{d}{d-1}\right)^{(d-1)/2}
\left(1 - \frac{\cutto^2}{\kappa^2}\right)^{(d-1)/2}.
\end{align*}
It remains to lower bound $\alpha$, which is immediate since
\begin{align*}
\alpha = \|RB^{1/2}p\|_2 = \|B^{1/2}p\|_2 \ge\kappa.
\end{align*}
Substituting this lower bound on $\alpha$ completes the proof.
\end{proof}

\begin{fact}%[big\_obs\_nan.pdf Lemma 18]
\label{lem:0.6}
When $\cutto/\kappa=\frac{1}{3\sqrt{d}}$, the RHS of Eq.~\eqref{eq:todd_use} is less than $0.6$.
\end{fact}
\begin{proof}
Plugging in the numbers, we have the RHS of Eq.~\eqref{eq:todd_use} equal to
$$
\frac{1}{3}\left(\frac{d}{d-1} \frac{9d-1}{9d}\right)^{(d-1)/2}
= \frac{1}{3}\left(1 + \frac{8}{9(d-1)}\right)^{9(d-1)/8 \, \cdot \, 4/9} \le \frac{1}{3} \exp(4/9) \le 0.52.
$$
Here we used the fact that $(1+\frac{1}{x})^x$ is monotonically increasing towards $e$ on $x\in [1, \infty)$.
\end{proof}

\subsection{Deviation Bounds} \label{app:dev_bounds}
\label{sec:deviation}
In this section we prove the deviation bounds. Note that the statement of the lemmas in this section, which are for \cdplearnslack, coincide with those stated in Section~\ref{sec:sketch} for \cdplearn. This is not surprising as the two algorithms draw data and estimate quantities in the same way.

\begin{lemma}[Deviation Bound for $\empVf{f}$]
\label{lem:devVf_app}
With probability at least $1-\delta$, 
$$
|\empVf{f} - \Vf{f}| \le \sqrt{\frac{1}{2\nest}\log\frac{2N}{\delta}}
$$
holds for all $f\in\Fcal$ simultaneously. 
Hence, we can set 
$
\nest \ge \frac{32}{\epsilon^2}\log\frac{2N}{\delta}
$ 
to guarantee that $|\empVf{f} - \Vf{f}|\le \epsilon/8$.
\end{lemma}

\begin{proof}
The bound follows from a straight-forward application of Hoeffding's inequality and the union bound, and we only need to verify that the $\Vf{f}$ is the expected value of the $\empVf{f}$, and the range of the random variables is $[0,1]$.
\end{proof}

\begin{lemma}[Deviation Bound for $\dtctberr{\ft{t}, \pit{t}, h}$]
\label{lem:devVht_app}
For any fixed $\ft{t}$, with probability at least $1-\delta$, 
$$ |\dtctberr{\ft{t}, \pit{t}, h} - \berr{\ft{t}, \pit{t}, h}| \le
3\sqrt{\frac{1}{2\neval}\log\frac{2H}{\delta}}
$$
holds for all $h\in[H]$ simultaneously. Hence, for any
$
\neval \ge \frac{288 H^2}{\epsilon^2}\log\frac{2H}{\delta},
$ 
with probability at least $1-\delta$ we have $|\dtctberr{\ft{t},
  \pit{t}, h} - \berr{\ft{t}, \pit{t}, h}|\le \frac{\epsilon}{8H}$. 
\end{lemma}
\begin{proof}
This bound is another straight-forward application of Hoeffding's
inequality and the union bound, except that the random variables that
go into the average have range $[-1, 2]$, and we have to realize that
$\dtctberr{\ft{t}, \pit{t},h}$ is an unbiased estimate of $\berr{\ft{t},
  \pit{t}, h}$.
\end{proof}

\begin{lemma}[Deviation Bound for $\empberr{f, \pit{t}, h_t }$]
\label{lem:phi_app}
%\akshay{If $\ntrain$ is big enough then $|\empberr{f,D(\hat \pi, h)} - \berr{f,D(\hat \pi, h)} | \le \phi$.}
For any fixed $\pit{t}$ and $h_t$, with probability at least $1-\delta$, 
$$
|\empberr{f,\pit{t}, h_t} - \berr{f,\pit{t}, h_t}| \le \sqrt{\frac{8K\log\frac{2N}{\delta}}{\ntrain}} + \frac{2K\log\frac{2N}{\delta}}{\ntrain}
$$
holds for all $f\in\Fcal$ simultaneously. 
Hence, for any
$
\ntrain \ge \frac{32K}{\phi^2}\log\frac{2N}{\delta}
$ 
and $\phi \le 4$, with probability at least $1-\delta$ we have $|\empberr{f,\pit{t}, h_t} - \berr{f,\pit{t}, h_t}|\le \phi$.
\end{lemma}

\begin{proof}
We first show that $\empberr{f, \pit{t}, h_t}$ is an average of i.i.d.~random variables with mean $\berr{f, \pit{t}, h_t}$. We use $\mu$ as a shorthand for the distribution over trajectories induced by 
$
a_1, \ldots, a_{h_t-1} \sim \pit{t}, a_h \sim \textrm{unif}(\Acal)
$, 
which is the distribution of data used to estimate $\empberr{f, \pit{t}, h_t}$. On the other hand, let $\mu'$  denote the distribution over trajectories induced by 
$
a_1, \ldots, a_{h_t-1} \sim \pit{t}, a_h \sim \pi_f.
$ 
The importance weight used in Eq.~\eqref{eq:estm_berr_slack} essentially converts the distribution from $\mu$ to $\mu'$, hence the expected value of $\empberr{f, \pit{t}, h_t}$ can be written as
\begin{align*}
&~ \EE_{\mu} \left[ K\mathbf{1}[a_h = \pi_f(x_h)]\left(f(x_h,a_h) - r_h - f(x_{h+1},\pi_f(x_{h+1}))\right) \right] \\
= &~ \EE_{\mu'} \left[ f(x_h, a_h) - r_h - f(x_{h+1},\pi_f(x_{h+1})) \right]
= \berr{f, \pit{t}, h_t}.
\end{align*}
Now, we apply Bernstein's inequality. We first analyze the 2nd-moment of the random variable. Defining $y(x_h,a_h,r_h,x_{h+1}) = f(x_h,a_h) - r_h - f(x_{h+1},\pi_f(x_{h+1})) \in [-2,1]$, the 2nd-moment is
\begin{align*}
&~ \EE_{\mu} \left[ \left(K \mathbf{1}[a_h = \pi_f(x_h)] y(x_h,a_h,r_h,x_{h+1}) \right)^2 \right]\\
= &~ \Pr_{\mu}[a_h = \pi_f(x_h)] \cdot \EE_{\mu} \left[ \left(K y(x_h,a_h,r_h,x_{h+1}) \right)^2  \,\big|\,  a_h = \pi_f(x_h) \right] + \Pr_{\mu}[a_h \ne \pi_f(x_h)] \cdot 0\\
\le &~ \frac{1}{K} \EE_{\mu} \left[ K^2 \cdot 4  \,\big|\,  a_h = \pi_f(x_h) \right]  = 4K.
\end{align*}
Next we check the range of the centered random variable. The uncentered variable lies in $[-2K,K]$, and the expected value is in $[-2,1]$, so the centered variable lies in $[-2K-1, K+2] \subseteq [-3K, 3K]$.  Applying Bernstein's inequality, we have with probability at least $1-\delta$,
\begin{align*}
|\empberr{f, \pit{t}, h_t } - \berr{f, \pit{t}, h_t }|
\le &~ \sqrt{\frac{2\,\text{Var}\left[ K \mathbf{1}[a_h = \pi_f(x_h)] y(x_h,a_h,r_h,x_{h+1}) \right] \log\frac{2N}{\delta}}{\ntrain}} + \frac{6K\log\frac{2N}{\delta}}{3\ntrain} \\
\le &~ \sqrt{\frac{8K\log\frac{2N}{\delta}}{\ntrain}} + \frac{2K\log\frac{2N}{\delta}}{\ntrain} \tag{variance is bounded by 2nd-moment}.
\end{align*}
As long as $\frac{2K\log\frac{2N}{\delta}}{\ntrain} \le 1$, the above is bounded by $2 \sqrt{\frac{8K\log\frac{2N}{\delta}}{\ntrain}}$. The choice of $\ntrain$ follows from solving $2 \sqrt{\frac{8K\log\frac{2N}{\delta}}{\ntrain}} = \phi$ for $\ntrain$, which indeed guarantees that $\frac{2K\log\frac{2N}{\delta}}{\ntrain} \le 1$ as $\phi \le 4$.
\end{proof}

\section{Proofs of Extensions}
\label{app:extensions}

\subsection{Proof for Unknown \brankbig (Theorem~\ref{thm:guessM})}

Since we assign $\frac{\delta}{i(i+1)}$ failure probability to the $i$-th call of Algorithm~\ref{alg:guessM}, the total failure probability is at most
\begin{align*}
\sum_{i=1}^\infty \frac{\delta }{i(i+1)} = 
\delta \sum_{i=1}^\infty \left(\frac{1}{i} - \frac{1}{i+1}\right)  = \delta.
\end{align*}
So with probability at least $1-\delta$, all high probability events
in the analysis of \cdplearn occur for every $i=1,2,\ldots$. Note that
regardless of whether $M'<M$, we never eliminate $f^\star$ according
to Lemma~\ref{lem:vol}. Hence Lemma~\ref{lem:optimism_explore} holds
and whenever the algorithm returns a policy it is near-optimal.

While the algorithm returns a near-optimal policy if it terminates, we
still must prove that the algorithm terminates. Since when $M' <
M$ Eq.~\eqref{eq:large_violation} and Lemma~\ref{lem:vol_app} do not
apply, we cannot naively use arguments from the analysis of
\cdplearn. However, we monitor the number of iterations that have passed
in each execution to \cdplearn and stop the subroutine when the actual
number of iterations exceeds the iteration complexity bound
(Lemma~\ref{lem:vol}) to prevent wasting more samples on the wrong
$M'$.

\cdplearn is guaranteed to terminate within the sample complexity
bound and output near-optimal policy when $M\le M'$. Since $M'$ grows
on a doubling schedule, for the first $M'$ that satisfies $M \le M'$,
we have $M' \le 2M$ and $i \le \log_2 M + 1$. Hence, the total number
of calls is bounded by $\log_2 M + 1$.

Finally, since the sample complexity bound in
Theorem~\ref{thm:cdp_complexity} is monotonically increasing in $M$
and $1/\delta$ and the schedule for $\delta'$ is increasing, we can
bound the total sample complexity by that of the last call to
\cdplearn multiplied by the number of calls.  The last call to
\cdplearn has $M' \le 2M$, and $\frac{i(i+1)}{\delta} \le
\frac{(\log_2 M + 2)(\log_2 M + 1)}{\delta}$, so the sample complexity
bound is only affected by factors that are at most logarithmic in the
relevant parameters.

\subsection{Proofs for Infinite Hypothesis Classes} \label{app:inf_hyp_class}

In this section we prove sample complexity guarantee for using
infinite hypothesis classes in Section~\ref{sec:inf_hyp_class}. Recall
that we are working with separated policy class $\Pi$ and \vval
function class $\Gcal$, and when running \cdplearn any occurrence of
$f\in\Fcal$ is replaced appropriately by $(\pi, g) \in
\Pi\times \Gcal$. For clarity, we use $(\pi,g)$ instead of $f$ in the
derivations in this section. We assume that the two function classes
have finite Natarajan dimension and pseudo dimension respectively.

The key technical step for the sample complexity guarantee is to
establish the necessary deviation bounds for infinite classes. Among
these deviation bounds, the bound on $\dtctberr{(\pit{t},g_t), \pit{t}, h}$
(Lemma \ref{lem:devVht}) does not involve union bound over $\Fcal$, so
it can be reused without modification. The other two bounds need to be
replaced by Lemma~\ref{lem:devVfDim} and \ref{lem:devBerrDim}, stated
below. With these lemmas, Theorem~\ref{thm:inf_hyp_class} immediately
follows simply by replacing the deviation bounds.

\begin{definition}
Define $d_{\Pi} = \max(\textrm{Ndim}(\Pi), 6), d_{\Gcal} = \max(\textrm{Pdim}(\Gcal), 6)$, and $d = d_\Pi + d_\Gcal$. 
\end{definition}

\begin{lemma}
\label{lem:devVfDim}
If 
\begin{align} \label{eq:devVfDim}
\nest \ge \frac{8192}{\epsilon^2}\left( d_{\Gcal}\log\frac{128e}{\epsilon} + \log(8e(d_{\Gcal}+1)) + \log\frac{1}{\delta}\right), 
\end{align}
then with probability at least $1-\delta$, $|\empVf{(\pi,g)} - \Vf{(\pi,g)}|\le \epsilon/8, ~\forall (\pi,g)\in\Pi \times \Gcal$. 
\end{lemma}
We remark that both the estimate $\empVf{(\pi,g)}$ and population
quantity $\Vf{(\pi,g)}$ are independent of $\pi$ in the separable
case, and hence the sample complexity is independent of $d_{\Pi}$.

\begin{lemma}
\label{lem:devBerrDim}
If
\begin{align} 
\ntrain \ge & \frac{1152K^2}{\phi^2}\left( 6d \log\left(2eKd\right) \log\frac{48eK}{\phi} + \log\Big(8e (6d \log\left(2eKd\right)+1)\Big) + \log\frac{3}{\delta}\right),  \label{eq:devBerrDim}
\end{align}
then for any fixed $\pi_t$ and $h_t$, with probability at least $1-\delta$,
$$
|\empberr{(\pi,g),\pit{t}, h_t} - \berr{(\pi,g),\pit{t}, h_t}| \le \phi, ~\forall (\pi,g)\in\Pi \times \Gcal.
$$
\end{lemma}

\begin{proof}[Proof of Theorem~\ref{thm:inf_hyp_class}]
Set the algorithm parameters to:
\begin{align*}
\phi &= \frac{\epsilon}{12H\sqrt{M}}, \qquad \nest = \frac{8192}{\epsilon^2}\left( d_{\Gcal}\log\frac{128e}{\epsilon} + \log(8e(d_{\Gcal}+1)) + \log\frac{3}{\delta}\right), \\
\neval &= \frac{288H^2}{\epsilon^2}\log\left(\frac{12 H^2 M \log(6H\sqrt{M}\normbound/\epsilon)}{\delta}\right),\\
\ntrain &= \frac{1152K^2}{\phi^2}\left( 6d \log\Big(2eK\right) \log\frac{48eK}{\phi} + \log\Big(8e(6d \log\left(2eKd\right)+1)\Big) \\
& \qquad \qquad \qquad + \log\frac{18H M\log (6H\sqrt{M} \normbound/\epsilon) }{\delta}\Big).
\end{align*}
The rest of the proof is essentially the same as the proof of Theorem~\ref{thm:cdp_complexity}, and the sample complexity follows by noticing that $\nest = \otil(\frac{d_{\Gcal} + \log(1/\delta)}{\epsilon^2})$ and $\ntrain = \otil(K^2  (d_{\Pi}+d_{\Gcal}+\log(1/\delta)) / \phi^2 )$.
\end{proof}

Lemma~\ref{lem:devVfDim} is a straight-forward application of Corollary~\ref{cor:devPdim} introduced in \ref{app:inf_hyp_def} and are not proved separately. The remainder of this section, we prove Lemma~\ref{lem:devBerrDim}. Before that, we review some standard definitions and results from statistical learning theory.

\subsubsection{Definitions and Basic Lemmas}
\label{app:inf_hyp_def}
Notations $\Xcal, x, n, d, \xi$ in this section are used according to conventions in the literature and may not share semantics with the same symbols used elsewhere in this paper.

\begin{definition}[VC-Dimension]
Given hypothesis class $\Hcal \subset \Xcal \to \{0,1\}$, its VC-dimension $\textrm{VC-dim}(\Hcal)$ is defined as the maximal cardinality of a set $X = \{x_1, \ldots, x_{|X|}\} \subset \Xcal$ that satisfies $|\Hcal_{X}| = 2^{|X|}$ (or $X$ is \emph{shattered} by $\Hcal$), where $\Hcal_{X}$ is the restriction of $\Hcal$ to $X$, namely $\{(h(x_1), \ldots, h(x_{|X|})): h\in\Hcal \}$. 
\end{definition}

\begin{lemma}[Sauer's Lemma]
\label{lem:sauer}
Given hypothesis class $\Hcal \subset \Xcal \to \{0,1\}$ with $d = \textrm{VC-dim}(\Hcal) < \infty$, we have $\forall X = (x_1, \ldots, x_n) \in \Xcal^n$, 
$$
|\Hcal_{X}| \le (n+1)^d.
$$
\end{lemma}

\begin{lemma}[Sauer's Lemma for Natarajan dimension \citep{ben1992characterizations, haussler1995generalization}]
\label{lem:natarajan}
Given hypothesis class $\Hcal \subset \Xcal \to \Ycal$ with $\textrm{Ndim}(\Hcal) \le d$, we have $\forall X = (x_1, \ldots, x_n) \in \Xcal^n$, 
$$
|\Hcal_{X}| \le \left(\frac{ne(K+1)^2}{2d}\right)^d,
$$
where $K = |\Ycal|$.
\end{lemma}

\begin{definition}[Covering number]
\label{def:covering_number}
Given hypothesis class $\Hcal \subset \Xcal \to \mathbb{R}$, $\epsilon > 0$, $X =(x_1, \ldots, x_n) \in \Xcal^n$, the covering number $\Ncal_1(\alpha, \Hcal, X)$ is defined as the minimal cardinality of a set $C \subset \mathbb{R}^n$, such that for any $h\in \Hcal$ there exists $c = (c_1, \ldots, c_n)\in C$ where $\frac{1}{n} \sum_{i=1}^{n} |h(x_i) - c_i| \le \alpha$.
\end{definition}

\begin{lemma}[Bounding covering number by pseudo dimension \citep{haussler1995sphere}]
\label{lem:pdim_covering}
Given hypothesis class $\Hcal \subset \Xcal \to \mathbb{R}$ with $\textrm{Pdim}(\Hcal) \le d$, we have for any $X\in \Xcal^n$, 
$$
\Ncal_1(\alpha, \Hcal, X) \le e(d+1)\left(\frac{2e}{\alpha}\right)^d.
$$
\end{lemma}

\begin{lemma}[Uniform deviation bound using covering number \citep{pollard1984convergence}; also see \citet{devroye1996probabilistic}, Theorem 29.1]
\label{lem:udb_covering}
Let $\Hcal \subset \Xcal \to [0,b]$ be a hypothesis class, and $(x_1, \ldots, x_n)$ be i.i.d.~samples drawn from some distribution supported on $\Xcal$. For any $\alpha>0$,
$$
\Pr \left\{ \sup_{h\in\Hcal} \left| \frac{1}{n}\sum_{i=1}^n h(x_i) - \EE [h(x_1)] \right| > \alpha \right\} \le 8\, \EE\left[\Ncal_1\big(\alpha/8, \Hcal, (x_1, \ldots, x_n)\big)\right] \exp\left(-\frac{n \alpha^2}{128b^2}\right).
$$
\end{lemma}

\begin{corollary}[Uniform deviation bound using pseudo dimension]
\label{cor:devPdim}
Suppose $\textrm{Pdim}(\Hcal) \le d$, then 
$$
\Pr \left\{ \sup_{h\in\Hcal} \left| \frac{1}{n}\sum_{i=1}^n h(x_i) - \EE [h(x_1)] \right| > \alpha \right\} \le 8e(d+1)\left(\frac{16e}{\alpha}\right)^d \exp\left(-\frac{n \alpha^2}{128b^2}\right).
$$
To guarantee that this probability is upper bounded by $\delta$, it suffices to have
$$
n \ge \frac{128}{\alpha^2}\left( d \log\frac{16e}{\alpha} + \log(8e(d+1)) + \log\frac{1}{\delta}\right).
$$
\end{corollary}

\subsubsection{Proof of Lemma~\ref{lem:devBerrDim}}
\label{app:inf_hyp_result}
%In this section we prove the deviation bounds for the situation where we use $\Pi \subset \Xcal \to \Acal$ and $\Gcal \subset \Xcal \to [0, 1]$ to model policies and \vval functions separately, sketched in Section~\ref{sec:inf_hyp_class}. %under the assumptions that they have finite combinatorial dimensions. 

The idea is to establish deviation bounds for each of the three terms
in the definition of $\empberr{(\pi,g),\pit{t},h_t}$
(Eq.~\eqref{eq:estm_berr}). Each term takes the form of an importance
weight multiplied by a real-valued function, and we first show that
the function space formed by these products has bounded
pseudo dimension. We state this supporting lemma in terms of an
arbitrary value-function class $\Vcal$ which might operate on an input
space $\Xcal'$ different from the context space $\Xcal$. In the
sequel, we instantiate $\Vcal$ and $\Xcal'$ in the the lemma with
specific choices to prove the desired results.

\begin{lemma}
\label{lem:pdim}
Let $\Ycal$ be a label space with $|\Ycal| = K$, let $\Pi \subseteq \Xcal\to \Ycal$ be a function class with Natarajan dimension at most $d_\Pi \in [6, \infty)$, and let $\Vcal \subseteq \Xcal' \to [0,1]$ be a class with pseudo dimension at most $d_\Vcal \in [6,\infty)$. 
The hypothesis class $\Hcal = \{(x,a,x') \mapsto \mathbf{1}[a =\pi(x)]g(x') : \pi \in \Pi, g \in \Vcal\}$ has pseudo dimension $\textrm{Pdim}(\Hcal) \le 6(d_{\Pi} + d_{\Vcal}) \log\left(2eK(d_{\Pi} + d_{\Vcal})\right)$.
\end{lemma}
\begin{proof}
Recall that $\textrm{Pdim}(\Hcal) = \textrm{VC-dim}(\Hcal^+)$, so it
suffices to show that for any \\$X = \{(x_1, a_1, x_1', \xi_1),
\ldots, (x_d, a_d, x_d', \xi_d)\} \in (\Xcal\times\Acal\times\Xcal'
\times \mathbb{R})^d$ where $d = 6(d_{\Pi}+d_{\Vcal})
\log\left(2eK(d_{\Pi}+ d_{\Vcal})\right)$, $|\Hcal^+_X| < 2^d$. Note
that since $g(x) \in [0,1]$ for all $g, x$
\begin{align*}
\Hcal^+ = &~ \{(x,a,x',\xi) \mapsto
\mathbf{1}\Big[\mathbf{1}[a=\pi(x)] g(x') > \xi \Big]\} \\ 
= &~ \{ (x,a,x',\xi) \mapsto\mathbf{1}[\xi < 0] + \mathbf{1}[\xi \ge
  0] \cdot \mathbf{1}[a=\pi(x)] \cdot \mathbf{1}[g(x') > \xi] \}  
\end{align*}
For points where $\xi_i < 0$, all hypotheses in $\Hcal^+$ produce
label $1$, so without loss of generality we can assume that $\xi_i \ge
0, i=1, \ldots, d$.

With a slight abuse of notation, let $\Pi_{X}$ denote the restriction
of $\Pi$ to the set of contexts $\{x_1, \ldots, x_d\}$ (actions and
future contexts $(a_1, x_1'), \ldots, (a_d, x_d')$ are ignored since
$\Pi$ does not operate on them), and $\Vcal^+_{X}$ denote the
restriction of $\Vcal^+$ to $\{(x_1', \xi_1), \ldots, (x_d',
\xi_d)\}$. $\Hcal_X^+$ can be produced by the Cartesian product of
$\Pi_X$ and $\Vcal^+_X$ as follows:
\begin{align*}
\Hcal^+_X = \{(\mathbf{1}[a_1 = \alpha_1] \beta_1, \ldots,
\mathbf{1}[a_d = \alpha_d] \beta_d): (\alpha_1, \ldots, \alpha_d) \in
\Pi_X, (\beta_1, \ldots, \beta_d) \in \Vcal^+_X\}.
\end{align*}
Therefore, $|\Hcal^+_X| \le |\Pi_X| \, |\Vcal^+_X|$. Recall that $\textrm{Ndim}(\Pi) \le d_{\Pi}$ and $\textrm{VC-dim}(\Vcal^+) = \textrm{Pdim}(\Vcal) \le d_{\Vcal}$. Applying Lemma~\ref{lem:natarajan} and \ref{lem:sauer}:
\begin{align*}
|\Hcal^+_X| \le \left(\frac{de(K+1)^2}{2d_{\Pi}}\right)^{d_{\Pi}} (d+1)^{d_{\Vcal}}.
\end{align*}
The logarithm of the RHS is
\begin{align*}
&~ d_{\Pi} \log \left(\frac{de(K+1)^2}{2d_{\Pi}}\right) +  d_{\Vcal} \log (d+1) 
< d_{\Pi} \log (de(K+1)^2) +  d_{\Vcal} \log (d+1) \\ 
\le &~ d_{\Pi} \log d + 2 d_{\Pi} \log (2eK) +  d_{\Vcal} \log (d + 1) 
\le 2 (d_{\Pi} + d_{\Vcal}) \log(2eK) + (d_{\Pi} + d_{\Vcal}) \log(2d).
\end{align*}
It remains to be shown that this is less than $\log(2^d) = d\log 2$. Note that
\begin{align*}
d\log 2 > 3(d_{\Pi}+d_{\Vcal})(\log(2eK) + \log(d_{\Pi}+ d_{\Vcal})),
\end{align*}
so we only need to show that $(d_{\Pi} + d_{\Vcal}) \log(2d) \le (d_{\Pi}+d_{\Vcal})\log(2eK) + 3(d_{\Pi}+d_{\Vcal}) \log(d_{\Pi}+ d_{\Vcal})$. Now
\begin{align*}
&~ (d_{\Pi} + d_{\Vcal}) \log(2d)
= (d_{\Pi} + d_{\Vcal}) \Big( \log(12(d_{\Pi}+d_{\Vcal})) + \log\log(2eK(d_{\Pi}+d_{\Vcal})) \Big) \\
\le &~ 2(d_{\Pi} + d_{\Vcal}) \log(d_{\Pi}+d_{\Vcal}) + (d_{\Pi} + d_{\Vcal})  \log\Big(\log(2eK) + \log(d_{\Pi}+d_{\Vcal}) \Big) \tag{$d_{\Pi} + d_{\Vcal} \ge 12$} \\
\le &~ 2(d_{\Pi} + d_{\Vcal}) \log(d_{\Pi}+d_{\Vcal}) + (d_{\Pi} + d_{\Vcal}) \big(\log(2eK) + \log(d_{\Pi}+d_{\Vcal}) \big). \tag*{\qedhere}
\end{align*}
\end{proof}

\begin{proof}[Proof of Lemma~\ref{lem:devBerrDim}]
Recall that when we are given a policy class $\Pi$ and separate \vval
function class $\Gcal$, for every $\pi \in \Pi,g \in \Gcal$, we
instead estimate average Bellman error with
%% $\empberr{f,\pit{t}, h_t}$ becomes the following when we are given policy class $\Pi$ and \vval function class $\Gcal$: for every $\pi\in\Pi, g\in\Gcal$, we estimate
$$
\empberr{(\pi,g),\pit{t},h_t} = \frac{1}{\ntrain}\sum_{i=1}^{\ntrain} \frac{\mathbf{1}[a_{h_t}^{(i)} = \pi(x_{h_t}^{(i)})]}{1/K} \Big(g(x_{h_t}^{(i)}) - r_{h_t}^{(i)} - g(x_{h_{t}+1}^{(i)})\Big).
$$
So it suffices to show that the averages of $\mathbf{1}[a_{h_t}^{(i)} = \pi(x_{h_t}^{(i)})]g(x_{h_t}^{(i)})$,  $\mathbf{1}[a_{h_t}^{(i)} = \pi(x_{h_t}^{(i)})]r_{h_t}^{(i)}$,  $\mathbf{1}[a_{h_t}^{(i)} = \pi(x_{h_t}^{(i)})]g(x_{h_{t}+1}^{(i)})$ are $\frac{\phi}{3K}$-close to their expectations with probability at least $1-\delta/3$, respectively. It turns out that, we can use Lemma~\ref{lem:pdim} for all the three terms. For the first and the third terms, we apply Lemma~\ref{lem:pdim} with $\Vcal = \Gcal, \Xcal' = \Xcal$, and obtain the necessary sample size directly from Corollary~\ref{cor:devPdim}. For the second term, we apply Lemma~\ref{lem:pdim} with $\Vcal = \{x \mapsto x\}, \Xcal' =\mathbb{R}$. Note that in this case $\Vcal$ is a singleton with the only element being the identity function over $\mathbb{R}$, so it is clear that $\textrm{Pdim}(\Vcal) < 6 \le d_{\Gcal}$, hence the sample size for the other two terms is also adequate for this term.
\end{proof}

\subsection{Proofs for \cdplearnslackbf}
Recall that the main lemmas for analyzing \cdplearnslack have been proved in Appendix~\ref{app:main-proofs}, so below we directly prove Theorem~\ref{thm:slack}.

\begin{proof}[Proof of Theorem~\ref{thm:slack}]
Suppose the preconditions of Lemma~\ref{lem:optimism_explore_app} (Eq.~\eqref{eq:vf_vpi_app}) and Lemma~\ref{lem:vol_app} (Eq.~\eqref{eq:learning_works_app}) hold; we show them by invoking the deviation bounds later. By Lemma~\ref{lem:optimism_explore_app}, when the algorithm terminates, the value of the output policy is at least
$$
\vstarslack - \epsilon' - H\slack. 
$$
Recall that $\epsilon' = \epsilon + 2H(3\sqrt{M} (\slack + \slackM) + \slackM)$ (Line \ref{lin:def_eps_prime_slack}), so the suboptimality compared to $\vstarslack$ is at most 
$$
\epsilon + 2H(3\sqrt{M} (\slack + \slackM) + \slackM) + H \slack \le \epsilon + 8H\sqrt{M}(\slack + \slackM),
$$
which establishes the suboptimality claim. 

It remains to show the sample complexity bound.
Applying Lemma~\ref{lem:optimism_explore_app}, in every iteration $t$ before the algorithm terminates, 
$$
\berr{\ft{t},  \pi_{t},h_t } \ge \frac{\epsilon'}{2H} =  \frac{\epsilon}{2H} + 3\sqrt{M}(\slack+\slackM) + \slackM = 3\sqrt{M}(2\phi + \slack + \slackM) + \slackM,
$$
thanks to the choice of $\phi$ and $\epsilon'$. For level $h = h_t$, Eq.~\eqref{eq:large_violation_app} is satisfied. According to Lemma~\ref{lem:vol_app}, the event $h_t = h$ can happen at most $M\log \left(\frac{\normbound}{2\phi}\right)/ \log\frac{5}{3}$ times for every $h\in[H]$. Hence, the total number of iterations in the algorithm is at most 
$$
H   M\log \left(\frac{\normbound}{2\phi}\right)/ \log\frac{5}{3} = H M\log \left(\frac{6H\sqrt{M} \normbound}{\epsilon}\right) / \log\frac{5}{3}.
$$
Now we are ready to apply the deviation bounds to show that Eq.~\eqref{eq:vf_vpi_app} and \ref{eq:learning_works} hold with high probability. We split the total failure probability $\delta$ among the following events:
\begin{enumerate}
\item Estimation of $\empVf{f}$ (Lemma~\ref{lem:devVf_app}; only once): $\delta / 3$.
\item Estimation of $\dtctberr{\ft{t},\pit{t},h}$ (Lemma~\ref{lem:devVht_app}; every iteration):  $\delta / \left( 3H M\log \left(\frac{6H\sqrt{M} \normbound}{\epsilon}\right) / \log\frac{5}{3}\right)$.
\item Estimation of $\empberr{f, \pit{t},h_t}$ (Lemma~\ref{lem:phi_app}; every iteration): same as above.
\end{enumerate}

Applying Lemma~\ref{lem:devVf_app}, \ref{lem:devVht_app},
\ref{lem:phi_app} with the above failure probabilities, the choices of
$\nest, \neval, \ntrain$ in the algorithm statement satisfy the
preconditions of Lemmas~\ref{lem:optimism_explore_app} and
\ref{lem:vol_app}. In particular, the choice of $\nest$ and $\neval$
guarantee that $|\empVf{f} - \Vf{f}| \le \epsilon/8$ and
$|\dtctberr{\ft{t},\pit{t},h} - \berr{\ft{t},\pit{t},h}| \le
\epsilon/(8H)$, which are tighter than needed as $\epsilon \le
\epsilon'$ (only $\epsilon'/8$ and $\epsilon'/(8H)$ are needed
respectively, but tightening these bounds does not improve the sample
complexity significantly, so we keep them the same as in
Theorem~\ref{thm:cdp_complexity} for simplicity). The remaining
calculation of sample complexity is exactly the same as in the proof
of Theorem~\ref{thm:cdp_complexity}.
\end{proof}

\section*{Acknowledgements}
Part of this work was completed while NJ and AK were at Microsoft Research. NJ was also partially supported by Rackham Predoctoral Fellowship in University of Michigan.

\bibliographystyle{plainnat} \bibliography{rl}

\end{document}